\documentclass{article}

\usepackage{amsfonts}       
\usepackage{nicefrac}       
\usepackage{microtype}      
\usepackage{lipsum}
\usepackage{graphicx}
\usepackage{amsmath}
\usepackage{amsthm}
\usepackage{ bbold }
\usepackage{algorithm}
\usepackage{algorithmic}
\usepackage{mathrsfs}
\usepackage{mathtools}
\usepackage{amssymb}
\usepackage{hyperref}
\usepackage{caption}
\usepackage{subcaption}
\usepackage{hyperref}
\usepackage{authblk}

\usepackage[numbers]{natbib}

\usepackage[a4paper, total={6in, 9.5in}]{geometry}

\usepackage{enumerate}
\usepackage[dvipsnames,svgnames]{xcolor}

\newtheorem{theorem}{Theorem}

\newtheorem{lemma}[theorem]{Lemma}
\newtheorem{prop}[theorem]{Proposition}
\newtheorem{remark}{Remark}
\newtheorem{definition}{Definition}

\usepackage{microtype}
\usepackage{graphicx}
\usepackage{wrapfig}
\usepackage{booktabs} 

\newcommand{\N}{\mathbb{N}}

\newcommand{\cellof}[1]{L_n(#1)}
\newcommand{\cardcellof}[1]{N_n(L_n(#1))}
\newcommand{\Esp}[1]{\mathbb{E} \left[ #1\right]}
\newcommand{\Prob}[1]{\mathbb{P} \left( #1 \right)}
\newcommand{\ind}[1]{\mathbb{1}_{#1}}
\newcommand{\NBC}{|L_n(X) \cap \mathcal{B} |}
\newcommand{\NWC}{|L_n(X) \cap \mathcal{W}|}
\newcommand{\expkstar}[1]{2^{k^\star #1}}

\usepackage[colorinlistoftodos,textwidth=2.3cm]{todonotes}

\usepackage[normalem]{ulem}
\newcommand{\replace}[2]{{#2}}

\title{Analyzing the tree-layer structure of Deep Forests}

\author[1]{Ludovic Arnould} 
\author[1]{Claire Boyer}
\author[2]{Erwan Scornet}

\affil[1]{LPSM, Sorbonne Universit\'e}
\affil[2]{CMAP, Ecole Polytechnique }

\title{Analyzing the tree-layer structure of Deep Forests}

\begin{document}

\maketitle

\begin{abstract}
    Random forests on the one hand, and neural networks on the other hand, have met great success in the machine learning community for their predictive performance. Combinations of both have been proposed in the literature, notably leading to the so-called deep forests (DF) \cite{zhou2019deep}. In this paper, our aim is not to benchmark DF performances but to investigate instead their underlying mechanisms. Additionally, DF architecture can be generally simplified into more simple and 
    computationally efficient shallow forests networks. Despite some instability, the latter may outperform standard predictive tree-based methods.
    We exhibit a theoretical framework in which a shallow tree network is shown to enhance the performance of classical decision trees. 
    In such a setting, we provide tight theoretical lower and upper bounds on its excess risk. 
    These theoretical results show the interest of tree-network architectures for well-structured data provided that the first layer, acting as a data encoder, is rich enough.
\end{abstract}

\section{Introduction}

Deep Neural Networks (DNNs) are among the most widely used machine learning algorithms. They are composed of parameterized differentiable non-linear modules trained by gradient-based methods, which rely on the backpropagation procedure. Their performance mainly relies on layer-by-layer processing as well as feature transformation across layers. Training neural networks usually requires complex hyper-parameter tuning \cite{NIPS2011_4443} and a huge amount of data. Although DNNs recently achieved great results in many areas, they remain very complex to handle and unstable to input noise \cite{zheng2016improving}.

Recently, several attempts have been made to consider networks with non-differentiable modules. Among them the Deep Forest (DF) algorithm \cite{zhou2019deep}, which uses Random Forests (RF) \cite{breiman2001random} as neurons, has received a lot of attention in recent years in various applications such as hyperspectral image processing \cite{liu2020morphological}, medical imaging \cite{sun2020adaptive}, drug interactions \cite{su2019deep, zeng2020network} or even fraud detection \cite{zhang2019distributed}.

Since the DF procedure stacks multiple layers, each one being composed of complex nonparametric RF estimators, the rationale behind the procedure remains quite obscure. However DF methods exhibit impressive performances in practice, suggesting that stacking RFs and extracting features from these estimators at each layer is a promising way to leverage on the RF performance in the neural network framework. 
{The goal of this paper is not an exhaustive empirical study of prediction performances of DF \citep[see][]{zhou2019deepbis}  but rather to understand how stacking trees in a network fashion may result in competitive infrastructure.}

\paragraph{Related Works.}
Different manners of stacking trees exist \replace{}{ \citep[see][for a general survey on stacking methods]{ghods2020survey}}, as the Forwarding Thinking Deep Random Forest (FTDRF), proposed by \cite{ftdrf}, for which the proposed network contains trees which directly transmit their output to the next layer (contrary to Deep Forest in which their output is first averaged before being passed to the next layer). A different approach by  \cite{feng2018multi} consists in  rewriting tree gradient boosting as a simple neural network whose layers can be made arbitrary large depending on the boosting tree structure. The resulting estimator is more simple than DF but does not leverage on the ensemble method properties of random forests. 

In order to prevent overfitting and to lighten the model, several ways to simplify DF architecture have been investigated. \cite{DF_confidence_screening} considers RF whose complexity varies through the network, and combines it with a confidence measure to pass high confidence instances directly to the output layer.
Other directions towards DF architecture simplification are to play on the nature of the RF involved \cite{DERT}  (using Extra-Trees instead of Breiman's RF),  on the number of RF per layer 
\cite{jeong2020lightweight} (implementing layers of many forests with few trees), 
or even on the number of features passed between two consecutive layers \cite{su2019deep} by relying on an importance measure to process only the most important features at each level.
The simplification can also occur once the DF architecture is trained, as in \cite{kim2020interpretation} selecting in each forest the most important paths to reduce the network time- and memory-complexity.
Approaches to increase the approximation capacity of DF have also been proposed by adjoining weights to trees or to forests in each layer \cite{utkin2017discriminative,utkin2020improvement}, replacing the forest by more complex estimators (cascade of ExtraTrees) \cite{berrouachedi2019deep}, or by combining several of the previous modifications notably incorporating data preprocessing \cite{guo2018bcdforest}.
Overall, the related works on DF exclusively represent algorithmic contributions without a formal understanding of the driving mechanisms at work inside the forest cascade. 
\paragraph{Contributions.}
In this paper, we analyze the benefit of combining trees in  network architecture both theoretically and numerically. 
As the performances of DF have already been validated by the literature \citep[see][]{zhou2019deepbis}, the main goals of our study are (i) to quantify the potential benefits of DF over RF, and (ii) to understand the mechanisms at work in such complex architectures.
We show in particular that much lighter configuration can be on par with DF default configuration, leading to a drastic reduction of the number of parameters in few cases. For most datasets, considering DF with two layers is already an improvement over the basic RF algorithm. However, the performance of the overall method is highly dependent on the structure of the first random forests, which leads to stability issues.  
%
%
By establishing tight lower and upper bounds on the risk, we prove that a shallow tree-network may outperform an individual tree in the specific case of a well-structured dataset if the first encoding tree is rich enough. This is a first step to understand the interest of extracting features from trees, and more generally the benefit of tree networks. 
\paragraph{Agenda.}
DF are formally described in Section \ref{sec:deep_forest}. 
Section~\ref{sec:ref_numerical_analysis} is devoted to the numerical study of DF, by evaluating the influence of the number of layers in DF architecture, by showing that shallow sub-models of one or two layers perform the best, and finally by understanding the influence of tree depth in cascade of trees.
Section~\ref{sec:theoretical_results} contains the theoretical analysis of the shallow centered tree network. For reproducibility purposes, all codes together with all experimental procedures are to be found in the supplementary materials.

\section{Deep Forests}
\label{sec:deep_forest}

\subsection{Description}

Deep Forest \cite{zhou2019deep} is a hybrid learning procedure in which random forests  are used as the elementary components (neurons) of a neural network. 
Each layer of DF is composed of an assortment of Breiman's forests and Completely-Random Forests (CRF) \cite{fan2003random} and trained one by one. 
In a classification setting, each forest of each layer outputs a class probability distribution for any query point $x$, corresponding to the distribution of the labels in the node containing $x$.
At a given layer, the distributions output by  all forests of this layer are concatenated, together with the raw data. This new vector serves as input for the next DF layer. 
This process is repeated for each layer and the final classification is performed by averaging the forest outputs of the best layer (without raw data) and applying the \texttt{argmax} function. The overall architecture is depicted in Figure \ref{fig:DF}.

\begin{figure}[!htp]
    \centering
    \includegraphics[width = 0.43\textwidth]{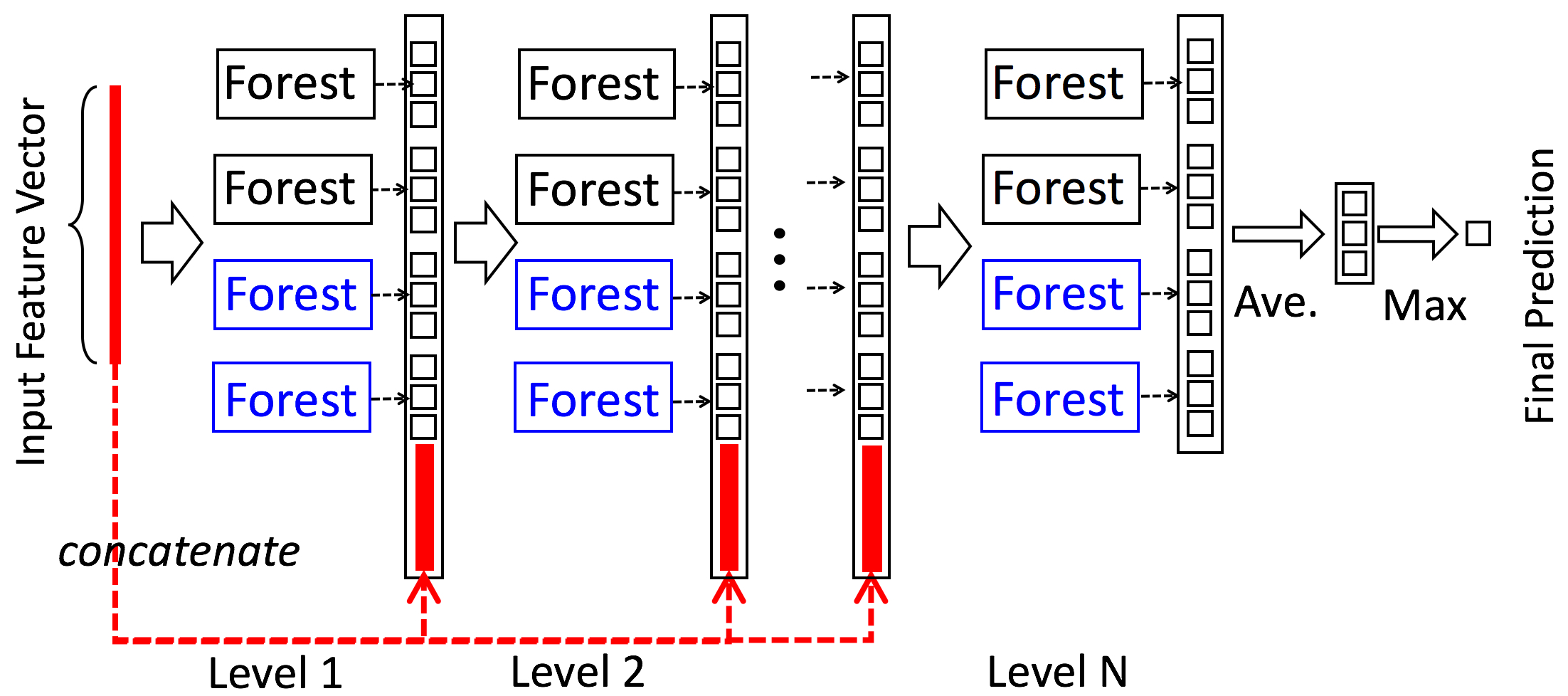}
    \caption{Deep Forest architecture (the scheme is taken from \cite{zhou2019deep}).}
    \label{fig:DF}
\end{figure}

\subsection{DF hyperparameters}
Deep Forests contain an important number of tuning parameters. Apart from the traditional parameters of random forests, DF architecture depends on the number of layers, the number of forests per layer, the type and proportion of random forests to use (Breiman or CRF). In \cite{zhou2019deep}, the default configuration is set to 8 forests per layer, 4 CRF and 4 RF, 500 trees per forest (other forest parameters are set to \texttt{sk-learn} \cite{scikit-learn} default values), and layers are added until 3 consecutive layers do not show score improvement.

Due to their large number of parameters and the fact that they use a complex algorithm as elementary bricks, DF consist in a potential high-capacity procedure. However, as a direct consequence, the numerous pa\-ra\-me\-ters are difficult to estimate (requiring specific tuning of the optimization process) and need to be stored which leads to high prediction time and large memory consumption. Besides, the layered structure of this estimate, and the fact that each neuron is replaced by a powerful learning algorithm makes the whole prediction hard to properly interpret.


As already pointed out, several attempts to lighten the architecture have been conducted. In this paper, we will propose and assess the performance of a lighter DF configuration on tabular datasets.

\begin{remark}
DF \cite{zhou2019deep} was first designed to classify images. To do so, a pre-processing network called Multi Grained Scanning (MGS) based on convolutions is first applied to the original images. Then the Deep Forest algorithm runs with the newly created features as inputs. 
\end{remark}

\section{Refined numerical analysis of DF architectures} 
\label{sec:ref_numerical_analysis}

In order to understand the benefit of using a complex architecture like Deep Forests, we compare different configurations of DF on six datasets in which the output is binary, multi-class or continuous, see Table~\ref{tab:dataset_description} for description. All classification datasets belong to the UCI repository, the two regression ones are Kaggle datasets (Housing data and Airbnb Berlin 2020)\footnote{https://www.kaggle.com/raghavs1003/airbnb-berlin-2020 \\
https://www.kaggle.com/c/house-prices-advanced-regression-techniques/data}. 
\replace{}{Note that the Fashion Mnist features are built using the Multi Grained Scanning process from the DF original article \cite{zhou2019deep} (see \ref{app:exp_fashion} for the encoding details). }

\begin{table}[h]
\small\addtolength{\tabcolsep}{-4pt}
\centering
 \begin{tabular}{||c c c c||} 
 \hline
 Dataset & Type (Nb of classes) & Train/Val/Test Size & Dim \\ [0.5ex] 
 \hline
 Adult & Class. (2) & 26048/ 6512/ 16281 & 14 \\ 
 \hline
 Higgs & Class. (2) & 120000/ 28000/ 60000 & 28 \\
 \hline
 Fashion Mnist & Class (10) & 24000/ 6000/ 8000 &260 \\
\hline 
Letter &  Class. (26) & 12800/ 3200/ 4000 & 16\\
\hline
Yeast &  Class. (10) & 830/ 208/ 446 & 8\\
 \hline
 Airbnb & Regr. & 73044/ 18262/ 39132 & 13\\
 \hline
 Housing & Regr. & 817/ 205/ 438 & 61 \\ [0.1ex] 
 \hline
\end{tabular}
\caption{Description of the datasets.}
\label{tab:dataset_description}
\end{table}

\vspace{-0.3cm}
In what follows, we propose a light DF configuration. We show that our light configuration performance is comparable to the performance of the default DF architecture of \cite{zhou2019deep}, thus questioning the relevance of deep models. Therefore, we analyze the influence of the number of layers in DF architectures, showing that DF improvements mostly rely on the first layers of the architecture.
To gain insights about the quality of the new features created by the first layer, we consider a shallow tree network for which we evaluate the performance as a function of the first-tree depth. 

\subsection{Towards DF simplification}
\label{sec:toward_simplification}
\paragraph{Setting.} We compare the performances of the following DF architectures on the datasets summarized in Table \ref{tab:dataset_description}: 
\begin{enumerate}[(i)]
\item  the default setting of DF, described in Section~\ref{sec:deep_forest}; 
\item the best DF architecture obtained by grid-searching over the number of forests per layer, the number of trees per forest and the maximum depth of each tree. \replace{}{The selected architecture is chosen with respect to the performances achieved on validation datasets};
\item  a new light DF architecture, composed of 2 layers, 2 forests per layer (one RF and one CRF) with only 50 trees of depth 30 trained only once;
\item the first layer of the best DF;
\item the first layer of the light DF;
\item a ``Flattened best DF as RF" which consists in one RF with as many trees as in the best DF with similar forest parameters (refer to Supplementary Materials \ref{app:best_config_details} and Table \ref{tab:best_configurations} for details);
\item a ``Flattened light DF as RF" which corresponds to one RF with as many trees as in the light DF with similar forest parameters.
\end{enumerate}

\paragraph{Results.} Results are presented in Figures \ref{fig:DFvsLight_clas} and \ref{fig:DFvsLight_reg}.
 Each bar plot respectively corresponds to the average accuracy or the average $R^2$ score over 10 tries for each test dataset; the error bars stand for accuracy or $R^2$ standard deviation.
 The description of the resulting best DF architecture for each dataset is given in Table \ref{tab:best_configurations} (Supplementary Materials). 
\begin{figure}[h]
    \centering
        \includegraphics[width = 0.43\textwidth]{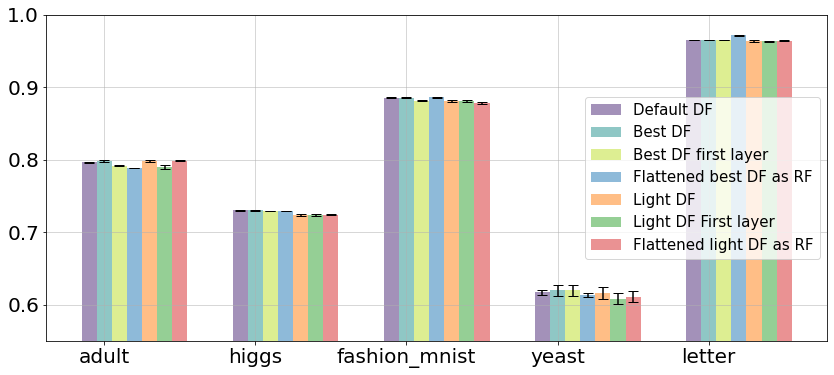}
    \caption{Accuracy of different DF architectures for classification datasets (10 runs per bar plot).}
    \label{fig:DFvsLight_clas}
\end{figure}
\begin{figure}[h]
    \centering
    \vspace{-0.3cm}
    \includegraphics[width = 0.43\textwidth]{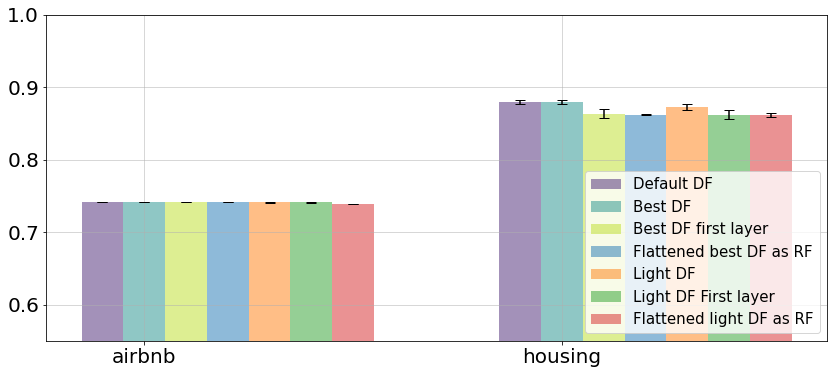}
    \caption{$R^2$ score of different DF architectures for regression datasets (10 runs per bar plot).}
    \label{fig:DFvsLight_reg}
\end{figure}

\replace{}{
As highlighted in Figure \ref{fig:DFvsLight_clas}, the performance of the light configuration for classification datasets is comparable to the default and the best configurations' one, while being much more computationally efficient: faster to train, faster at prediction, cheaper in terms of memory (see Table \ref{tab:mem_time_light_df} in the Supplementary Materials for a comparison of computing time and memory consumption). Moreover, except on the Letter dataset, the DF performs better than its RF equivalent. The results for the Letter dataset can be explained by the fact that the CRFs within the DF are outperformed by Breiman RFs in this specific case.
Overall, for classification tasks, the small performance enhancement of Deep Forests (Default or Best DF) over our light configuration should be assessed in the light of their additional complexity. This questions the usefulness of stacking several layers made of many forests, resulting in a heavy architecture. }
We further propose an in-depth analysis of the role of each layer to the global DF performance. 

\subsection{Tracking the best sub-model}
\label{subsec:optimal_layer}

\paragraph{Setting.} On all the previous datasets, we train a DF architecture by specifying the \replace{}{maximal} number $p$ of layers. Unspecified hyper-parameters are set to default value (see Section \ref{sec:deep_forest}). For each $p$, we consider the truncated sub-models composed of layer $1$, layer $1$-$2$, $\hdots$, layer $1$-$p$, where layer $1$-$p$ is the original DF with $p$ layers. For each value of $p$, we consider the previous nested sub-models with $1, 2, \hdots, p$ layers, and compute the predictive accuracy of the best sub-model. 


\paragraph{Results.} We only display results for the Adult dataset in Figure~\ref{fig:acc_depth} (all the other datasets show similar results, see Section~\ref{app:subsec:optimal_layer} of the Supplementary Materials). \replace{}{The score (accuracy or $R^2$-score) corresponds to the result on the test dataset.} 
We observe that adding layers to the Deep Forest does not significantly change the accuracy score. Even if the variance changes by adding layers, we are not able to detect any pattern, which suggests that the variance of the procedure performance is unstable with respect to the \replace{}{maximal} number of layers.

\begin{figure}[h]
    \centering
    \includegraphics[width = 0.49\linewidth]{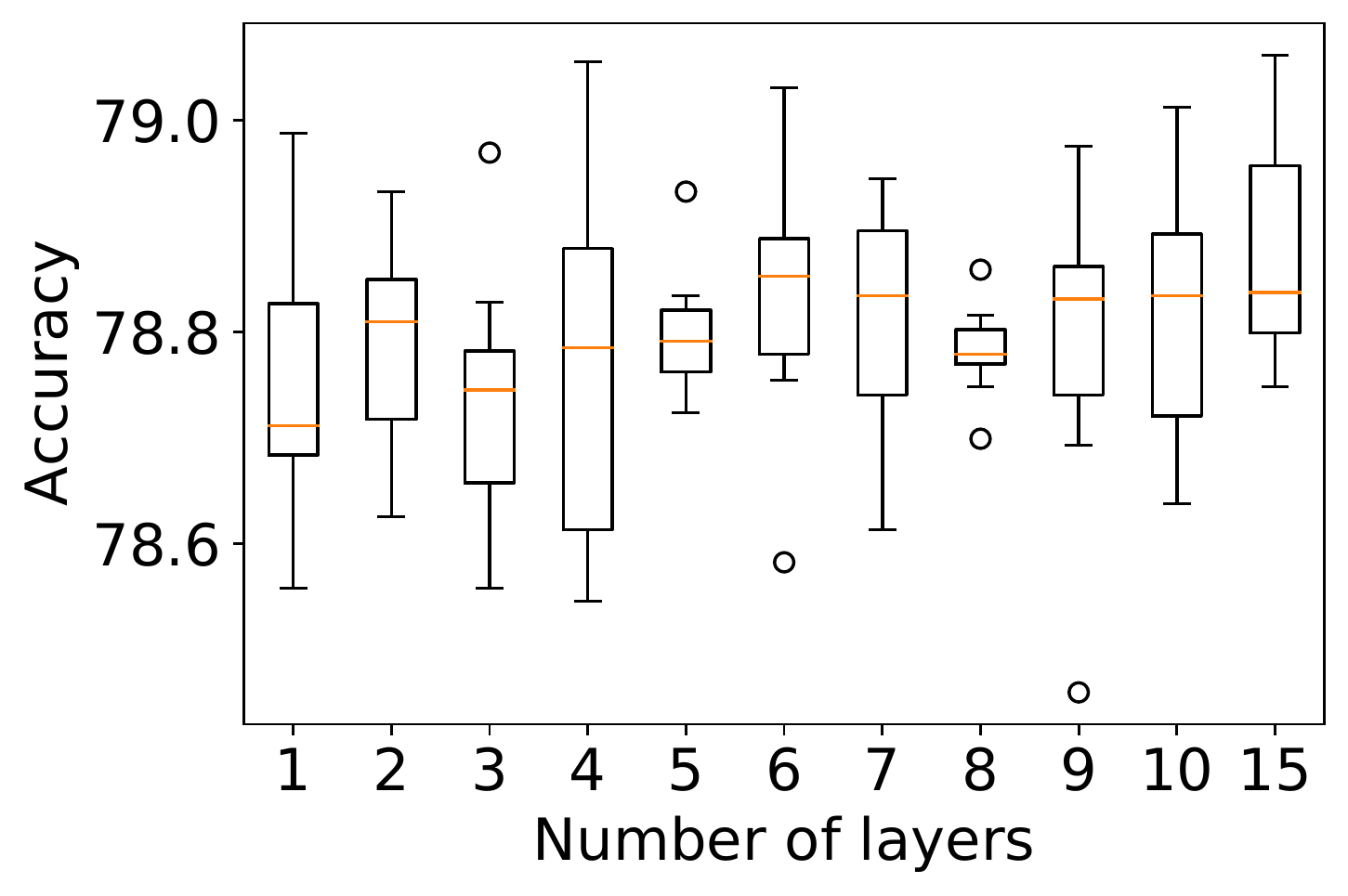}
    \includegraphics[width = 0.49\linewidth]{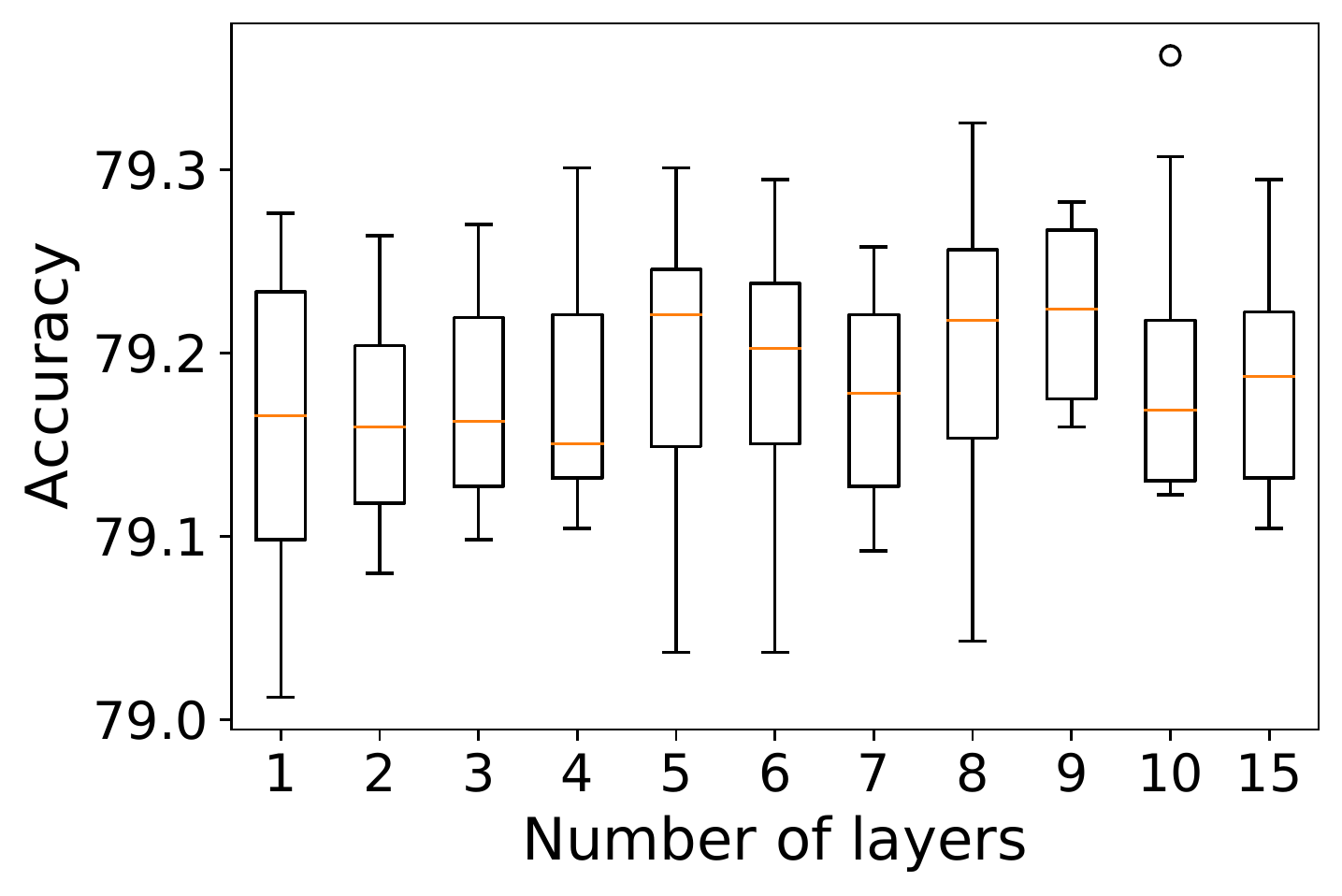}
    \caption{Adult dataset. Boxplots over 10 runs of the accuracy of a DF sub-model with 1 (Breiman) forest by layer (left) or 4 forests (2 Breiman, 2 CRF) by layer (right), depending on the maximal number of layers of the global DF model.}
    \label{fig:acc_depth}
\end{figure}

\begin{figure}[h]
    \centering
    \includegraphics[width = 0.49\linewidth]{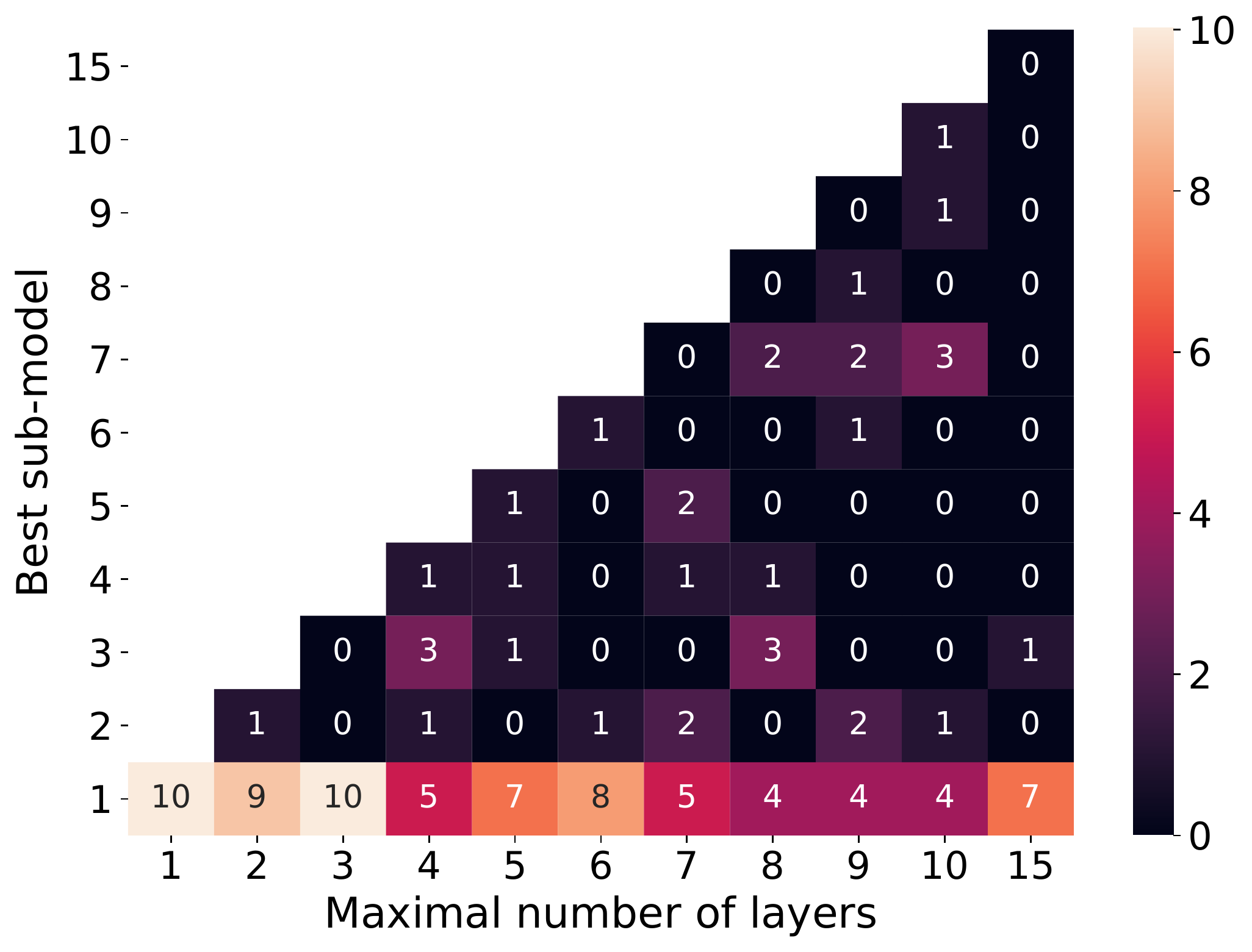}
    \includegraphics[width = 0.49\linewidth]{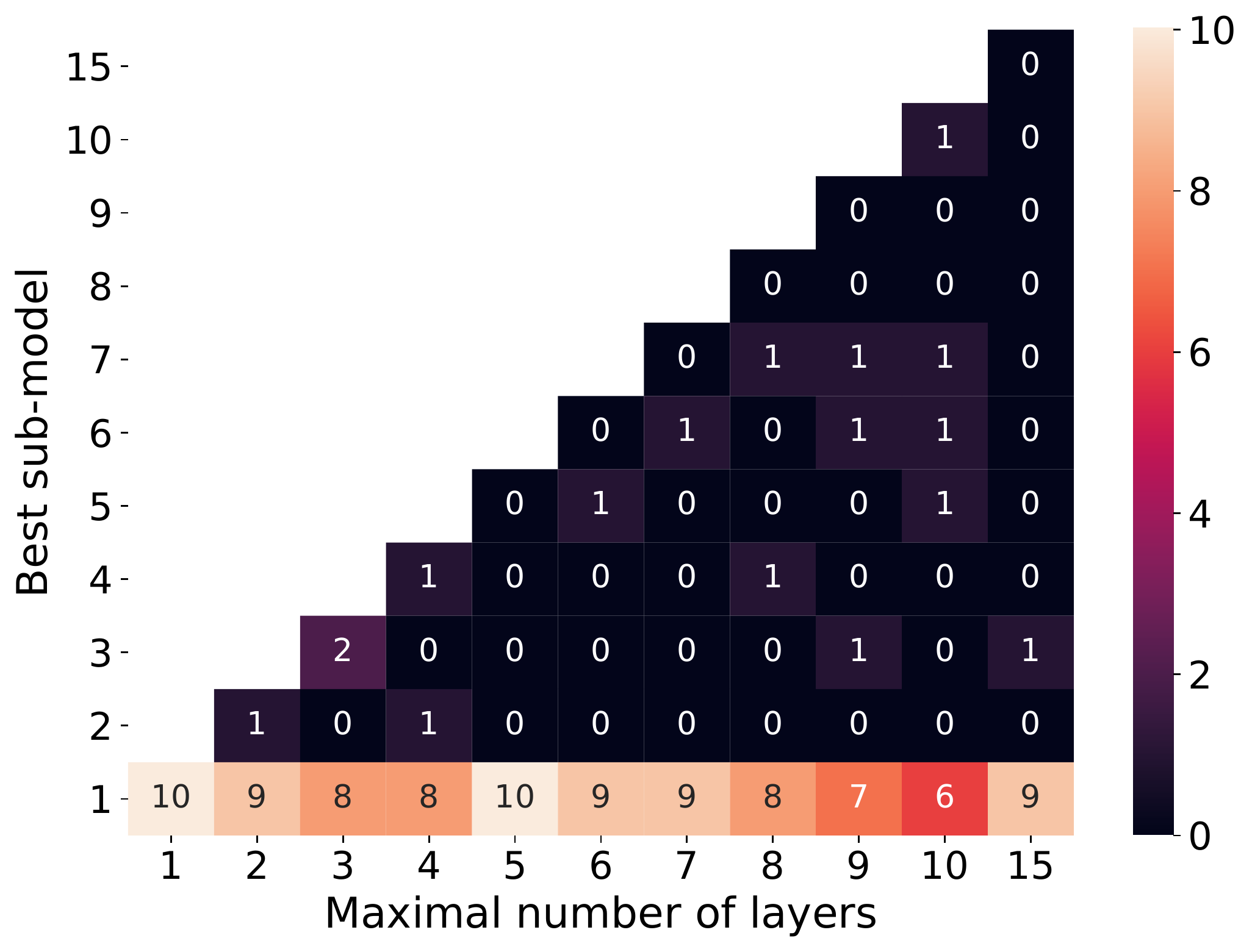}
    \caption{Adult dataset. Heatmap counting the optimal layer index over 10 tries of a default DF with 1 (Breiman) forest per layer (left) or 4 forests (2 Breiman, 2 CRF) per layer (right), with respect to the maximal  number of layers. The number corresponding to $(n,m)$ on the x- and y-axes indicates how many times out of 10 the layer $m$ is optimal when running a cascade network with a maximal number $n$ of layers.}
    \label{fig:heatmap}
\end{figure}

Globally, we observe that the sub-models with one or two layers often lead to the best performance (see Figure \ref{fig:heatmap} for the Adult dataset and Supplementary Materials  \ref{app:subsec:optimal_layer}). When the dataset is small (Letter or Yeast), the sub-model with only one layer (i.e.\ a standard RF \replace{}{or an aggregation of RFs}) is almost always optimal since a single RF with no maximum depth constraint already overfits on most of these datasets. Therefore the second layer, building upon the predictions of the first layer, entails overfitting as well, therefore leading to no improvement of the overall model.
Besides, one can explain the pre\-do\-mi\-nan\-ce of small sub-models by the weak additional flexibility created by each layer: on the one hand, each new feature vector size corresponds to the number of classes times the number of forests which can be small with respect to the number of input features; on the other hand, the different forests within one layer are likely to produce similar probability outputs, especially if the number of trees within each forest is large. 
The story is a little bit different for the Housing dataset, for which the best submodel is between 2 and 6. As noticed before, this may be the result of the frustratingly simple representation of the new features created at each layer. Eventually, these numerical experiments cor\-ro\-bo\-ra\-te the relevance of shallow DF as the light configuration proposed in the previous section. 

We note that adding forests in each layer decreases the number of layers needed to achieve a pre-specified performance. This is surprising and is opposed to the common belief that in Deep Neural Networks, adding layers is usually better than adding neurons in each layer.

We can conclude from the empirical results that the first two layers convey the performance enhancement in DF. Contrary to NNs, depth is not an important feature of DFs. The following studies thus focus on two-layer architectures which are deep enough to reproduce the improvement of deeper architectures over single RFs.

\subsection{A precise understanding of depth enhancement}
\label{sec:shallowCART}

In order to finely grasp the influence of tree depth in DF, we study a simplified version: a shallow CART tree network, composed of two layers, with one CART per layer. 

\paragraph{Setting.} In such an architecture, the first-layer tree is fitted on the training data. For each sample, the first-layer tree outputs a probability distribution (or a value in a regression setting), which is referred to as ``encoded data" and given as input to the second-layer tree, with the raw features as well.
For instance, considering binary classification data with classes 0 and 1, with raw features $(x_1,x_2,x_3)$, the input of the second-layer tree is a 5-dimensional feature vector $(x_1, x_2, x_3, p_0, p_1)$, with $p_0$ (resp.\ $p_1$) the predicted probabilities by the first-layer tree for the class $0$ (resp.\ $1$).

For each dataset of Table \ref{tab:dataset_description}, we first determine the optimal depth $k^\star$ of a single CART tree via 3-fold cross validation. 
Then, for a given first-layer tree with a fixed depth, we fit a second-layer tree, allowing its depth to vary. 
We then compare the resulting shallow tree networks in three different cases: when the (fixed) depth of the first tree is (i) less than $k^\star$, (ii) equal to $k^\star$, and (iii) larger than $k^\star$. 
We add the optimal single tree performance to the comparison.

\begin{figure}[h!]
    \centering
    \includegraphics[width = 0.45\textwidth]{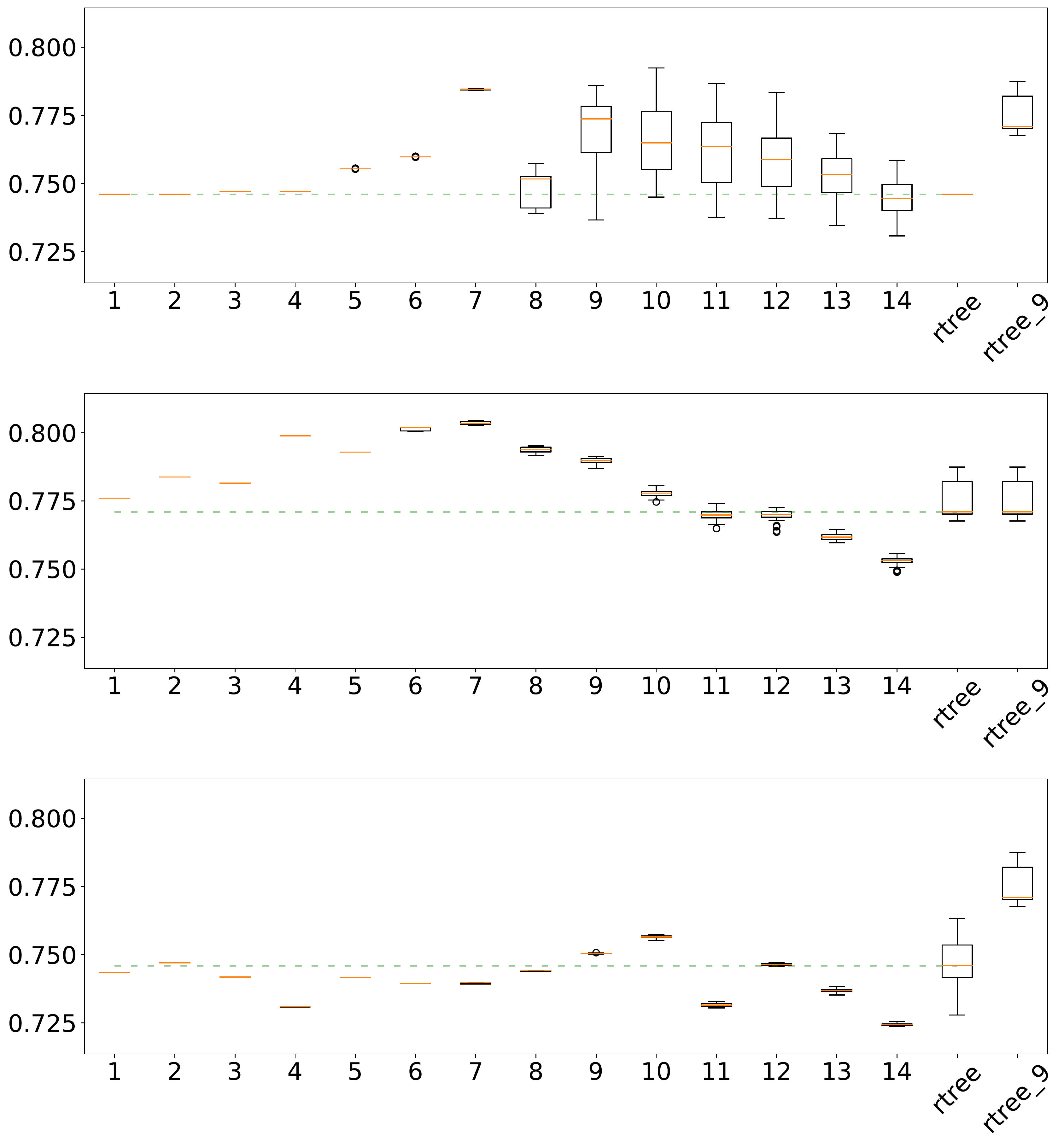}
    \caption{Adult dataset. Accuracy \replace{}{on the test dataset} of a two-layer tree architecture w.r.t.\ the second-layer tree depth, when the first-layer (encoding) tree  is of depth 2 (top), 9 (middle), and 15 (bottom). \texttt{rtree} is a single tree of respective depth 2 (top), 9 (middle), and 15 (bottom),  applied on raw data. For this dataset, the optimal depth of a single tree is 9 and the tree with the optimal depth is depicted as \texttt{rtree\_9} in each plot.
    The green dashed line indicates the median score of the \texttt{rtree}. All boxplots are obtained by 10 different runs.}
    \label{fig:tree_encoding_influence2}
\end{figure}

\paragraph{Results.}  Results are displayed in Figure  \ref{fig:tree_encoding_influence2}
for the {Adult dataset} only
(see Supplementary Materials~\ref{app:sec:shallowCART} for the results on the other datasets).
Specifically noticeable in Figure \ref{fig:tree_encoding_influence2} (top), the tree network architecture can introduce performance instability when the second-layer tree grows (e.g.\ when the latter is successively of depth 7, 8 and 9).

Furthermore, when the encoding tree is not deep enough (top), the second-layer tree improves the accuracy until it approximately reaches the optimal depth $k^\star$. In this case, the second-layer tree compensates for the poor encoding, but cannot improve over a single tree with optimal depth $k^\star$.
Conversely, when the encoding tree is more developed than an optimal single tree (bottom) {- overfitting regime}, the second-layer tree may not lead to any improvement, or worse,  may degrade the performance of the first-layer tree. 
On all datasets, the second-layer tree is observed to always make its first cut over the new features (see Figure~\ref{fig:adult_etree_optimal_ZOOM} and Supplementary Materials).
%
%
In the case of binary classification, a single cut of the second-layer tree along a new feature yields to gather all the leaves of the first tree, predicted respectively as $0$ and $1$, into two big leaves, therefore reducing the predictor variance (cf.\ Figure \ref{fig:tree_encoding_influence2} (middle and bottom)). 
Furthermore, when considering multi-label classification with $n_{\textrm{classes}}$, the second-layer tree must cut over at least $n_{\textrm{classes}}$ features to recover the partition of the first tree (see Figure \ref{fig:letter_encoding_influence}). Similarly, in the regression case, the second tree needs to perform a number of splits equal to the number of leaves of the first tree in order to recover the partition of the latter. 
\begin{figure}[h!]
    \centering
    \includegraphics[width = 0.4\textwidth]{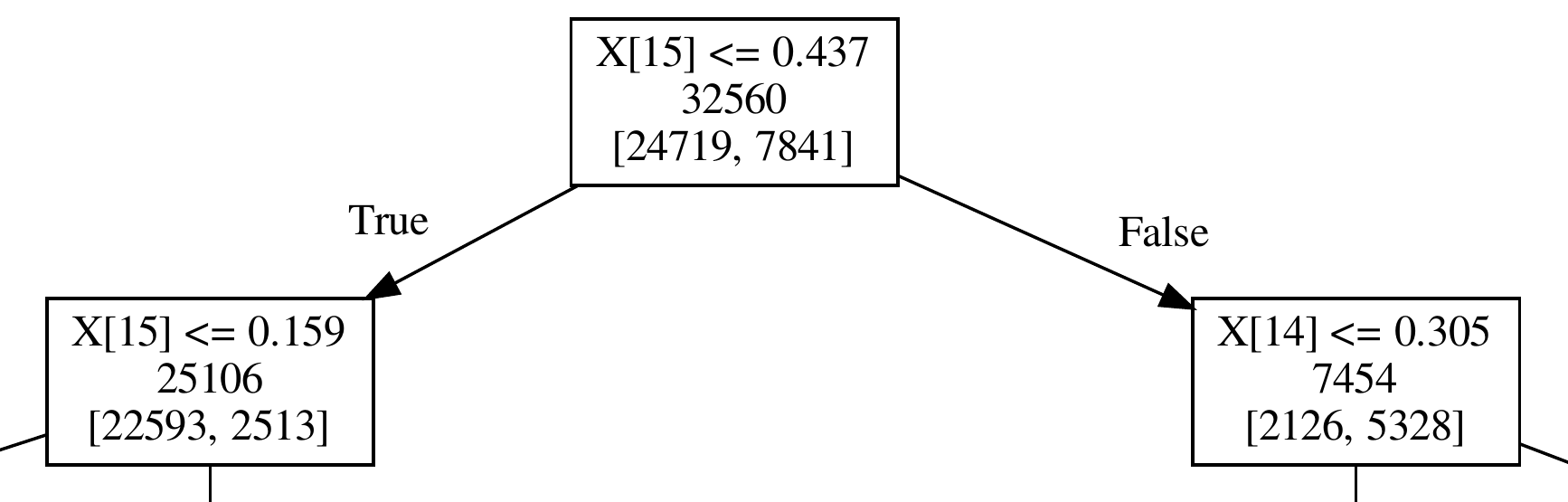}
    \caption{Adult dataset. Focus on the first levels of the second-layer tree structure when the first layer tree is of depth 9 (optimal depth). Raw features range from X[0] to X[13], X[14] and X[15] are the features built by the first-layer tree.}
    \label{fig:adult_etree_optimal_ZOOM}
\end{figure}

In Figure \ref{fig:tree_encoding_influence2} (middle), one observes that with a first-layer tree of optimal depth, the second-layer tree may outperform an optimal single tree, by improving both the average accuracy and its variance. We aim at theoretically quantifying this performance gain in the next section.

\section{Theoretical study of a shallow tree network}
\label{sec:theoretical_results}

In this section, we focus on the theoretical analysis of a simplified tree network. Our aim is to exhibit settings in which a tree network outperforms a single tree. Recall that the second layer of a tree network gathers tree leaves of the first layer with similar distributions. For this reason, we believe that a tree network is to be used when the dataset has a very specific structure, in which the same link between the input and the output can be observed in different subareas of the input space.  Such a setting is described in Section~\ref{sec:pb_setting}

To make the theoretical analysis possible, we study centered trees (see Definition~\ref{def:shallow_tree_network}) instead of CART. Indeed, studying the original CART algorithm is still nowadays a real challenge and analyzing stacks of CART seems out-of-reach in short term. As highlighted by the previous empirical analysis, we believe that the results we establish theoretically are shared by DF. All proofs are postponed to the Supplementary Materials. 



\subsection{The network architecture}

We assume to have access to a dataset $\mathcal{D}_n = \{(X_1, Y_1),$ $ \hdots,(X_n,Y_n)\}$  of i.i.d.\ copies of the generic pair $(X,Y)$  \replace{}{with $X$ living in $[0,1]^d$ and $Y \in \{0,1\}$ being the label associated to $X$}.

\paragraph{Notations.}
Given a decision tree, we denote by $\cellof{X}$ the leaf of the tree containing $X$ and $\cardcellof{X}$ the number of data points falling into $\cellof{X}$. The prediction of such a tree at point $X$ is given by 
$$
\hat{r}_n(X) = \frac{1}{\cardcellof{X}} \displaystyle \sum_{X_i \in \cellof{X}} Y_i$$
with the convention $0/0=0$, i.e. the prediction for $X$ in a leaf with no observations is arbitrarily set to zero. 

\paragraph{A shallow centered tree network.} We want to theoretically analyze the benefits of stacking trees. To do so, we focus on two trees in cascade and will try to determine, in particular, the influence of the first (encoding) tree on the performance of the whole tree network. 
To catch the variance reduction property of tree networks already emphasized in the previous section, we consider a regression setting:
let $r(x) = \mathbb{E}[Y |X=x]$ be the regression function and for any function $f$, its quadratic risk is defined as $R(f) = \mathbb{E}[(f(X) - r(X))^2],$ where the expectation is taken over $(X,Y, \mathcal{D}_n)$.

\begin{definition}[Shallow centered tree network] 
\label{def:shallow_tree_network}The shallow tree network consists in two trees in cascade:
\begin{itemize}
    \item \textbf{(Encoding layer)} The first-layer tree is a \emph{cycling centered tree of depth $k$}. It is built independently of the data by splitting recursively on each variable, at the center of the cells. \replace{}{The first cut is made along the first coordinate,  the second along the second coordinate, etc. The tree construction is stopped when exactly $k$ cuts have been made}. 
    For each point $X$, we extract the empirical mean $\bar{Y}_{\cellof{X}}$ of the outputs $Y_i$ falling into the leaf $\cellof{X}$ and we pass the new feature $\bar{Y}_{\cellof{X}}$ to the next layer, together with the original features $X$. 
    \item \textbf{(Output layer)} The second-layer tree is a \emph{centered tree of depth $k'$} for which a cut can be performed at the center of a cell along a raw feature  (as done by the encoding tree) or along the new feature  $\bar{Y}_{\cellof{X}}$. In this latter case, two cells corresponding to $\{\bar{Y}_{\cellof{X}} < 1/2\}$ and $\{\bar{Y}_{\cellof{X}} \geq 1/2\}$ are created.
\end{itemize}
The resulting predictor composed of the two trees in cascade, of respective depth $k$ and $k'$,  trained on the data $(X_1,Y_1), \hdots, (X_n, Y_n)$ is denoted by $\hat{r}_{k,k',n}$. 
\end{definition}

The two cascading trees can be seen as two layers of trees, hence the name of the shallow tree network. 
Note in particular that $\hat{r}_{k,0,n}(X)$ is the prediction given by the first encoding tree only and  outputs, as a classical tree, the mean of the $Y_i$'s falling into a leaf containing $X$. 

\subsection{Problem setting}
\label{sec:pb_setting}

\paragraph{Data generation.} 
The data $X$  is assumed to be uniformly distributed over \replace{}{ $[0,1]^d$} and $Y \in \{0,1\}$. 
Let $k^{\star}$ be \replace{an even integer}{a multiple of $d$} and let $p \in (1/2,1]$. 
\replace{For all $i,j \in \{1, \hdots, 2^{k^{\star}/2}\}$, we denote $C_{ij}$ the cell $  \left[\frac{i-1}{2^{k^{\star}/2}},\frac{i}{2^{k^{\star}/2}}\right) \times \left[\frac{j-1}{2^{k^{\star}/2}},\frac{j}{2^{k^{\star}/2}}\right)$. }{
We \replace{}{build a regular partition of the space with \textit{cells}} $C_1,\hdots, C_{\expkstar{}}$ of generic form
$$
\prod_{k=1}^d \left[\frac{i_k}{2^{k^{\star}/d}},\frac{i_k+1}{2^{k^{\star}/d}}\right),
$$ for $i_1,...,i_d \in \{0,...,2^{k^\star/d} -1\}$.}
We arbitrary assign a color (black or white) to each cell, which has a direct influence on the distribution of $Y$ in the cell. More precisely, for $x$ in a given cell $C$,
\begin{align}
\mathbb{P}[Y=1|X=x] =  \left\lbrace \begin{array}{cc}
   p  &  \textrm{if $C$ is a black cell, }\\
   1-p  & \textrm{if $C$ is a white one. }
\end{array}\right. \label{eq:binary_quantif}
\end{align}
We define $\mathcal{B}$ (resp. $\mathcal{W}$) as the union of black (resp. white) cells and $N_\mathcal{B} \in \{ 0 ,\hdots , 2^{k^\star}\}$ (resp. $N_\mathcal{W}$) as the number of black (resp. white) cells. Note that $N_\mathcal{W} = 2^{k^\star}-N_\mathcal{B}$. The location and the numbers of the black and white cells are arbitrary.  
This distribution corresponds to a \emph{generalized chessboard} structure. 
The whole distribution is thus pa\-ra\-me\-te\-ri\-zed by $k^{\star}$ ($2^{k^{\star}}$ is the total number of cells), $p$ and $N_\mathcal{B}$. Examples of this distribution are depicted in Figures \ref{fig:random_chessboard} and \ref{fig:Damier}  for different configurations and $d=2$. 

\begin{figure}
     \centering
     \begin{subfigure}[b]{0.14\textwidth}
         \centering
         \includegraphics[width=\textwidth]{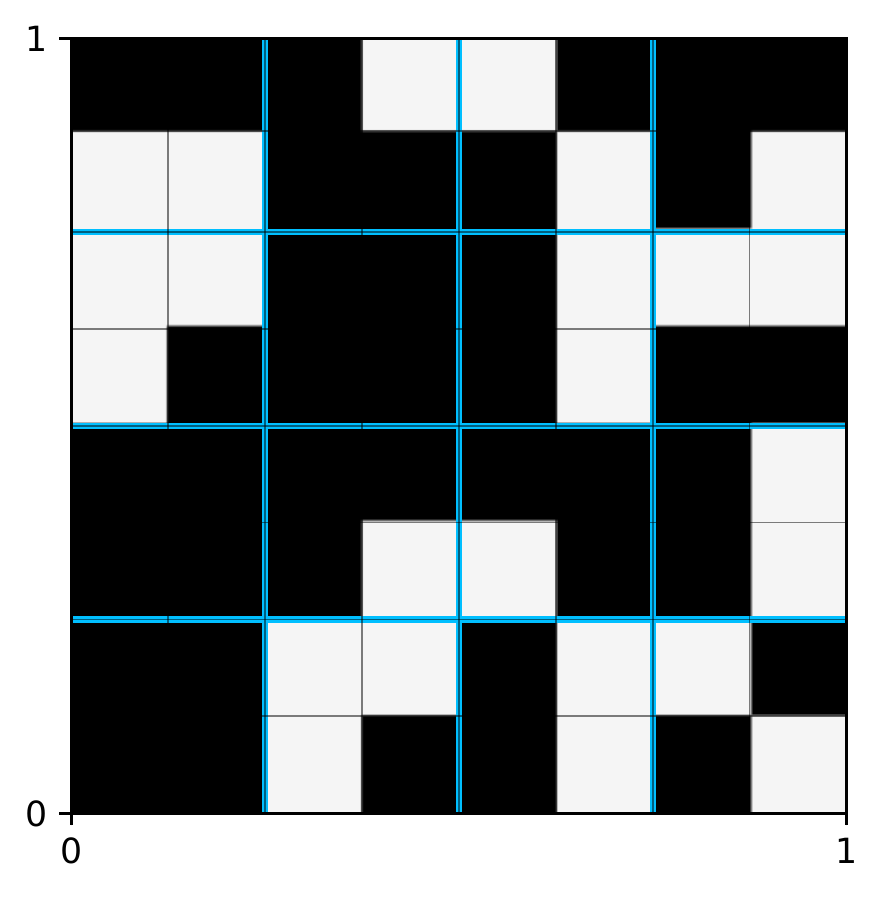}
         \caption{Depth 4}
         \label{fig:random_damier_small_depth}
     \end{subfigure}
     \hspace{0.05cm}
     \begin{subfigure}[b]{0.14\textwidth}
         \centering
         \includegraphics[width=\textwidth]{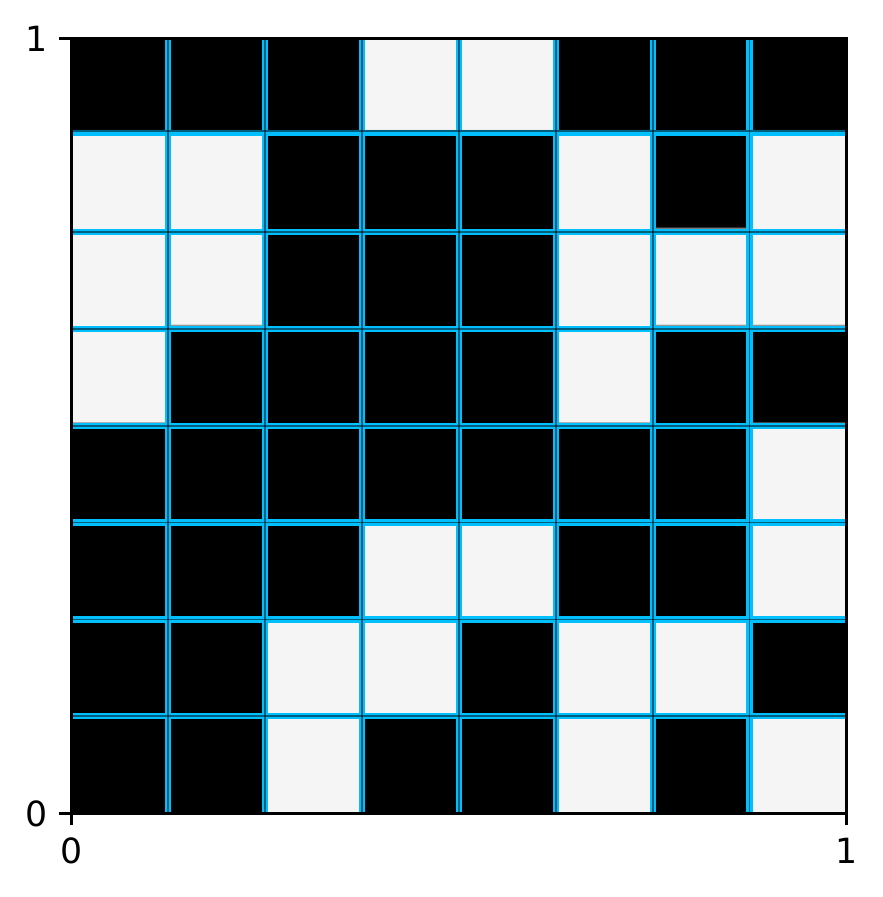}
         \caption{Depth 6}
         \label{fig:random_damier_perfect_depth}
     \end{subfigure}
     \hspace{0.05cm}
     \begin{subfigure}[b]{0.14\textwidth}
         \centering
         \includegraphics[width=\textwidth]{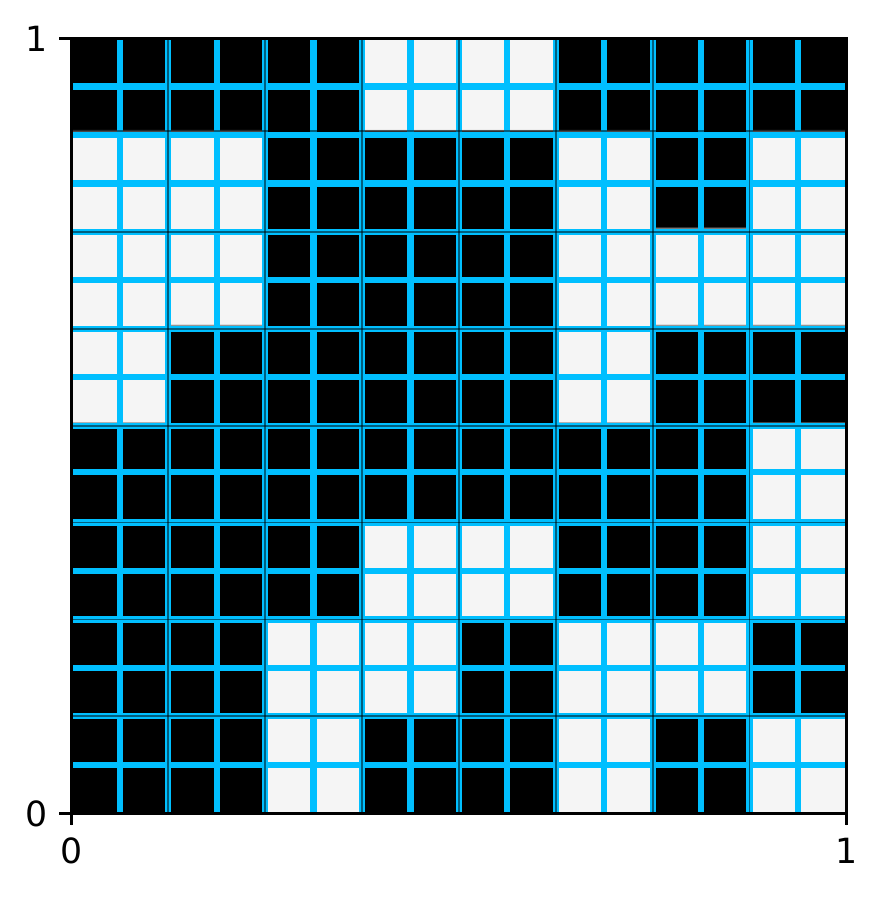}
         \caption{Depth 8}
         \label{fig:random_damier_big_depth}
     \end{subfigure}
        \caption{Arbitrary chessboard data distribution for $k^{\star} = 6$ and $N_\mathcal{B}=40$ black cells ($p$ is not displayed here). Partition of the (first) encoding tree of depth 4, 6, 8 (from left to right) is displayed in blue. The optimal depth of a single centered tree for this chessboard distribution is 6.}
        \label{fig:random_chessboard}
\end{figure}


\paragraph{Why such a structured setting?} 
The data distribution introduced above is highly structured, which can be seen as a restrictive study setting. \replace{Outside this specific framework, it seems difficult for shallow tree networks to improve over a single tree. For instance,}{However, the generalized chessboard is nothing but a discretized quantification of the regression function $r$ using only 2 values (see Equation~\eqref{eq:binary_quantif}). 
Going further than quantification towards general discretization does not seem appropriate for tree networks. To see this,
}
consider a more general distribution such as 
\begin{align*}
\mathbb{P}[Y=1|X=x] = 
   P_{ij}  &  \textrm{ when $x\in C_{ij}$},
\end{align*}
where $P_{ij}$ is a random variable drawn uniformly in $[0,1]$.

\begin{lemma}
\label{lem:counter_example}
Consider the previous setting with $k \geq k^{\star}$. In the infinite sample setting, the risks of a single tree and a shallow tree network are given by $R(\hat{r}_{k,0,\infty}) = 0$ and 
\begin{align*}
     R(\hat{r}_{k,1,\infty})  \geq  \frac{1}{48} \Big(  1 - \frac{8}{2^{k^{\star}}-1}\Big)
     + \frac{1}{2^{2^{k^{\star}}}} \frac{9}{24}. 
\end{align*}
\end{lemma}
Lemma~\ref{lem:counter_example} highlights the fact that a tree network has a positive bias, which is not the case for a single tree. Besides, by letting $k^{\star}$ tend to infinity (that is the size of the cells tends to zero), the above chessboard distribution boils down to a very generic classification framework. In this latter case, the tree network performs poorly since its risk is lower bounded by $1/48$. 
In short, when the data distribution is disparate across the feature space, the averaging performed by the second tree leads to a biased regressor.  
Note that Lemma \ref{lem:counter_example} involves a shallow tree network, performing only one cut on the second layer. But similar conclusions could be drawn for a deeper second-layer tree, until its depth reaches $k^\star$. Indeed, considering $\hat{r}_{k,k^\star,\infty}$ would result in an unbiased regressor, with comparable performances as of a single tree, while being much more complex.

Armed with Lemma \ref{lem:counter_example}, we believe that the intrinsic structure of DF and tree networks makes them useful to detect similar patterns spread across the feature space. This makes the generalized chessboard distribution particularly well suited for analyzing such behavior.
The risk of a shallow tree network in the infinite sample regime for the generalized chessboard distribution is studied in Lemma~\ref{lem:depth1}.
\begin{lemma}
\label{lem:depth1}
Assume that the data follows the generalized chessboard distribution described above with parameter $k^\star$, $N_\mathcal{B}$ and $p$.
In the infinite sample regime, the following holds for the shallow tree network $\hat{r}_{k,k',n}$ (Definition \ref{def:shallow_tree_network}). 
\begin{enumerate}[(i)]
    \item \textbf{Shallow encoding tree.} Let $k<k^\star$. The risk of the shallow tree network is minimal for all configurations of the chessboard if the second-layer tree is of depth $k'\geq k^{\star}$ and if the $k^{\star}$ first cuts are performed along raw features only. 
    \item \textbf{Deep encoding tree.} Let $k \geq k^\star$. The risk of the shallow tree network is minimal for all configurations of the chessboard if the second-layer tree is of depth $k'\geq 1$ and if the first cut is performed along the new feature $\bar{Y}_{\cellof{X}}$.
\end{enumerate}
\end{lemma}
In the infinite sample regime, Lemma~\ref{lem:depth1} shows that the pre-processing is useless when the encoding tree is shallow ($k<k^\star$): the second tree cannot leverage on the partition of the first one and needs to build a finer partition from zero.

Lemma~\ref{lem:depth1} also provides an interesting perspective on the second-layer tree which either acts as a copy of the first-layer tree or can simply be of depth one. 

\begin{remark}
The results established in Lemma~\ref{lem:depth1} for centered-tree networks also empirically hold for CART ones (see Figures \ref{fig:tree_encoding_influence2},\ref{fig:higgs_encoding_influence},\ref{fig:letter_encoding_influence},\ref{fig:yeast_encoding_influence},\ref{fig:airbnb_encoding_influence},\ref{fig:housing_encoding_influence}: $(i)$ the second-layer CART trees always make their first cut on the new feature and always near 1/2; $(ii)$ if the first-layer CART is biased, then the second-layer tree will not improve the accuracy of the first tree (see Figure \ref{fig:tree_encoding_influence2} (top)); $(iii)$ if the first-layer CART is developed enough, then the second-layer CART acts as a variance reducer (see Figure \ref{fig:tree_encoding_influence2}, middle and bottom). 
\end{remark}

\subsection{Main results}

Building on Lemma~\ref{lem:counter_example} and  \ref{lem:depth1}, we now focus on a shallow network whose second-layer tree is of depth one, and whose first cut is performed along the new feature $\bar{Y}_{\cellof{X}}$ at $1/2$. Two main regimes of training can be therefore identified when the first tree is either shallow  ($k < k^{\star}$) or deep  ($k \geq k^{\star}$). 


In the first regime ($k < k^{\star}$), to establish precise non-asymptotics bounds, we study the balanced chessboard distribution (see Figure~\ref{fig:Damier}). 
Such a distribution has been studied in the unsupervised literature, in order to generate distribution for $X$ via copula theory \cite{ghosh2002chessboard,ghosh2009patchwork} or has been mixed with other distribution in the RF framework  \cite{biau2008consistency}.
Intuitively, this is a worst-case configuration  for centered trees in terms of bias. Indeed, if $k< k^\star$, each leaf contains the same number of black and white cells. Therefore in expectation the mean value of the leaf is $1/2$ which is non informative. 

\begin{figure}
     \centering
     \begin{subfigure}[b]{0.14\textwidth}
         \centering
         \includegraphics[width=\textwidth]{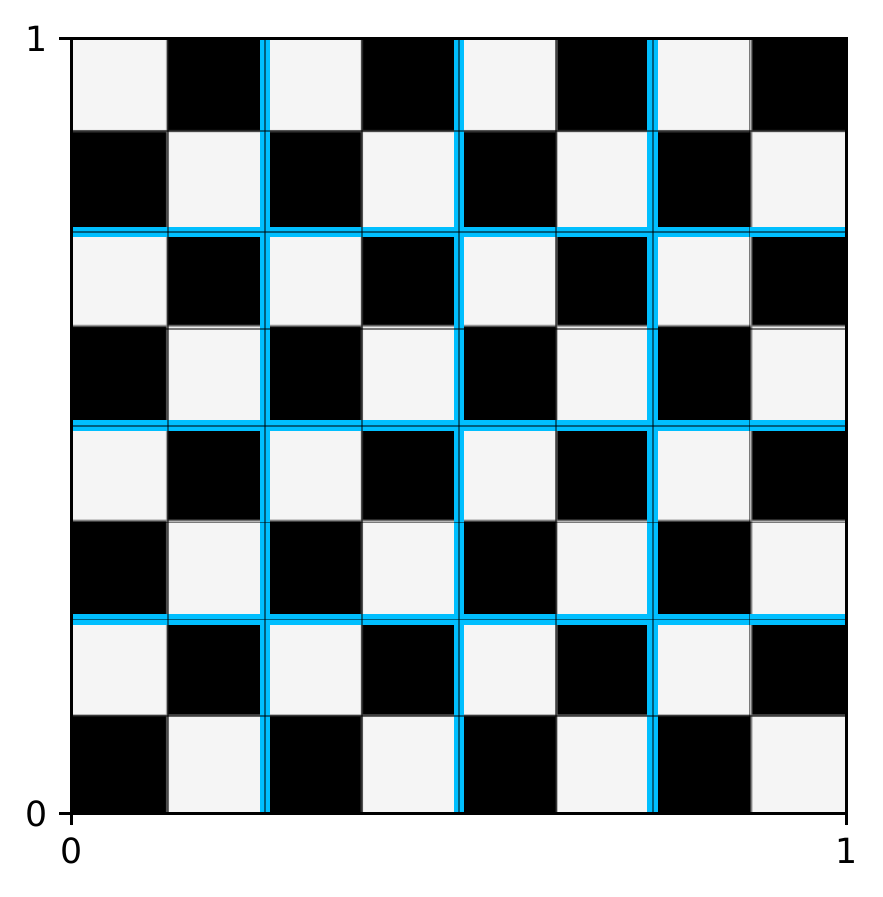}
         \caption{Depth 4}
         \label{fig:damier_small_depth}
     \end{subfigure}
     \hspace{0.05cm}
     \begin{subfigure}[b]{0.14\textwidth}
         \centering
         \includegraphics[width=\textwidth]{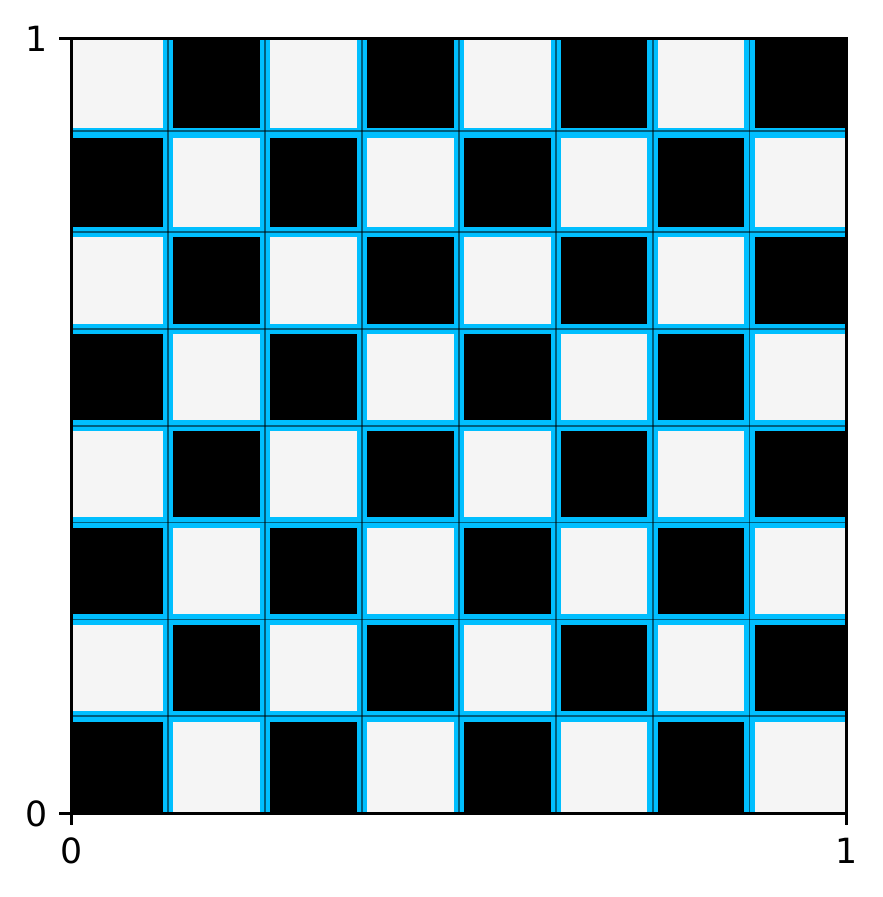}
         \caption{Depth 6}
         \label{fig:damier_perfect_depth}
     \end{subfigure}
     \hspace{0.05cm}
     \begin{subfigure}[b]{0.14\textwidth}
         \centering
         \includegraphics[width=\textwidth]{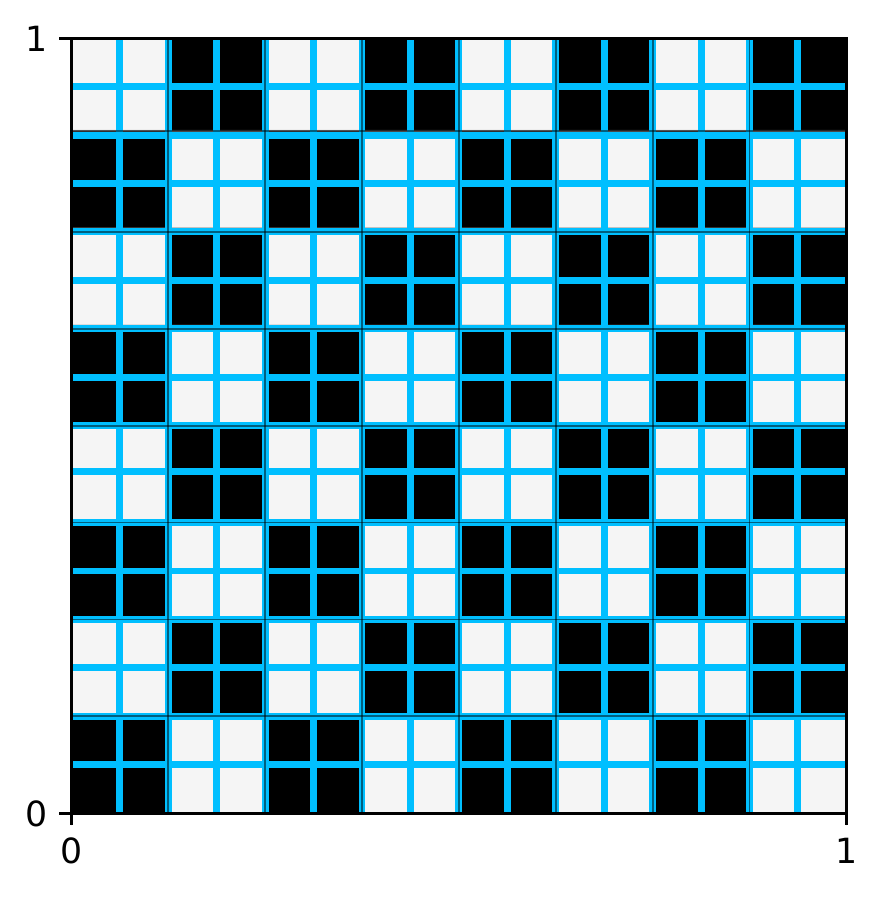}
         \caption{Depth 8}
         \label{fig:damier_big_depth}
     \end{subfigure}
        \caption{Chessboard data distribution for $k^{\star} = 6$ and $N_{\mathcal{B}}=2^{k^\star-1}$. Partition of the (first) encoding tree of depth 4, 6, 8 (from left to right) is displayed in blue. The optimal depth of a single centered tree for this chessboard distribution is 6.}
        \label{fig:Damier}
\end{figure}

\begin{prop}[Risk of a single tree and a shallow tree network when $k<k^{\star}$]
\label{prop:balanced_chessboard}
Assume that the data is drawn according to a balanced chessboard distribution with parameters $k^{\star}$, $N_\mathcal{B}=2^{k^\star-1}$ and $p>1/2$ (see Figure \ref{fig:Damier}). 
\begin{enumerate}
    \item Consider a single tree $\hat{r}_{k,0,n}$ of depth $k \in \mathbb{N}^\star$. We have,
    \begin{align*}
         R(\hat{r}_{k,0,n}) 
        &\leq \left(p-\frac{1}{2}\right)^2  + \frac{2^{k}}{2(n+1)} +   \frac{(1-2^{-k})^n}{4}; 
    \end{align*}
    and
    \begin{align*}
        R(\hat{r}_{k,0,n})
        & \geq  \left(p-\frac{1}{2}\right)^2  + \frac{2^k}{4(n+1)} \\
        & \quad +   \frac{(1-2^{-k})^n}{4} \left( 1 - \frac{2^k}{n+1} \right).
    \end{align*}
    \item Consider the shallow tree network $\hat{r}_{k,1,n}$. We have
        \begin{align*}
         & \hspace{-0.5cm} R (\hat{r}_{k,1,n})  \leq  \left( p - \frac{1}{2} \right)^2 + \frac{2^{k/2+3} (p-\frac{1}{2})}{\sqrt{\pi n}} \\
        & \hspace{-0.5cm} + \frac{7 \cdot 2^{2k+2}}{\pi^2 (n+1)} (1 + \varepsilon_{k,p}) +\frac{p^2+(1-p)^2}{2} \left(1-2^{-k}\right)^n
        \end{align*} 
        where 
        $\varepsilon_{k,p} = o(2^{-k/2})$ uniformly in $p$, and 
          $$R (\hat{r}_{k,1,n}) \geq \left( p-\frac{1}{2} \right)^2.$$
\end{enumerate}
\end{prop}
First, note that our bounds are tight in both cases ($k< k^{\star}$ and $k \geq k^{\star}$) since the rates of the upper bounds match that of the lower ones. 
The first statement in Proposition~\ref{prop:balanced_chessboard} quantifies the bias of a single tree of depth $k < k^{\star}$: the term $(p-1/2)^2$ appears in both the lower and upper bounds, which means that no matter how large the training set is, the risk of the tree does not tend to zero.
The shallow tree network suffers from the same bias term as soon as the first-layer tree is not deep enough. Here, the flaws of the first-layer tree transfer to the whole network. 
In all bounds, the term $(1-2^{-k})^n$ corresponding to the probability of $X$ falling into an empty cell is classic and cannot be eliminated for centered trees, whose splitting strategy is independent of the dataset. 

Proposition \ref{prop:new_risk_small_k} in the Supplementary Materials extends the previous result to the case of a random chessboard, in which each cell has a probability of being black or white. The same phenomenon is observed: the bias of the first layer tree is not reduced, even in the infinite sample regime. 

In the second regime ($k \geq k^{\star})$, the tree network may improve over a single tree as shown in Proposition \ref{prop:new_risk_large_k}.

\begin{prop}[Risk of a single tree and a shallow tree network when $k\geq k^\star$]
\label{prop:new_risk_large_k}
Consider a generalized chessboard with parameters ${k^\star}$, $N_\mathcal{B}$ and $p>1/2$. 
\begin{enumerate}
    \item Consider a single tree $\hat{r}_{k,0,n}$ of depth $k \in \mathbb{N}^\star$. We have
    \begin{align*}
    \hspace{-0.6cm}
     R(\hat{r}_{k,0,n}) \leq \frac{2^{k}p(1-p)}{n+1} 
    + \left(p^2 + (1-p)^2\right) \frac{(1-2^{-k})^n}{2}, 
    \end{align*}
    and
    \begin{align*}
        \hspace{-0.3cm}
         &R(\hat{r}_{k,0,n})  \geq  \frac{2^{k-1}p(1-p)}{n+1}  \\
        & + \Bigg(p^2 +  (1-p)^2  -  \frac{2^{k}p(1-p)}{n+1} \Bigg) \frac{(1-2^{-k})^n}{2}.
    \end{align*}
    \item Consider the shallow tree network $\hat{r}_{k,1,n}$. Letting  $$\bar{p}_{\mathcal{B}}^2=\left( \frac{N_\mathcal{B}}{\expkstar{}} p^2 + \frac{\expkstar{} - N_\mathcal{B}}{\expkstar{}}(1-p)^2 \right) (1-2^{-k})^n,$$
    we have
        \begin{align*}
            \hspace{-0.5cm} R(\hat{r}_{k,1,n}) \leq 2 \cdot \frac{p(1-p)}{n+1} + \frac{2^{k+1} \varepsilon_{n,k,p}}{n} + \bar{p}_{\mathcal{B}}^2,
        \end{align*}
        where $\varepsilon_{n,k,p} = n ( 1 - \frac{ 1 - e^{-2(p-\frac{1}{2})^2}}{2^k})^n$, and for all $n \geq  2^{k+1}(k+1)$, 
        \begin{align*}
            R (\hat{r}_{k,1,n}) \geq \frac{ 2p(1-p) }{n} - \frac{2^{k+3} (1-\rho_{k,p})^n}{n} + \bar{p}_{\mathcal{B}}^2,
        \end{align*}
        where $0<\rho_{k,p}<1$ depends only on $p$ and $k$.
\end{enumerate}
\end{prop}
Proposition~\ref{prop:new_risk_large_k} shows that there exists a benefit from using this network when the first-layer tree is deep enough. In this case, the risk of the shallow tree network is $O(1/n)$ whereas that of a single tree is $O(2^k/n)$. In presence of complex and highly structured data (large $k^{\star}$ and similar distribution in different areas of the input space), the shallow tree network benefits from a variance reduction phenomenon by a factor $2^k$.
These theoretical bounds are numerically assessed in the Supplementary Materials (see Figures \ref{fig:rtree_bound_small_depth} to \ref{fig:etree_bound_big_depth_app}) showing their tightness for a particular choice of the chessboard configuration. 

\replace{}{Finally, note that although the dimension $d$ does not explicitly appear in our bounds, it is closely related to $k^{\star}$. Indeed, in high dimensions, modelling the regression function requires a finer partition, hence a direct relation of the form $k^{\star} \gg d$. Therefore, obtaining an unbiased estimator with a reduced variance as in Proposition~\ref{prop:balanced_chessboard} is more stringent in high dimensions, since it requires to choose $k \geq k^{\star} \gg d$.} 

\section{Conclusion}

In this paper, we study both numerically and theoretically DF and its elementary components. We show that sta\-cking layers of trees (and forests) may improve the predictive performance of the algorithm. However, most of the improvements rely on the first DF-layers. We show that the performance of a shallow tree network (composed of single CART) depends on the depth of the first-layer tree. When the first-layer tree is deep enough, the second-layer tree may build upon the new features created by the first tree by acting as a variance reducer. 

To quantify this phenomenon, we propose a first theoretical analysis of a shallow tree network (composed of centered trees). Our study exhibits the crucial role of the first (encoding) layer: if the first-layer tree is biased, then the entire shallow network inherits this bias, otherwise the second-layer tree acts as a good variance reducer. One should note that this variance reduction cannot be obtained by averaging many trees, as in RF structure: the variance of an averaging of centered trees with depth $k$ is of the same order as one of these individual trees \cite{biau2012,klusowski2018sharp}, whereas two trees in cascade (the first one of depth $k$ and the second of depth $1$) may lead to a variance reduction by a $2^k$ factor. This highlights the benefit of tree-layer architectures over standard ensemble methods.
We thus believe that this first theoretical study of this shallow tree network paves the way of the mathematical understanding of DF.

First-layer trees, and more generally the first layers in DF architecture, can be seen as data-driven encoders. More precisely, the first layers in DF create an automatic embedding of the data, building on the specific conditional relation between the output and the inputs, therefore potentially improving the performance of the overall structure. Since preprocessing is nowadays an important part of all machine learning pipelines, we believe that our analysis is interesting beyond the framework of DF. 




\bibliographystyle{plain}
\bibliography{biblio}

\begin{thebibliography}{10}

\bibitem{NIPS2011_4443}
J.~S. Bergstra, R.~Bardenet, Y.~Bengio, and K.~Bal\'{a}zs.
\newblock Algorithms for hyper-parameter optimization.
\newblock In J.~Shawe-Taylor, R.~S. Zemel, P.~L. Bartlett, F.~Pereira, and
  K.~Q. Weinberger, editors, {\em Advances in Neural Information Processing
  Systems 24}, pages 2546--2554. Curran Associates, Inc., 2011.

\bibitem{berrouachedi2019deep}
A.~Berrouachedi, R.~Jaziri, and G.~Bernard.
\newblock Deep cascade of extra trees.
\newblock In {\em Pacific-Asia Conference on Knowledge Discovery and Data
  Mining}, pages 117--129. Springer, 2019.

\bibitem{DERT}
A.~Berrouachedi, R.~Jaziri, and G.~Bernard.
\newblock Deep extremely randomized trees.
\newblock In {\em International Conference on Neural Information Processing},
  pages 717--729. Springer, 2019.

\bibitem{biau2012}
G.~Biau.
\newblock Analysis of a random forests model.
\newblock {\em The Journal of Machine Learning Research}, 13(1):1063--1095,
  2012.

\bibitem{biau2008consistency}
G.~Biau, L.~Devroye, and G.~Lugosi.
\newblock Consistency of random forests and other averaging classifiers.
\newblock {\em Journal of Machine Learning Research}, 9(Sep):2015--2033, 2008.

\bibitem{breiman2001random}
L.~Breiman.
\newblock Random forests.
\newblock {\em Machine learning}, 45(1):5--32, 2001.

\bibitem{cribari2000note}
F.~Cribari-Neto, N.~L. Garcia, and K.~LP Vasconcellos.
\newblock A note on inverse moments of binomial variates.
\newblock {\em Brazilian Review of Econometrics}, 20(2):269--277, 2000.

\bibitem{fan2003random}
Wei Fan, Haixun Wang, Philip~S Yu, and Sheng Ma.
\newblock Is random model better? on its accuracy and efficiency.
\newblock In {\em Third IEEE International Conference on Data Mining}, pages
  51--58. IEEE, 2003.

\bibitem{feng2018multi}
J.~Feng, Y.~Yu, and Z-H Zhou.
\newblock Multi-layered gradient boosting decision trees.
\newblock In {\em Advances in neural information processing systems}, pages
  3551--3561, 2018.

\bibitem{ghods2020survey}
Alireza Ghods and Diane~J Cook.
\newblock A survey of deep network techniques all classifiers can adopt.
\newblock {\em Data Mining and Knowledge Discovery}, pages 1--42, 2020.

\bibitem{ghosh2002chessboard}
Soumyadip Ghosh and Shane~G Henderson.
\newblock Chessboard distributions and random vectors with specified marginals
  and covariance matrix.
\newblock {\em Operations Research}, 50(5):820--834, 2002.

\bibitem{ghosh2009patchwork}
Soumyadip Ghosh and Shane~G Henderson.
\newblock Patchwork distributions.
\newblock In {\em Advancing the Frontiers of Simulation}, pages 65--86.
  Springer, 2009.

\bibitem{guo2018bcdforest}
Y.~Guo, S.~Liu, Z.~Li, and X.~Shang.
\newblock Bcdforest: a boosting cascade deep forest model towards the
  classification of cancer subtypes based on gene expression data.
\newblock {\em BMC bioinformatics}, 19(5):118, 2018.

\bibitem{jeong2020lightweight}
M.~Jeong, J.~Nam, and B.~C. Ko.
\newblock Lightweight multilayer random forests for monitoring driver emotional
  status.
\newblock {\em IEEE Access}, 8:60344--60354, 2020.

\bibitem{kim2020interpretation}
S.~Kim, M.~Jeong, and B.~C. Ko.
\newblock Interpretation and simplification of deep forest.
\newblock {\em arXiv preprint arXiv:2001.04721}, 2020.

\bibitem{klusowski2018sharp}
J.~M. Klusowski.
\newblock Sharp analysis of a simple model for random forests.
\newblock {\em arXiv preprint arXiv:1805.02587}, 2018.

\bibitem{liu2020morphological}
B.~Liu, W.~Guo, X.~Chen, K.~Gao, X.~Zuo, R.~Wang, and A.~Yu.
\newblock Morphological attribute profile cube and deep random forest for small
  sample classification of hyperspectral image.
\newblock {\em IEEE Access}, 8:117096--117108, 2020.

\bibitem{ftdrf}
K.~Miller, C.~Hettinger, J.~Humpherys, T.~Jarvis, and D.~Kartchner.
\newblock Forward thinking: Building deep random forests.
\newblock {\em arXiv}, 2017.

\bibitem{DF_confidence_screening}
M.~{Pang}, K.~{Ting}, P.~{Zhao}, and Z.~{Zhou}.
\newblock Improving deep forest by confidence screening.
\newblock In {\em 2018 IEEE International Conference on Data Mining (ICDM)},
  pages 1194--1199, 2018.

\bibitem{scikit-learn}
F.~Pedregosa, G.~Varoquaux, A.~Gramfort, V.~Michel, B.~Thirion, O.~Grisel,
  M.~Blondel, P.~Prettenhofer, R.~Weiss, V.~Dubourg, J.~Vanderplas, A.~Passos,
  D.~Cournapeau, M.~Brucher, M.~Perrot, and E.~Duchesnay.
\newblock Scikit-learn: Machine learning in {P}ython.
\newblock {\em Journal of Machine Learning Research}, 12:2825--2830, 2011.

\bibitem{su2019deep}
R.~Su, X.~Liu, L.~Wei, and Q.~Zou.
\newblock Deep-resp-forest: A deep forest model to predict anti-cancer drug
  response.
\newblock {\em Methods}, 166:91--102, 2019.

\bibitem{sun2020adaptive}
L.~{Sun}, Z.~{Mo}, F.~{Yan}, L.~{Xia}, F.~{Shan}, Z.~{Ding}, B.~{Song},
  W.~{Gao}, W.~{Shao}, F.~{Shi}, H.~{Yuan}, H.~{Jiang}, D.~{Wu}, Y.~{Wei},
  Y.~{Gao}, H.~{Sui}, D.~{Zhang}, and D.~{Shen}.
\newblock Adaptive feature selection guided deep forest for covid-19
  classification with chest ct.
\newblock {\em IEEE Journal of Biomedical and Health Informatics},
  24(10):2798--2805, 2020.

\bibitem{utkin2017discriminative}
L.~V Utkin and M.~A Ryabinin.
\newblock Discriminative metric learning with deep forest.
\newblock {\em arXiv preprint arXiv:1705.09620}, 2017.

\bibitem{utkin2020improvement}
L.~V Utkin and K.~D Zhuk.
\newblock Improvement of the deep forest classifier by a set of neural
  networks.
\newblock {\em Informatica}, 44(1), 2020.

\bibitem{zeng2020network}
X.~Zeng, S.~Zhu, Y.~Hou, P.~Zhang, L.~Li, J.~Li, L~F. Huang, S.~J Lewis,
  R.~Nussinov, and F.~Cheng.
\newblock Network-based prediction of drug--target interactions using an
  arbitrary-order proximity embedded deep forest.
\newblock {\em Bioinformatics}, 36(9):2805--2812, 2020.

\bibitem{zhang2019distributed}
Y.~Zhang, J.~Zhou, W.~Zheng, J.~Feng, L.~Li, Z.~Liu, M.~Li, Z.~Zhang, C.~Chen,
  X.~Li, et~al.
\newblock Distributed deep forest and its application to automatic detection of
  cash-out fraud.
\newblock {\em ACM Transactions on Intelligent Systems and Technology (TIST)},
  10(5):1--19, 2019.

\bibitem{zheng2016improving}
S.~Zheng, Y.~Song, T.~Leung, and I.~Goodfellow.
\newblock Improving the robustness of deep neural networks via stability
  training.
\newblock In {\em Proceedings of the ieee conference on computer vision and
  pattern recognition}, pages 4480--4488, 2016.

\bibitem{zhou2019deep}
Z~Zhou and J.~Feng.
\newblock Deep forest: Towards an alternative to deep neural networks.
\newblock In {\em Proceedings of the Twenty-Sixth International Joint
  Conference on Artificial Intelligence, {IJCAI-17}}, pages 3553--3559, 2017.

\bibitem{zhou2019deepbis}
Z.~Zhou and J.~Feng.
\newblock Deep forest.
\newblock {\em National Science Review}, 6(1):74--86, 2019.

\end{thebibliography}

\clearpage

\appendix 

\renewcommand{\thesection}{S\arabic{section}}
\renewcommand{\thetable}{S\arabic{table}}
\renewcommand{\thefigure}{S\arabic{figure}}
\renewcommand{\thetheorem}{S\arabic{theorem}}
\renewcommand{\thefootnote}{S\arabic{footnote}}

\section{Additional figures}

\subsection{Computation times for Section \ref{sec:toward_simplification}}
\label{app:table_time}
\subsection{Computation times for Section \ref{sec:ref_numerical_analysis}}
\label{app:table_time}

\begin{table}[h!]
	  \begin{center}
	    \resizebox*{0.7\textwidth}{!}{%
        \begin{tabular}{|c |c |c| c| c| c | c |} 
        \hline
        & Yeast & Housing & Letter & Adult & Airbnb & Higgs \\ [0.5ex] 
        \hline
        Default DF time & 13m19s & 9m38s & 20m31 & 13m57s & 23m23s &  43m53s\\ 
        \hline
        Light DF time & 7s & 6s & 8s & 8s & 10s & 13s\\ 
        \hline
        Default DF MC (MB) & 11 & 6 & 174 & 139 & 166 & 531 \\ 
        \hline
        Light DF MC (MB) & 5  & 4 & 109 & 72 & 100 & 318\\
        \hline
        \end{tabular}%
        }
	  \end{center}	
	  \caption{Comparing the time and memory consumption of DF and Light DF.}
\label{tab:mem_time_light_df}
\end{table}

\subsection{Table of best configurations, supplementary to Section \ref{subsec:optimal_layer}}
\label{app:best_config_details}

\begin{table}[ht]
\centering
 \begin{tabular}{||c | c c ||} 
 \hline
 Dataset & Best configuration hyperparam. & Optimal sub-model\\ [0.5ex] 
 \hline
 Adult & 6 forests, 20 trees, max depth 30 & 2 \\ 
 \hline
 Higgs & 10 forests, 280 trees, max depth None &  2 \\
 \hline 
 Fashion Mnist & 8 forests, 500 trees, max depth None (default) & 2 \\
 \hline
 Letter & 8 forests, 500 trees, max depth None (default) & 1\\
 \hline
 Yeast & 6 forests, 200 trees, max depth 30 & 1\\
 \hline
 Airbnb & 10 forests, 400 trees, max depth None & 1\\
 \hline
 Housing & 8 forests, 280 trees, max depth 100 & 14\\ [0.1ex] 
 \hline
\end{tabular}
\caption{Details of the best configurations obtained in Figures \ref{fig:DFvsLight_clas} and \ref{fig:DFvsLight_reg}.}
\label{tab:best_configurations}
\end{table}

To find the best configuration, we ran a grid search over the following parameters : number of forests per layer (from 2 to 10) , number of trees per forest (from 30 to 1000), max depth of each tree (from 5 to 100 plus None). \\

\subsection{Fashion Mnist MGS encoding}
\label{app:exp_fashion}

The Fashion Mnist dataset was encoded using the MGS process with two forests, one Breiman RF, one CRF, both of them having 150 trees, 10 samples per leaf minimum and other parameters set to default. Three windows were used of sizes/strides. Then we apply a mean pooling process of size (3,3) to each created filter.

\clearpage

\subsection{Additional figures to Section \ref{sec:shallowCART}}
\label{app:sec:shallowCART}

\begin{figure}[H]
    \centering
    \includegraphics[width = 0.5\textwidth]{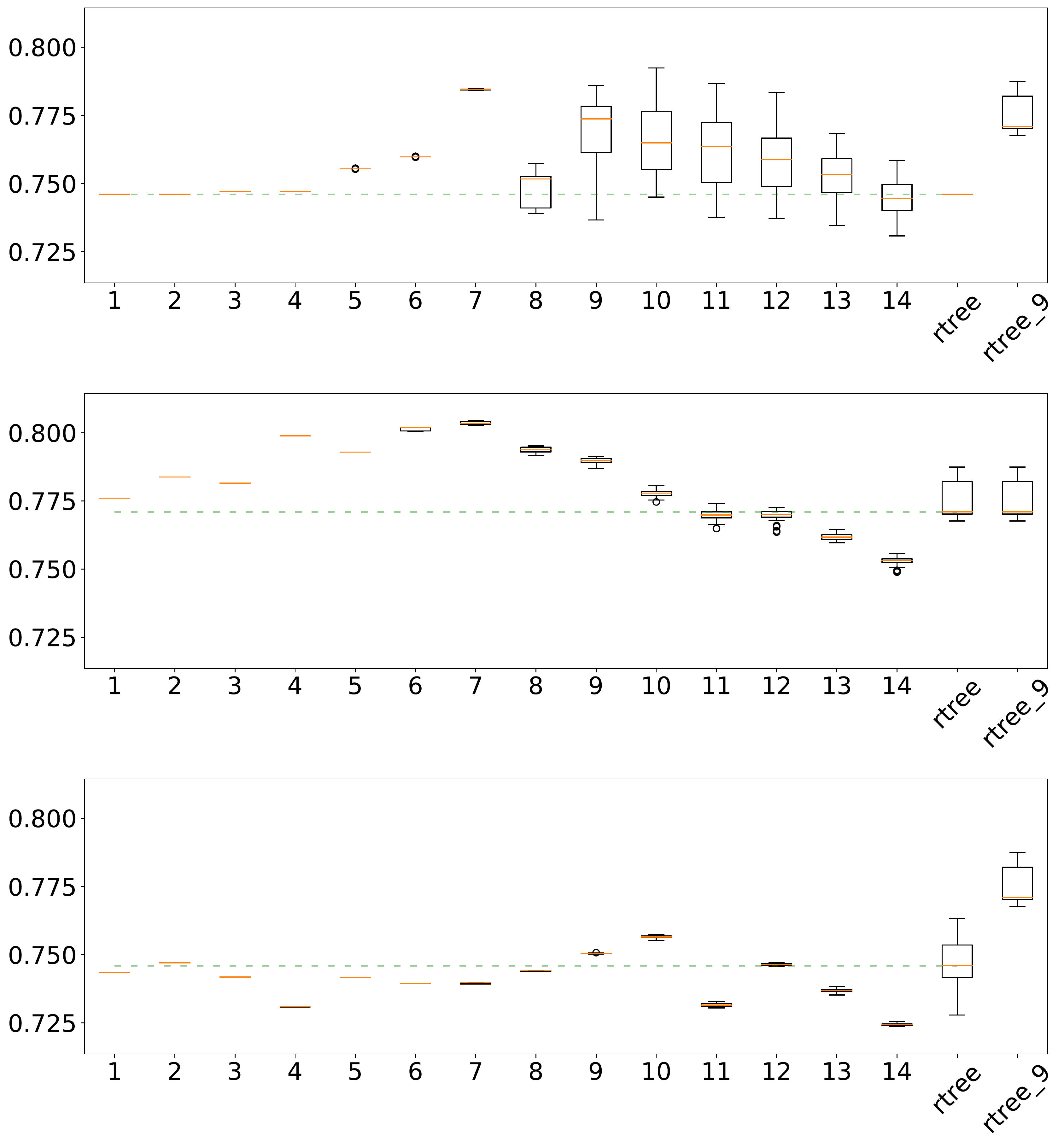}
    \caption{Adult dataset. Accuracy of a two-layer tree architecture w.r.t.\ the second-layer tree depth, when the first-layer (encoding) tree  is of depth 2 (top), 9 (middle), and 15 (bottom). \texttt{rtree} is a single tree of respective depth 2 (top), 9 (middle), and 15 (bottom),  applied on raw data. For this dataset, the optimal depth of a single tree is 9 and the tree with the optimal depth is depicted as \texttt{rtree\_9} in each plot.
    The green dashed line indicates the median score of the \texttt{rtree}. All boxplots are obtained by 10 different runs.}
    \label{fig:adult_encoding_influence}
\end{figure}

\begin{figure}[!h]
    \centering
    \includegraphics[width = 0.9\textwidth]{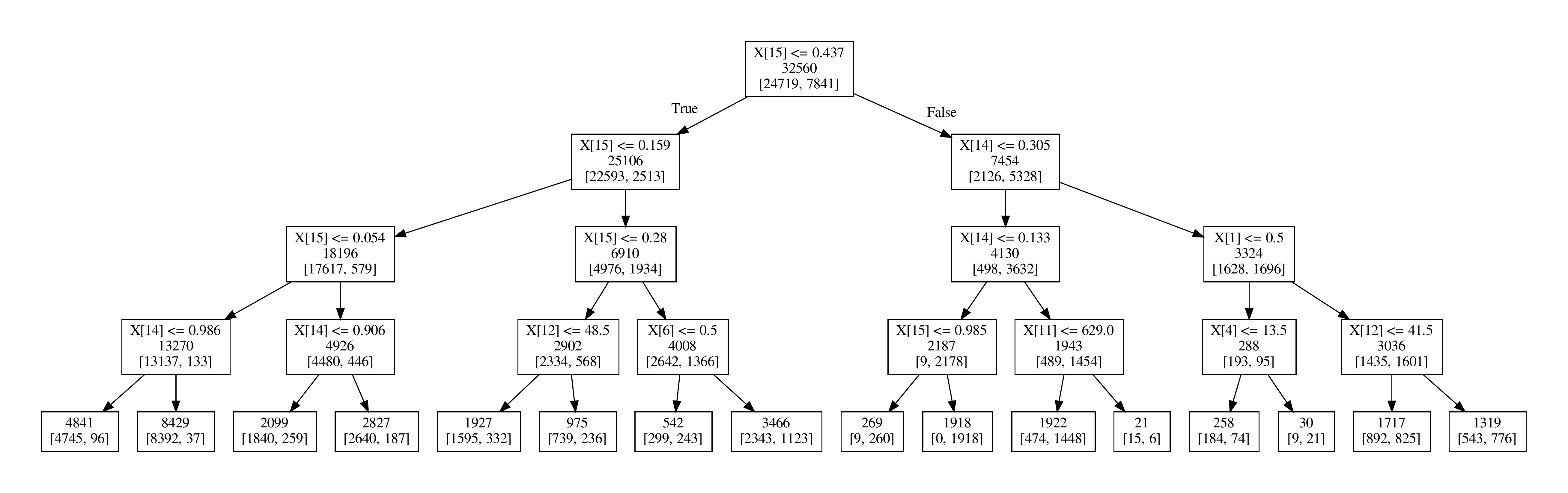}
    \caption{Adult dataset. Second-layer tree structure of depth 4 when the first-layer tree is of depth 9 (optimal depth). Raw features range from X[0] to X[13], X[14] and X[15] are the features built by the first-layer tree.}
    \label{fig:adult_etree_optimal}
\end{figure}
\clearpage

\begin{figure}[h]
    \centering
    \includegraphics[width = 0.5\textwidth]{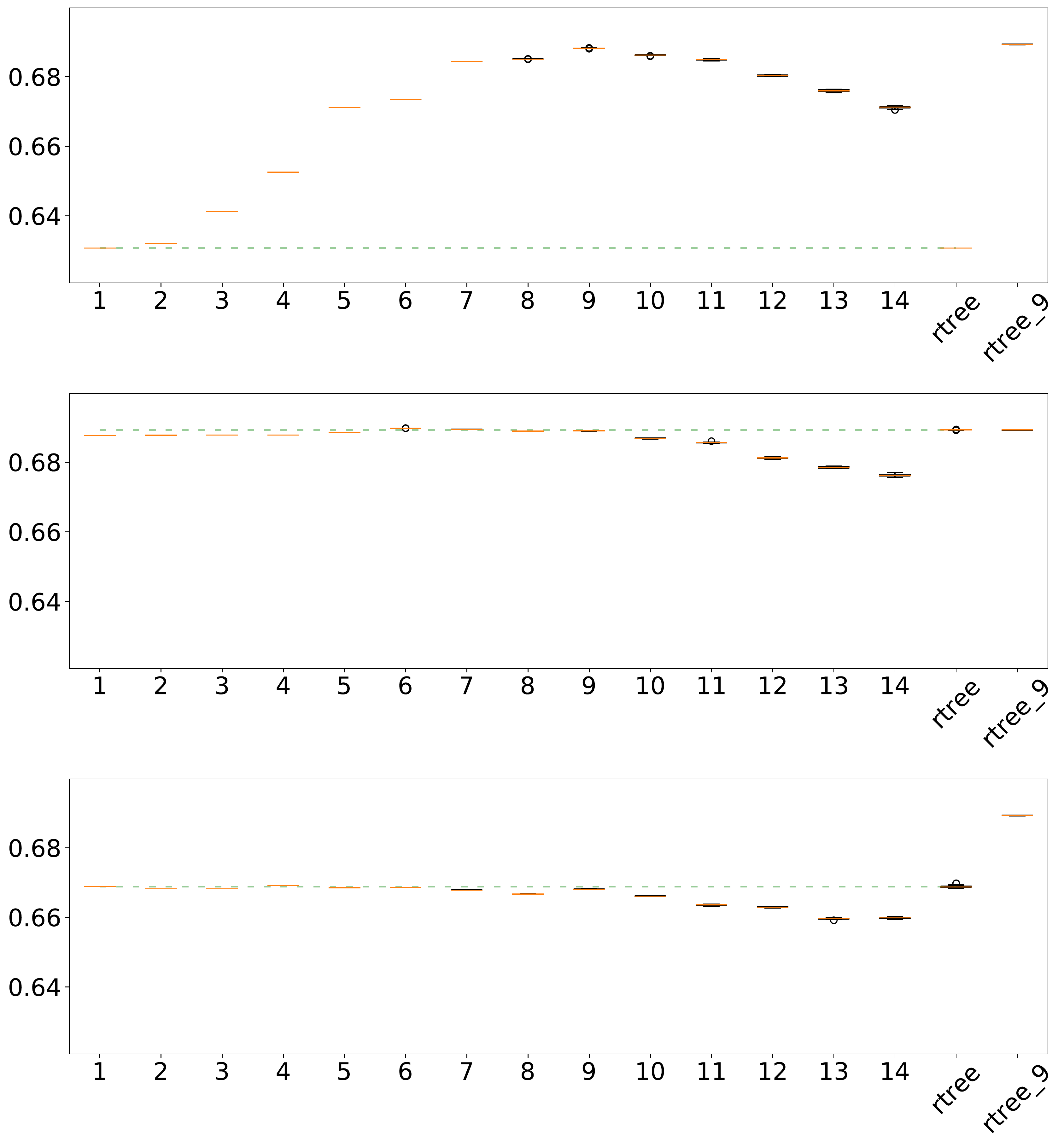}
    \caption{Higgs dataset. Accuracy of a two-layer tree architecture w.r.t.\ the second-layer tree depth, when the first-layer (encoding) tree  is of depth 2 (top), 9 (middle), and 15 (bottom). \texttt{rtree} is a single tree of respective depth 2 (top), 9 (middle), and 15 (bottom),  applied on raw data. For this dataset, the optimal depth of a single tree is 9 and the tree with the optimal depth is depicted as \texttt{rtree\_9} in each plot.
    The green dashed line indicates the median score of the \texttt{rtree}. All boxplots are obtained by 10 different runs.}
    \label{fig:higgs_encoding_influence}
\end{figure}

\begin{figure}[!h]
    \centering
    \includegraphics[width = 0.8\textwidth]{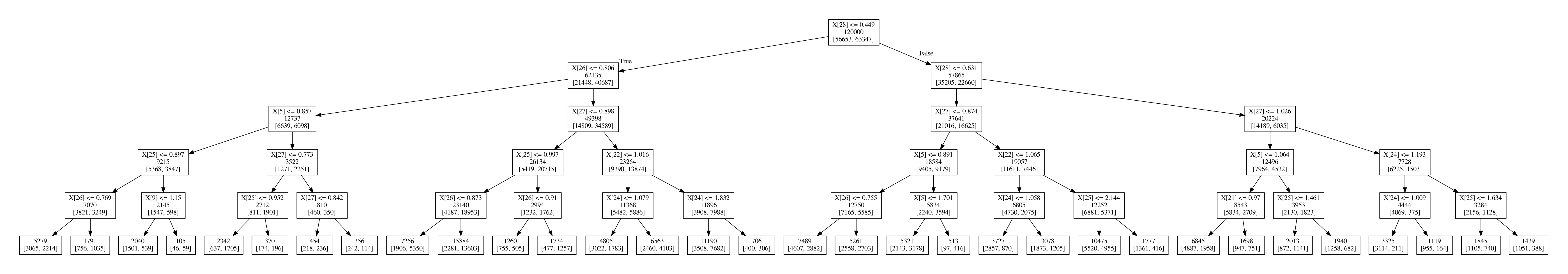}
    \caption{Higgs dataset. Second-layer tree structure of depth 5 when the first-layer tree is of depth 2 (low depth). Raw features range from X[0] to X[13], X[14] and X[15] are the features built by the first-layer tree.}
    \label{fig:Higgs_etree_low}
\end{figure}

\begin{figure}[!h]
    \centering
    \includegraphics[width = 0.8\textwidth]{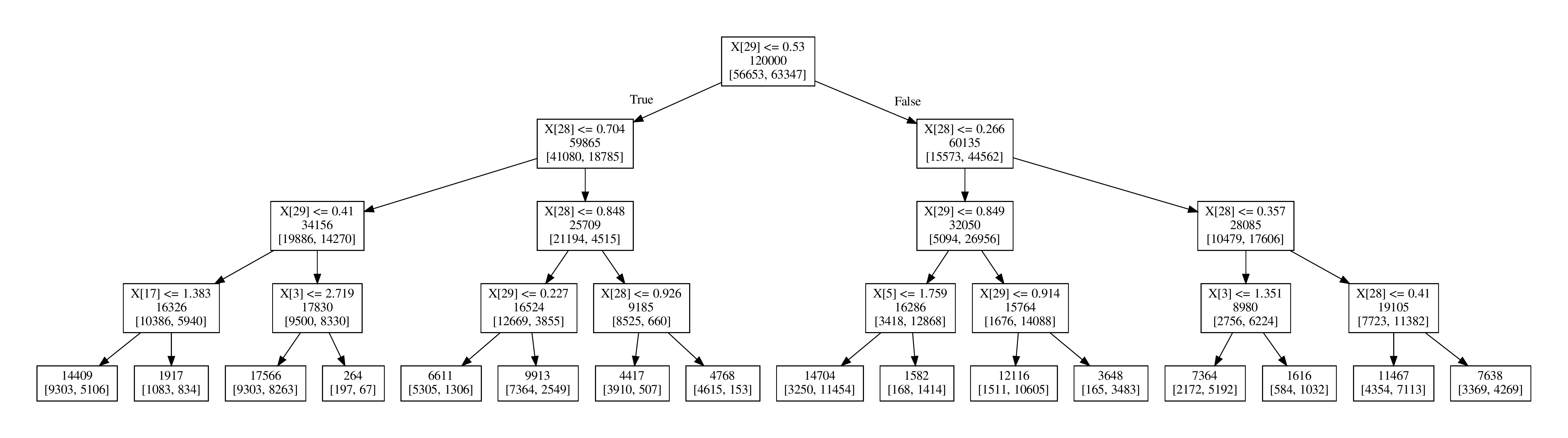}
    \caption{Higgs dataset. Second-layer tree structure of depth 4 when the first-layer tree is of depth 9 (optimal depth). Raw features range from X[0] to X[27], X[28] and X[29] are the features built by the first-layer tree.}
    \label{fig:Higgs_etree_optimal}
\end{figure}

\clearpage

\begin{figure}[h]
    \centering
    \includegraphics[width = 0.5\textwidth]{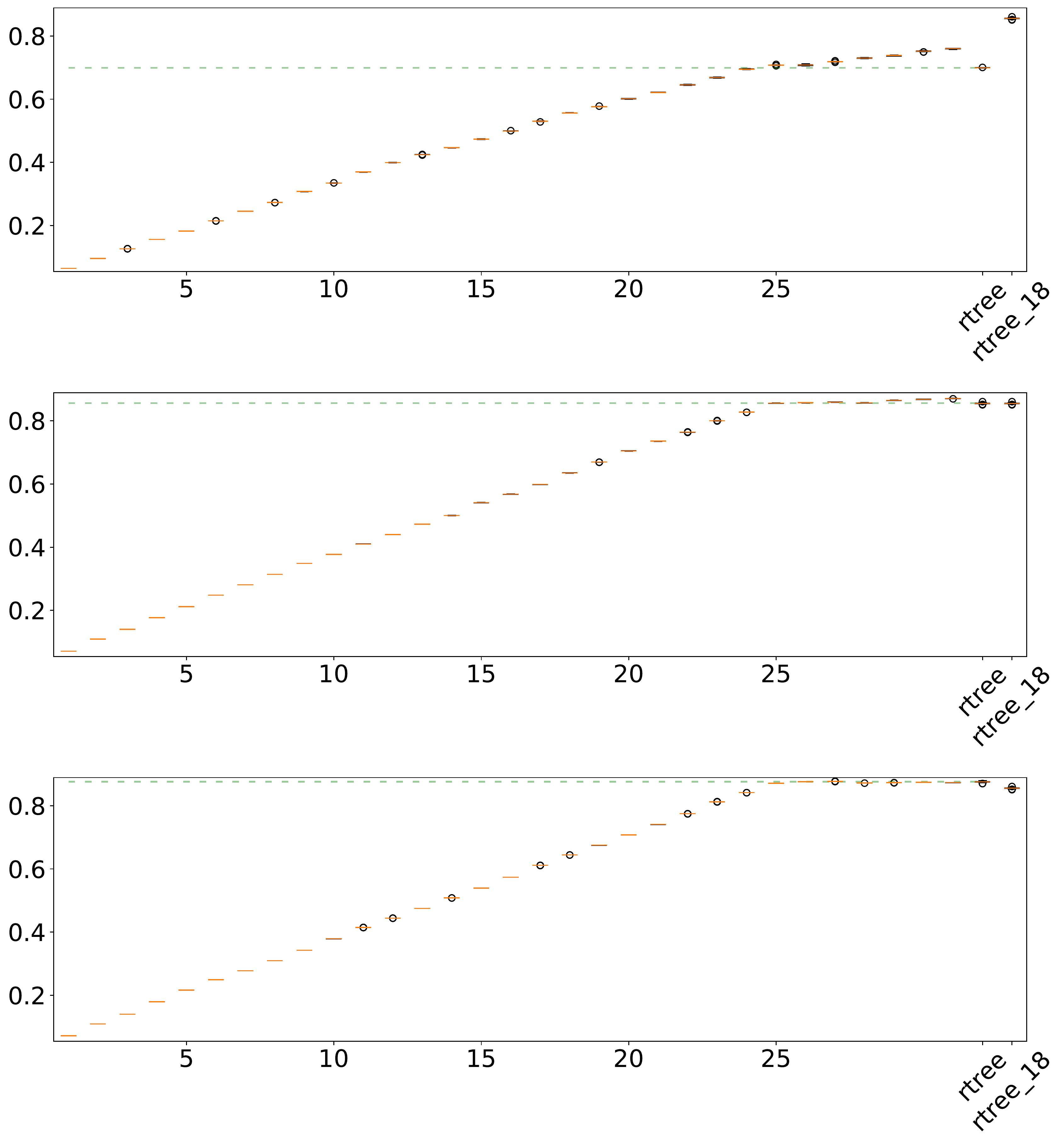}
    \caption{Letter dataset. Accuracy of a two-layer tree architecture w.r.t.\ the second-layer tree depth, when the first-layer (encoding) tree  is of depth 10 (top), 18 (middle), and 26 (bottom). \texttt{rtree} is a single tree of respective depth 10 (top), 18 (middle), and 26 (bottom), applied on raw data. For this dataset, the optimal depth of a single tree is 18 and the tree with the optimal depth is depicted as \texttt{rtree\_18} in each plot.
    The green dashed line indicates the median score of the \texttt{rtree}. All boxplots are obtained by 10 different runs.}
    \label{fig:letter_encoding_influence}
\end{figure}

\begin{figure}[h]
    \centering
    \includegraphics[width = 0.5\textwidth]{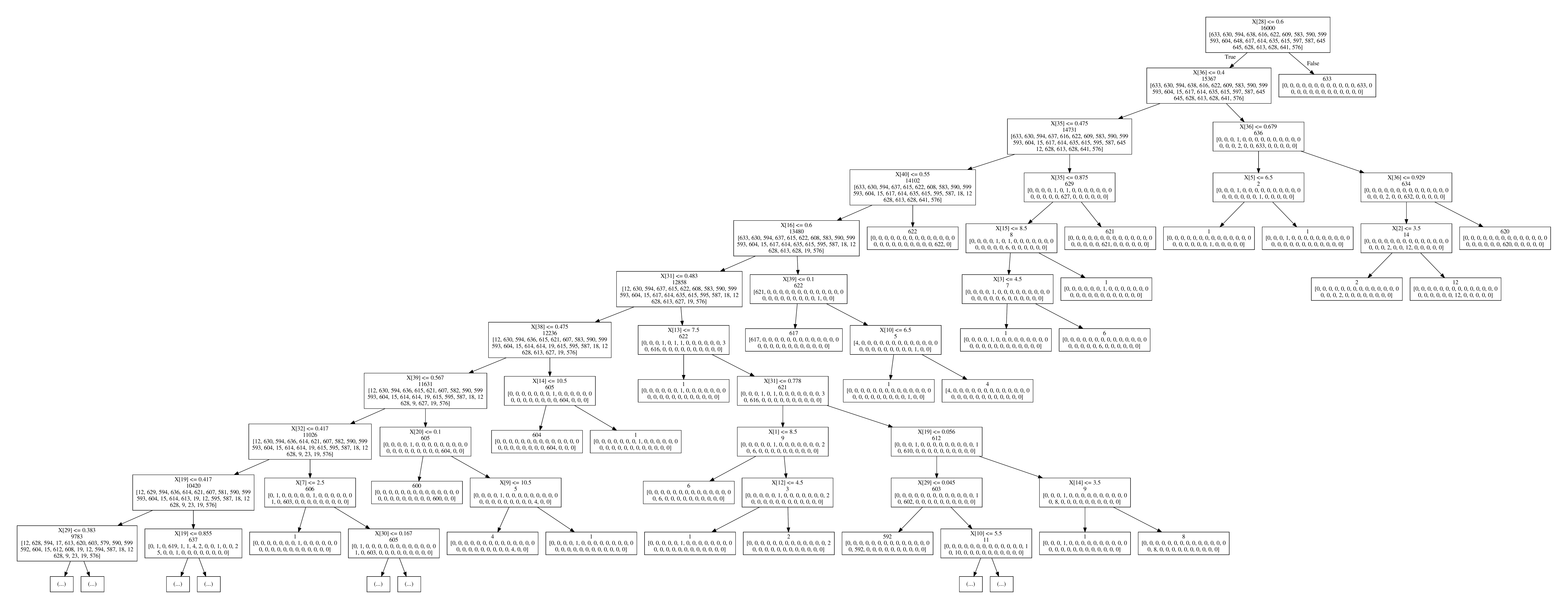}
    \caption{Letter dataset. Second-layer tree structure of depth 30 when the first-layer tree is of depth 18 (optimal depth). We only show the first part of the tree up to depth 10. Raw features range from X[0] to X[15]. The features built by the first-layer tree range from X[16] to X[41].}
    \label{fig:letter_etree_optimal}
\end{figure}

\clearpage

\begin{figure}[H]
    \centering
    \includegraphics[width = 0.5\textwidth]{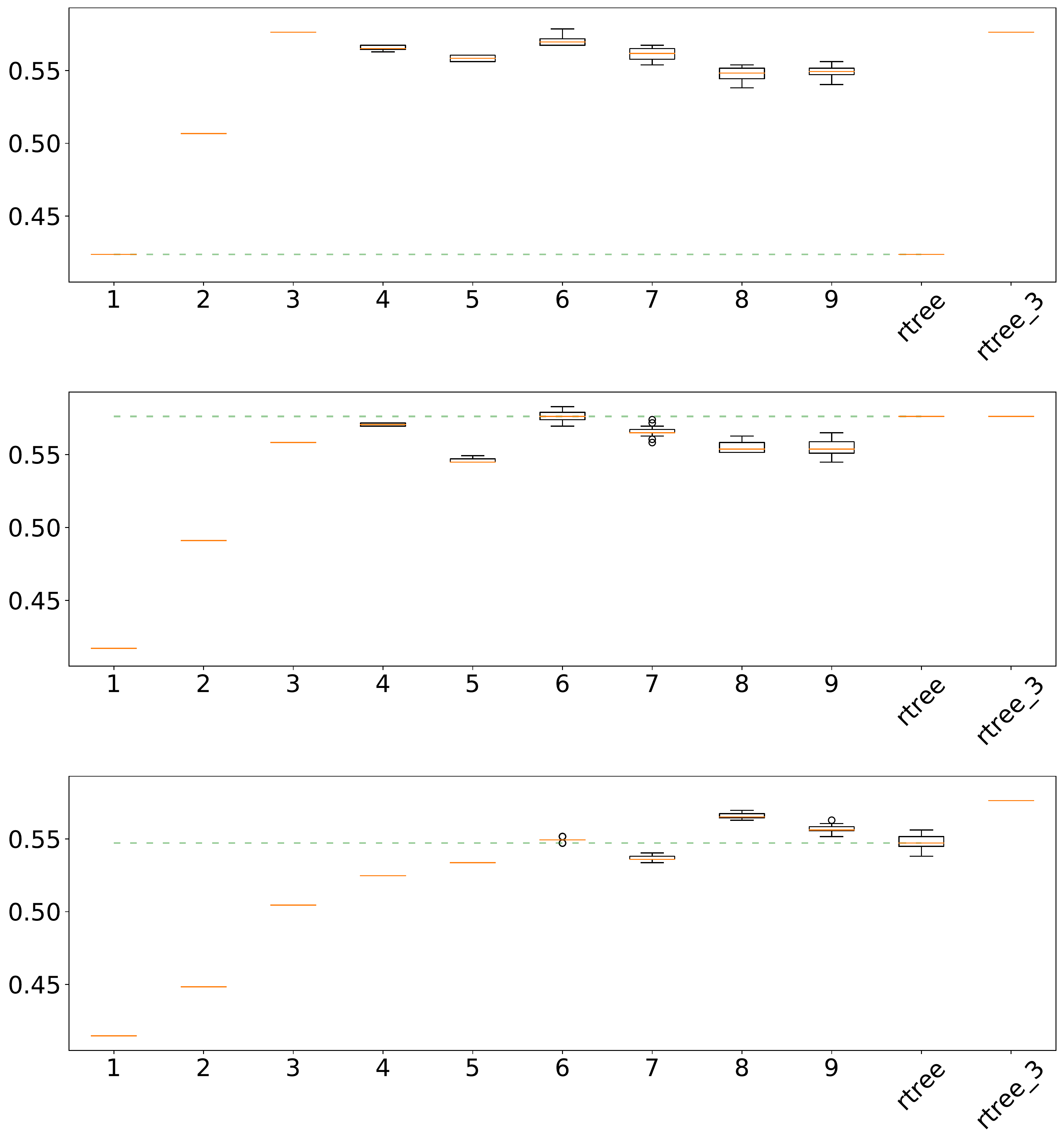}
    \caption{Yeast dataset. Accuracy of a two-layer tree architecture w.r.t.\ the second-layer tree depth, when the first-layer (encoding) tree  is of depth 1 (top), 3 (middle), and 8 (bottom). \texttt{rtree} is a single tree of respective depth 1 (top), 3 (middle), and 8 (bottom),  applied on raw data. For this dataset, the optimal depth of a single tree is 3 and the tree with the optimal depth is depicted as \texttt{rtree\_3} in each plot.
    The green dashed line indicates the median score of the \texttt{rtree}. All boxplots are obtained by 10 different runs.}
    \label{fig:yeast_encoding_influence}
\end{figure}

\begin{figure}[!h]
    \centering
    \includegraphics[width = 0.8\textwidth]{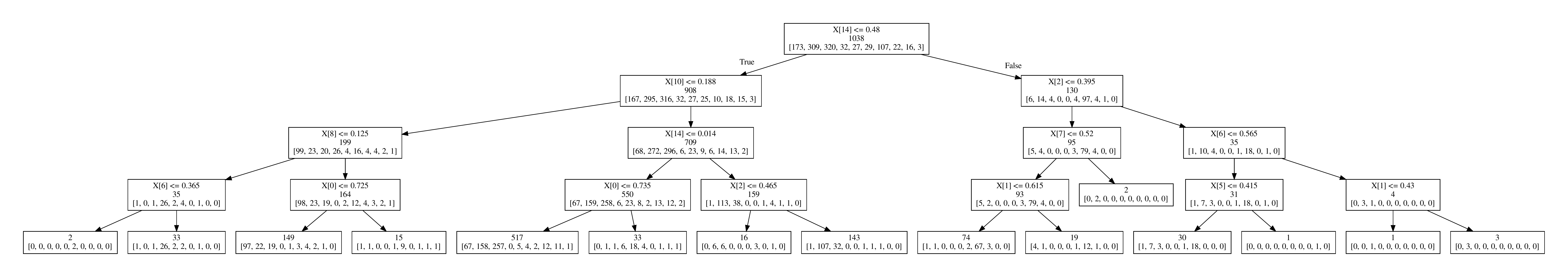}
    \caption{Yeast dataset. Second-layer tree structure of depth 4 when the first-layer tree is of depth 3 (optimal depth). Raw features range from X[0] to X[7]. The features built by the first-layer tree range from X[8] to X[17].}
    \label{fig:yeast_etree_optimal}
\end{figure}

\begin{figure}[h]
    \centering
    \includegraphics[width = 0.5\textwidth]{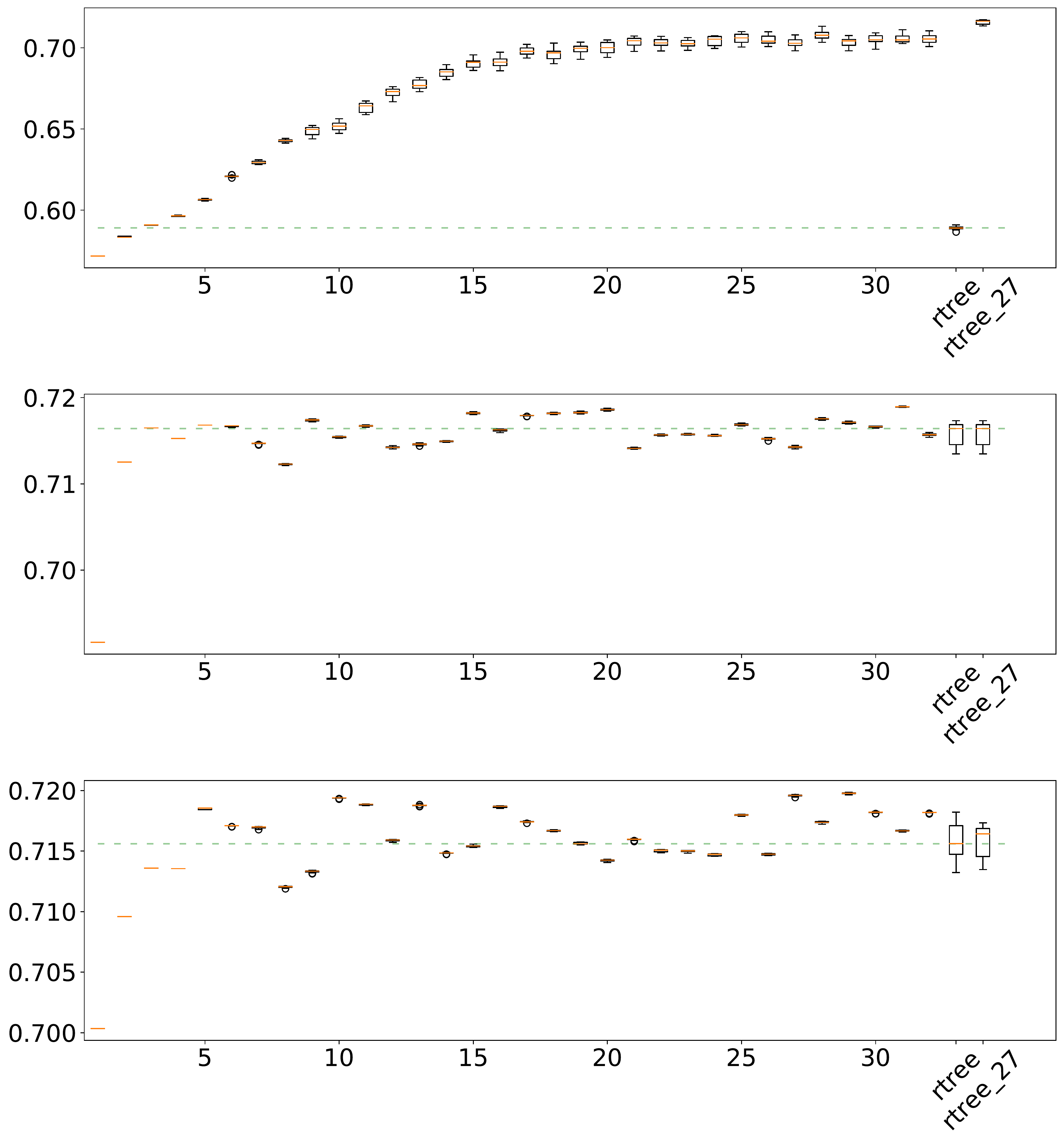}
    \caption{Airbnb dataset. $R^2$ score of a two-layer tree architecture w.r.t.\ the second-layer tree depth, when the first-layer (encoding) tree  is of depth 10 (top), 27 (middle), and 32 (bottom). \texttt{rtree} is a single tree of respective depth 10 (top), 27 (middle), and 32 (bottom),  applied on raw data. For this dataset, the optimal depth of a single tree is 27 and the tree with the optimal depth is depicted as \texttt{rtree\_27} in each plot.
    The green dashed line indicates the median score of the \texttt{rtree}. All boxplots are obtained by 10 different runs.}
    \label{fig:airbnb_encoding_influence}
\end{figure}

\begin{figure}[h]
    \centering
    \includegraphics[width = 0.8\textwidth]{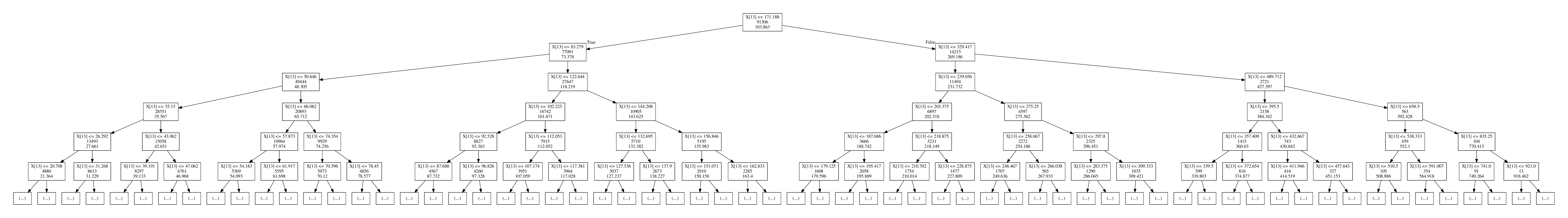}
    \caption{Airbnb dataset. Second-layer tree structure of depth 28 when the first-layer tree is of depth 26 (optimal depth). We only show the first part of the tree up to depth 5. Raw features range from X[0] to X[12], X[13] is the feature built by the first-layer tree.}
    \label{fig:airbnb_etree_optimal}
\end{figure}

\clearpage 

\begin{figure}[h]
    \centering
    \includegraphics[width = 0.5\textwidth]{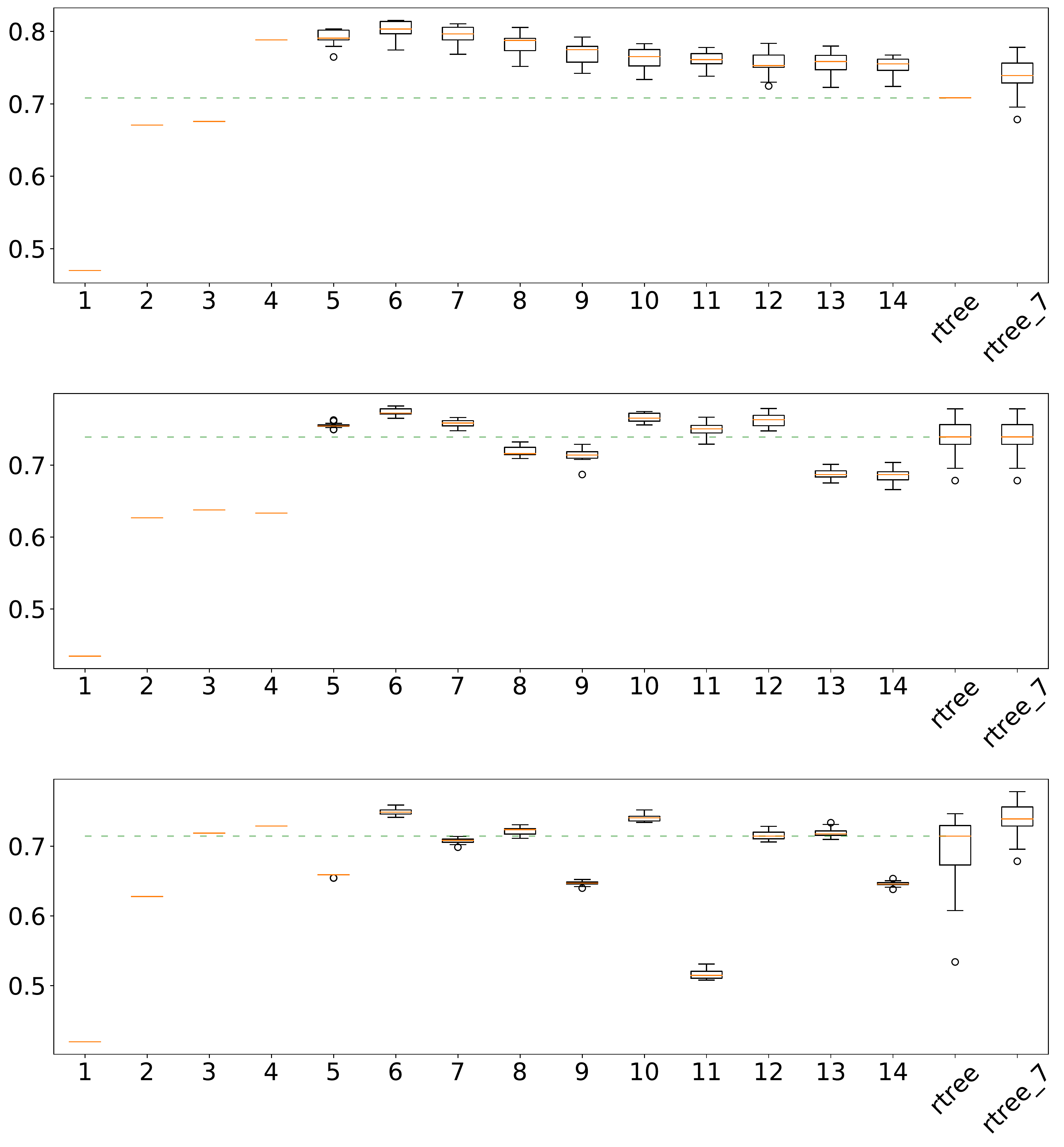}
    \caption{Housing dataset. $R^2$ score of a two-layer tree architecture w.r.t.\ the second-layer tree depth, when the first-layer (encoding) tree  is of depth 3 (top), 7 (middle), and 12 (bottom). \texttt{rtree} is a single tree of respective depth 3 (top), 7 (middle), and 12 (bottom),  applied on raw data. For this dataset, the optimal depth of a single tree is 9 and the tree with the optimal depth is depicted as \texttt{rtree\_7} in each plot.
    The green dashed line indicates the median score of the \texttt{rtree}. All boxplots are obtained by 10 different runs.}
    \label{fig:housing_encoding_influence}
\end{figure}

\begin{figure}[h]
    \centering
    \includegraphics[width = 0.6\textwidth]{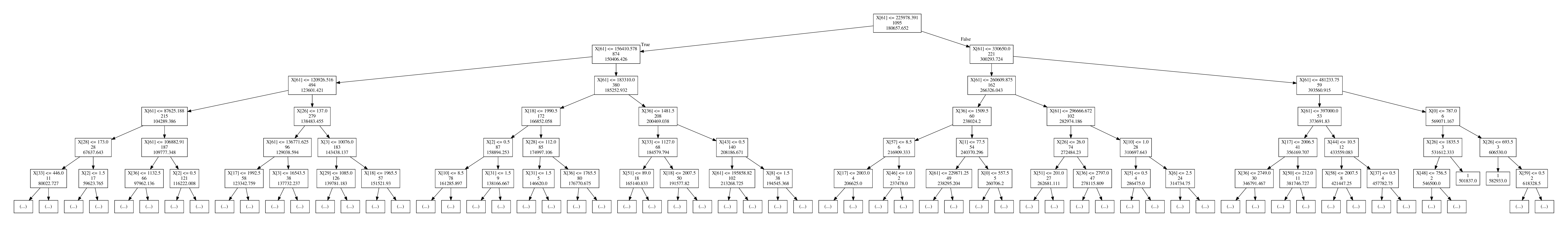}
    \caption{Housing dataset. Second-layer tree structure of depth 10 when the first-layer tree is of depth 7 (optimal depth). We only show the first part of the tree up to depth 5. Raw features range from X[0] to X[60], X[61] is the feature built by the first-layer tree.}
    \label{fig:housing_etree_optimal}
\end{figure}

\subsection{Additional figures to Section \ref{subsec:optimal_layer}}
\label{app:subsec:optimal_layer}

\begin{figure}[!htp]
    \centering
    \includegraphics[width = 0.4\textwidth]{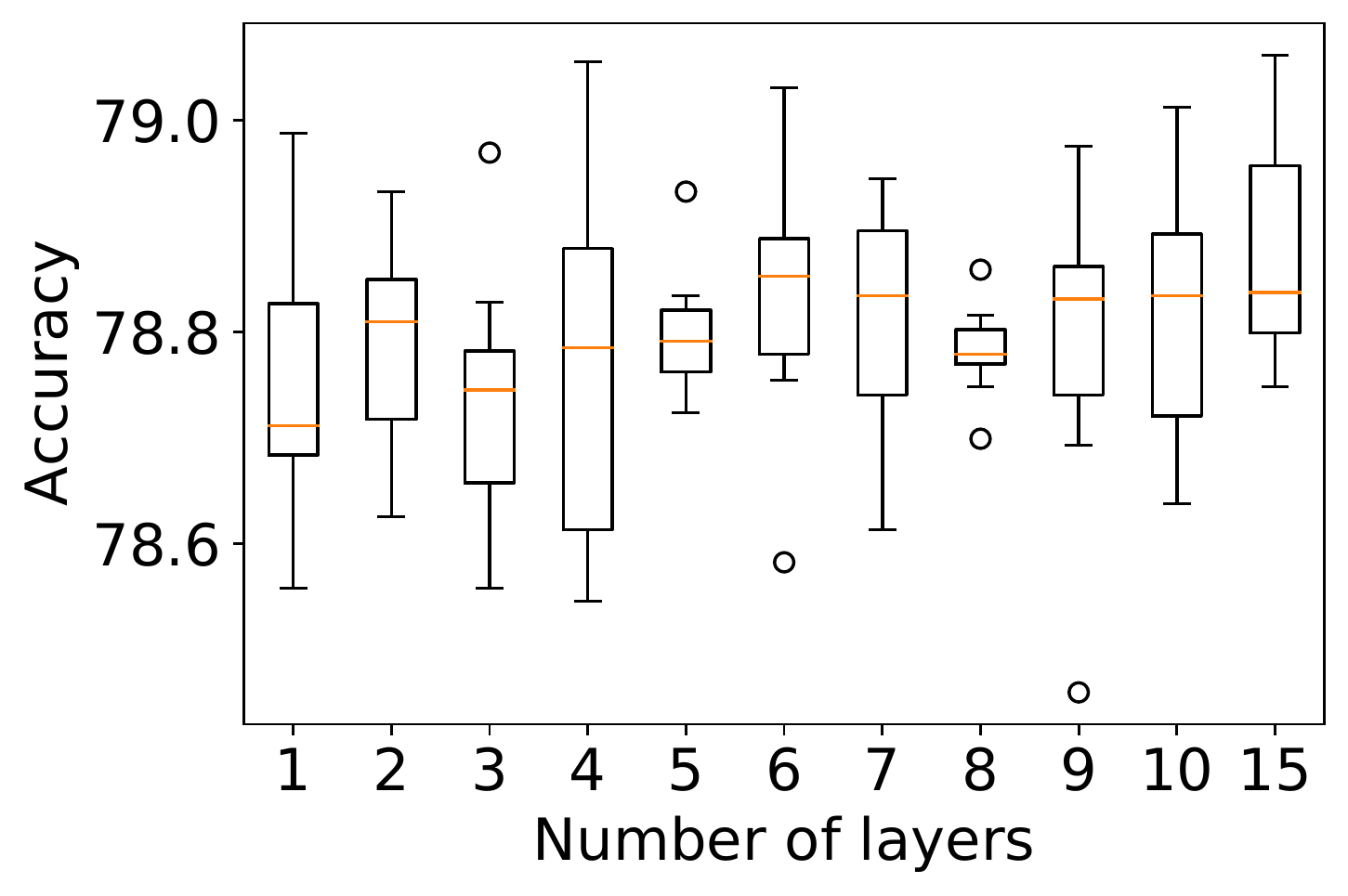}
    \includegraphics[width = 0.4\textwidth]{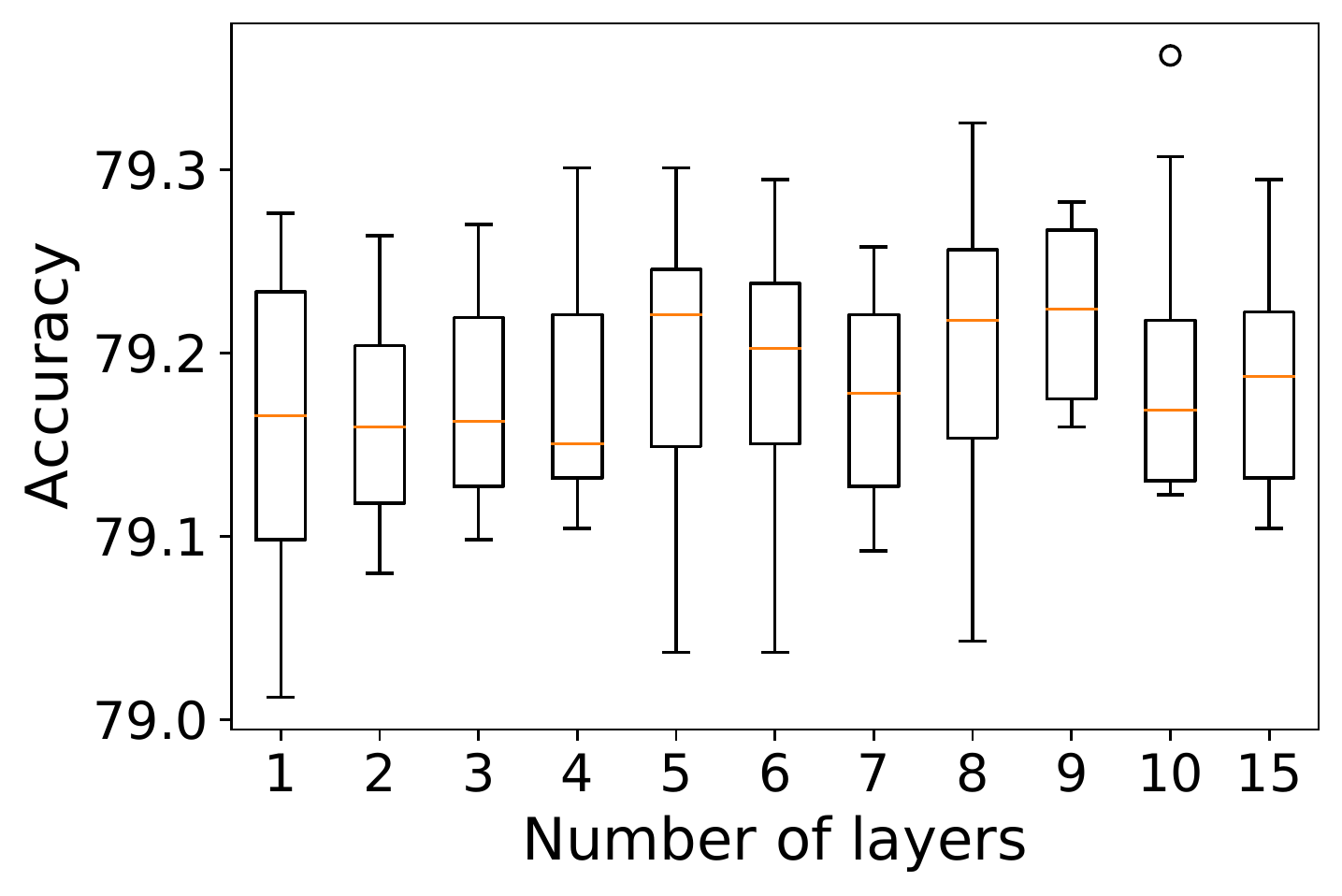}
    \caption{Adult dataset. Boxplots over 10 tries of the accuracy of a DF with 1 forest by layer (left) or 4 forests by layer (right), with respect to the DF maximal number of layers.}
    \label{fig:adult_acc_depth}
\end{figure}

\begin{figure}[!htp]
    \centering
    \includegraphics[width = 0.4\textwidth]{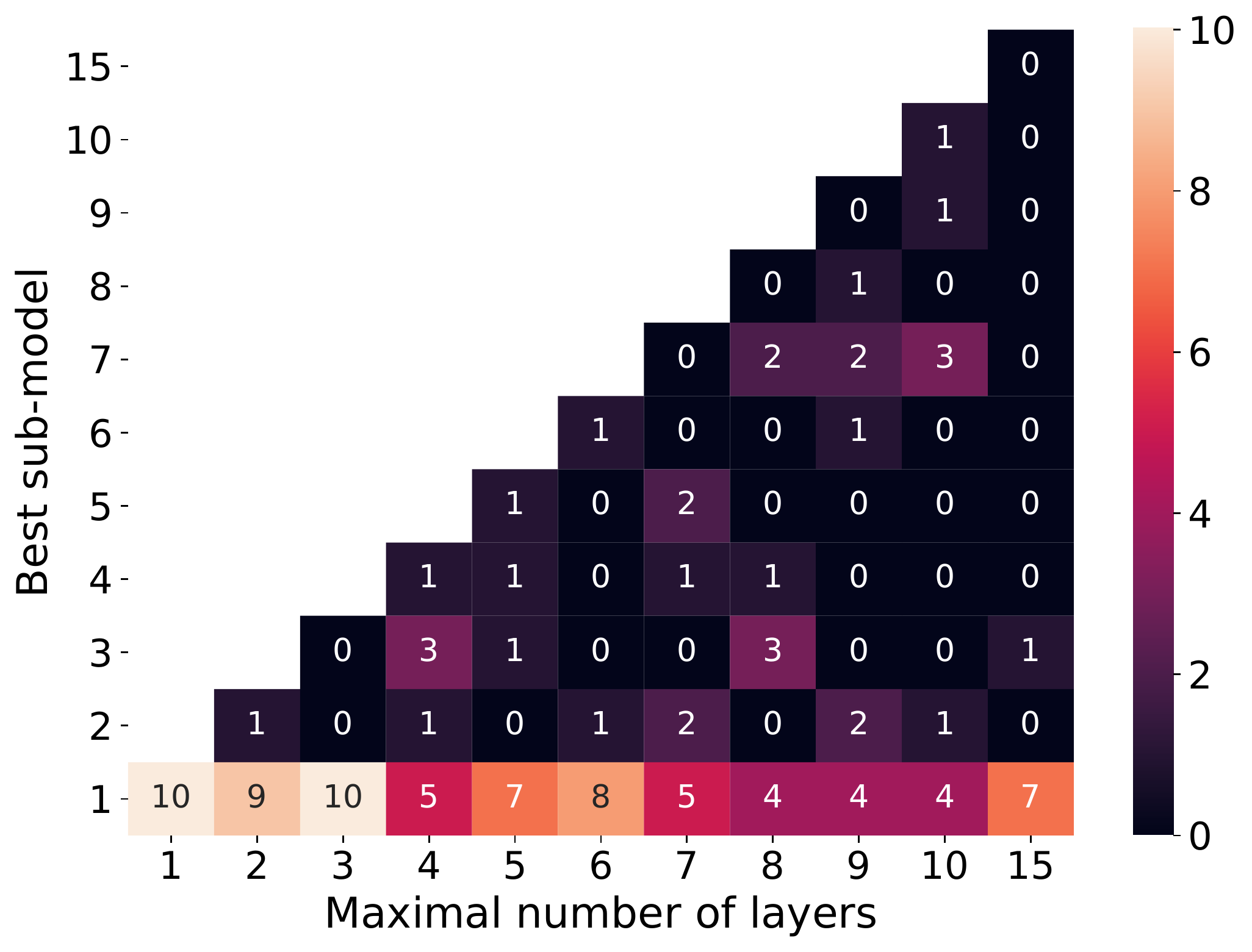}
    \includegraphics[width = 0.4\textwidth]{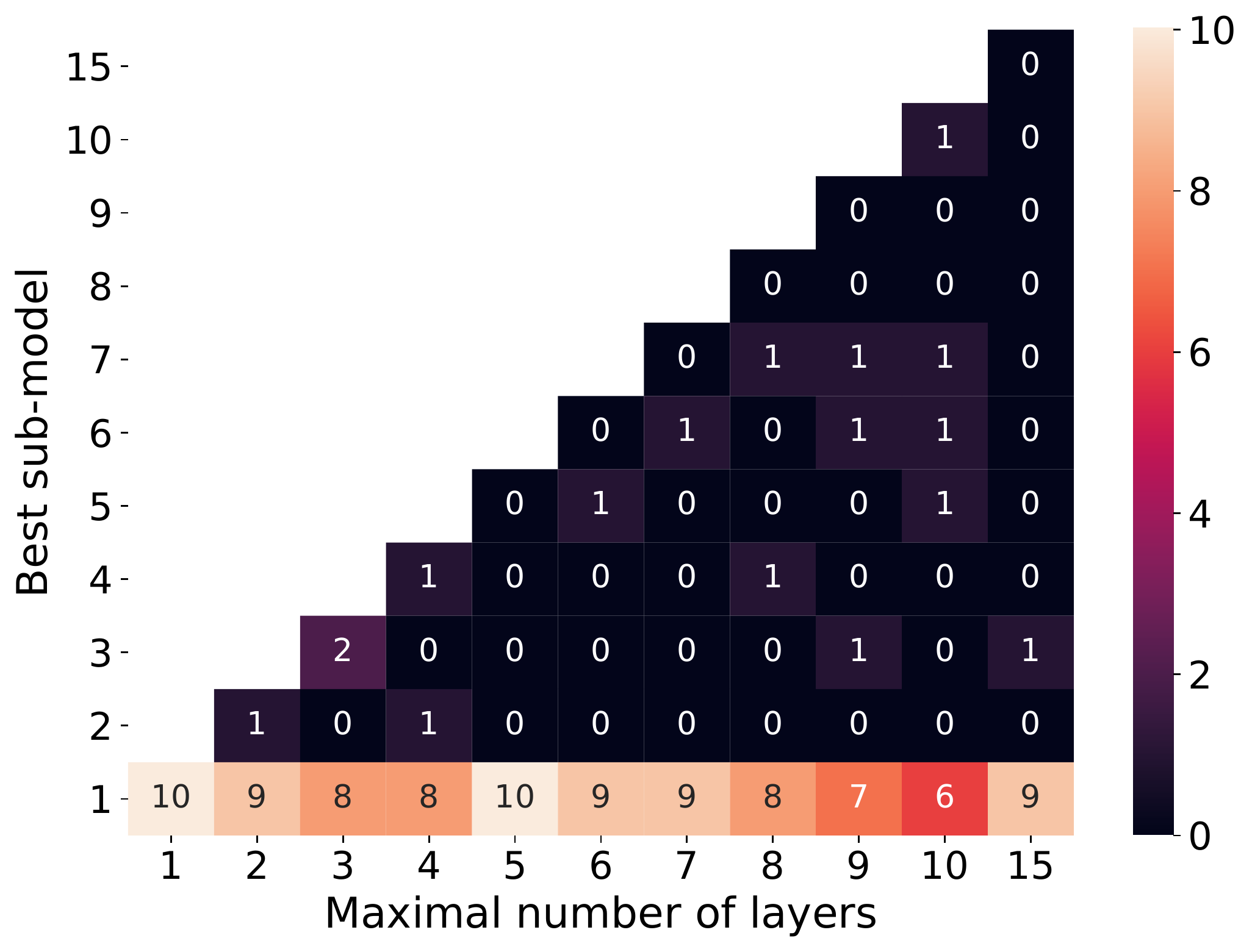}
    \caption{Adult dataset. Heatmap counting the index of the sub-optimal model over 10 tries of a default DF with 1 (Breiman) forest per layer (left) or 4 forests (2 Breiman, 2 CRF) per layer (right), with respect to the maximal  number of layers.}
    \label{fig:adult_heatmap}
\end{figure}

\begin{figure}[!htp]
    \centering
    \includegraphics[width = 0.4\textwidth]{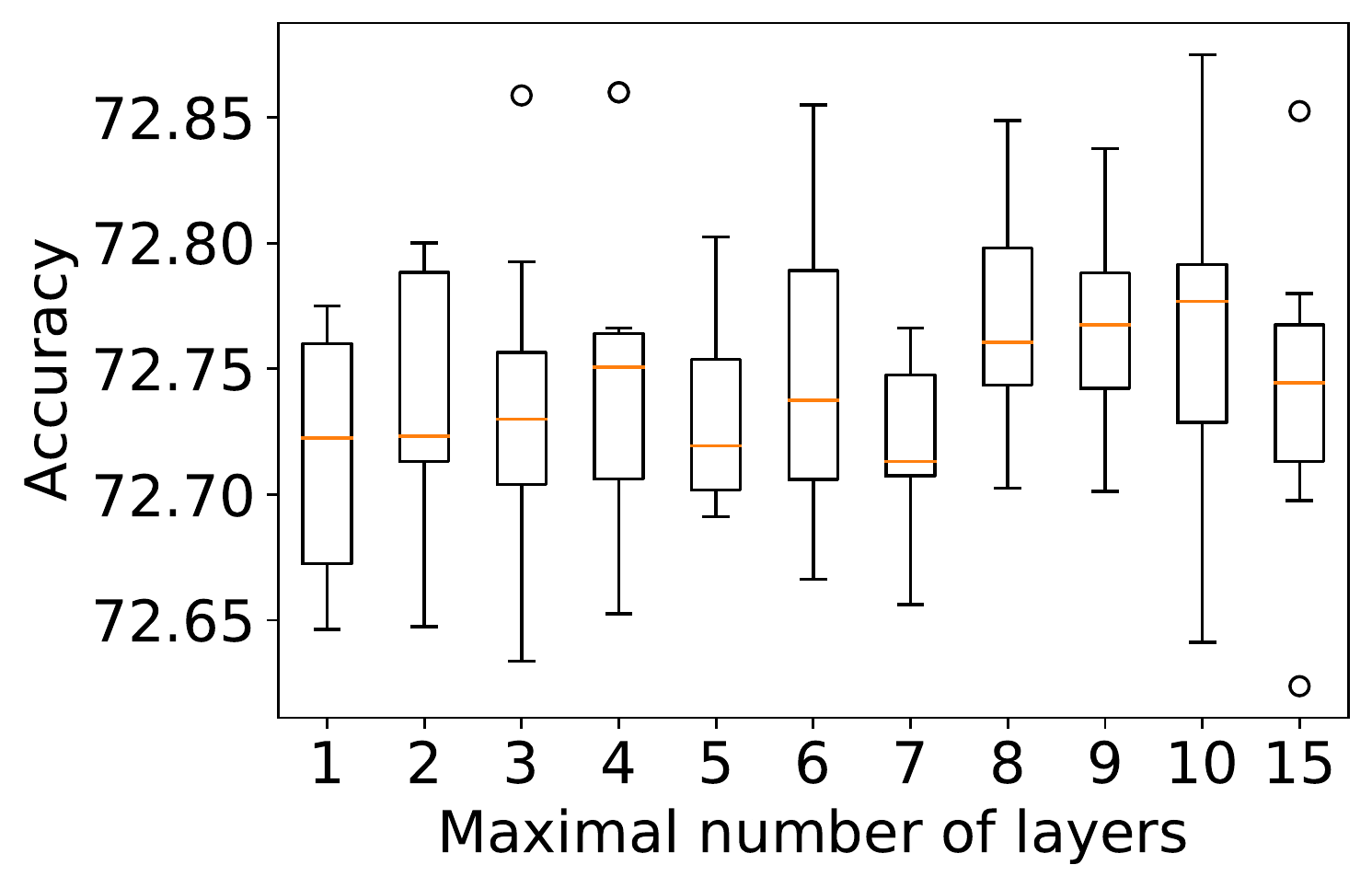}
    \includegraphics[width = 0.4\textwidth]{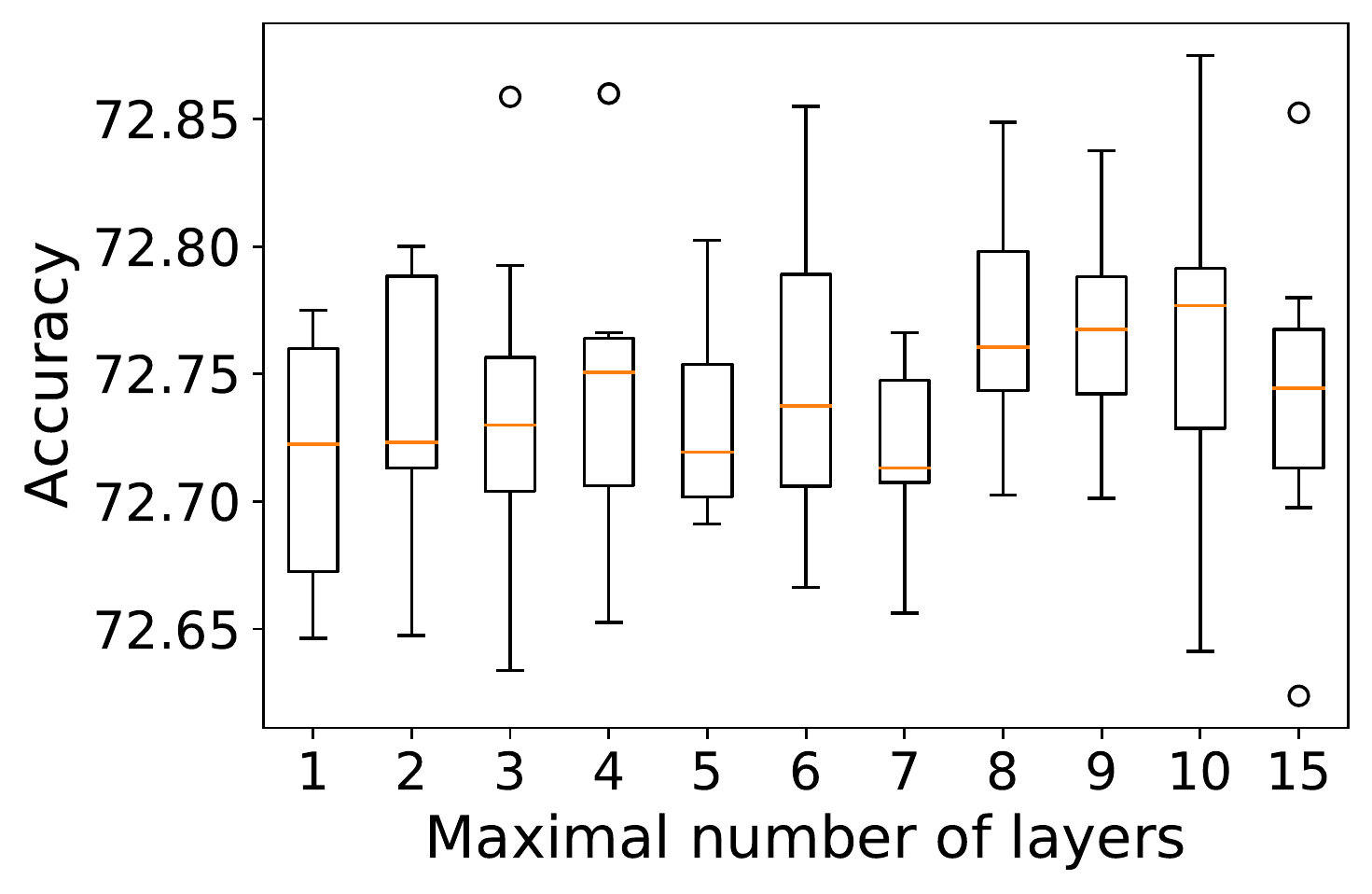}
    \caption{Higgs dataset. Boxplots over 10 tries of the accuracy of a DF with 1 forest by layer (left) or 4 forests by layer (right), with respect to the DF maximal number of layers.}
    \label{fig:higgs_acc_depth}
\end{figure}

\begin{figure}[!htp]
    \centering
    \includegraphics[width = 0.4\textwidth]{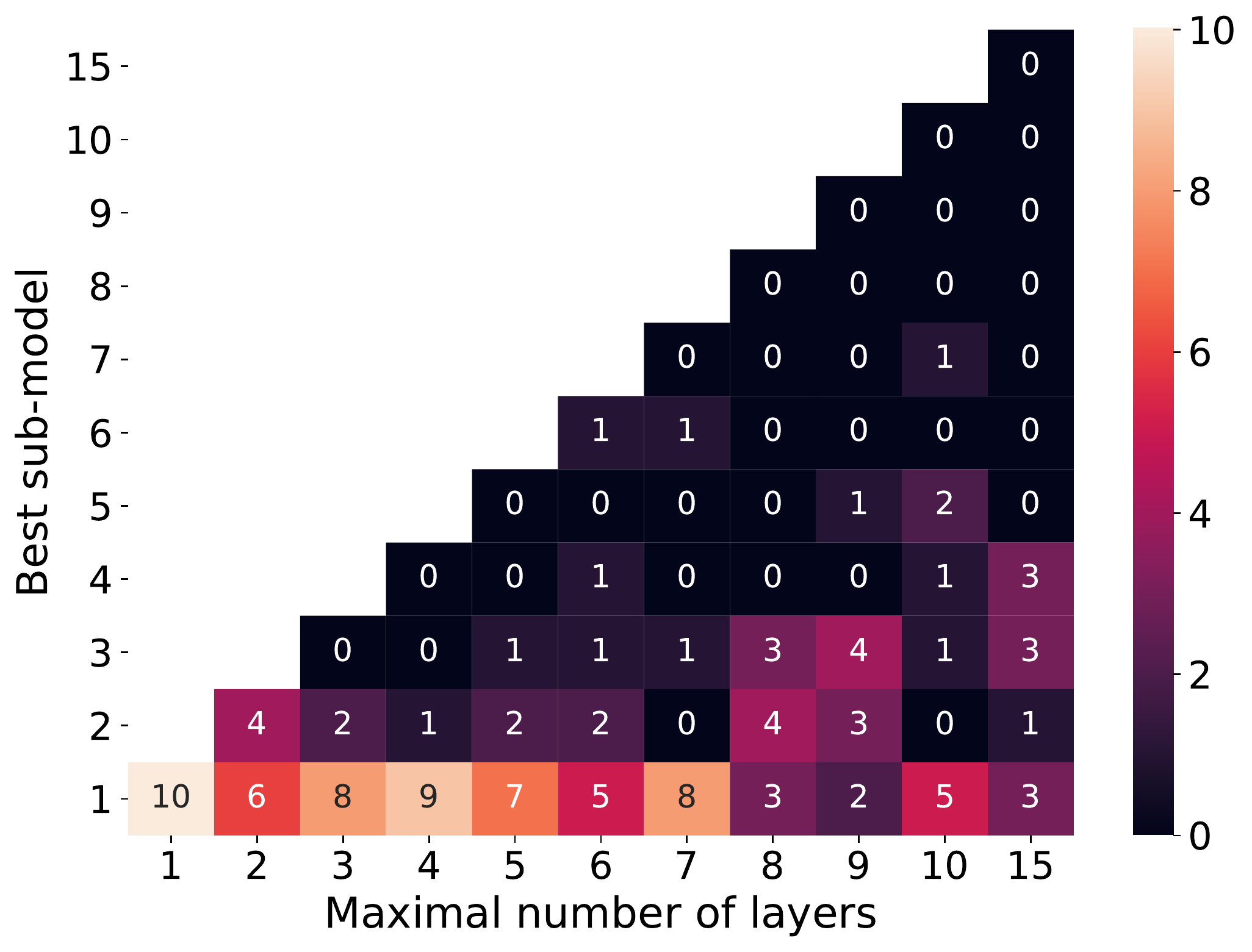}
    \includegraphics[width = 0.4\textwidth]{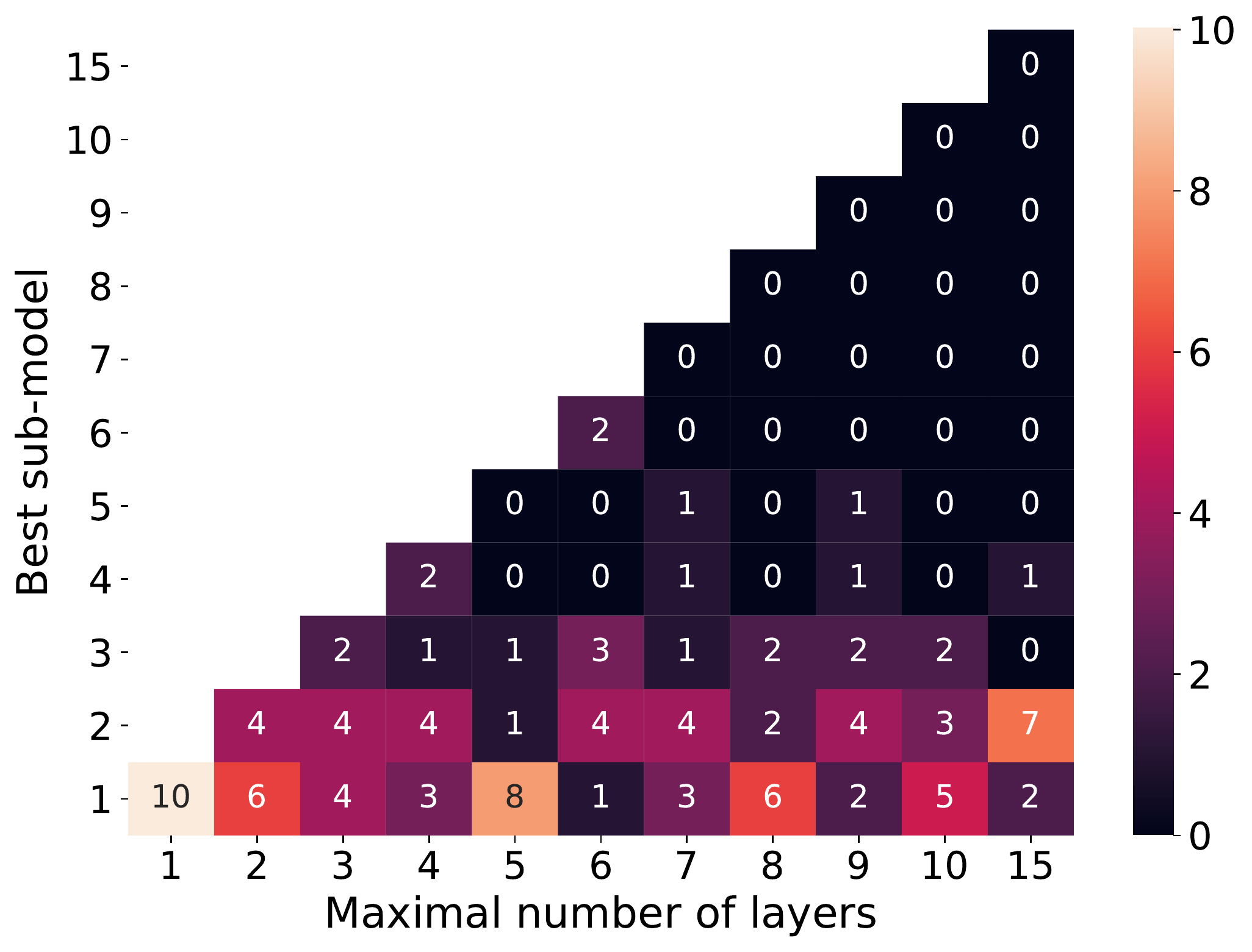}
    \caption{Higgs dataset. Heatmap counting the index of the sub-optimal model over 10 tries of a default DF with 1 (Breiman) forest per layer (left) or 4 forests (2 Breiman, 2 CRF) per layer (right), with respect to the maximal  number of layers.}
    \label{fig:higgs_heatmap}
\end{figure}

\begin{figure}[!htp]
    \centering
    \includegraphics[width = 0.4\textwidth]{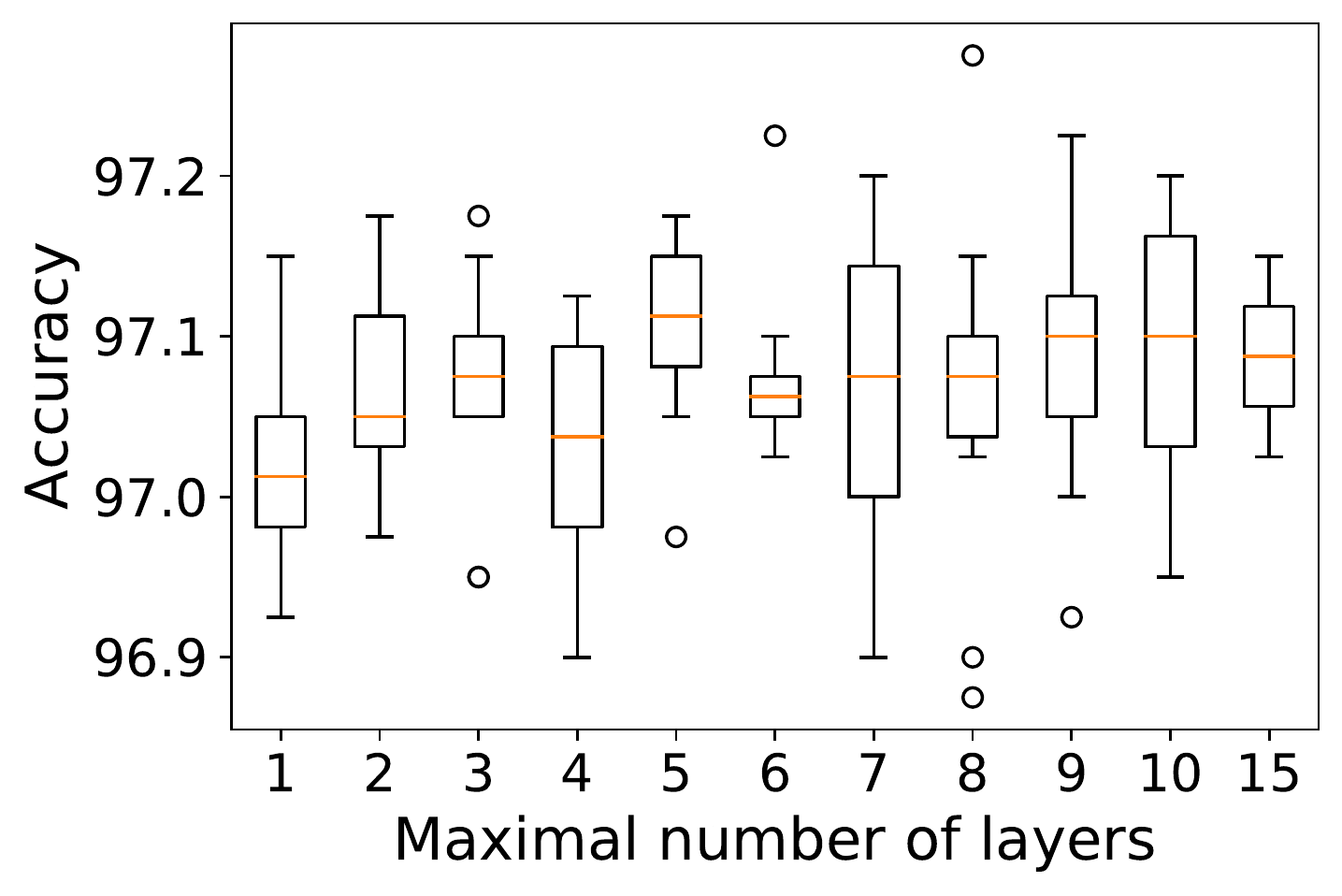}
    \includegraphics[width = 0.4\textwidth]{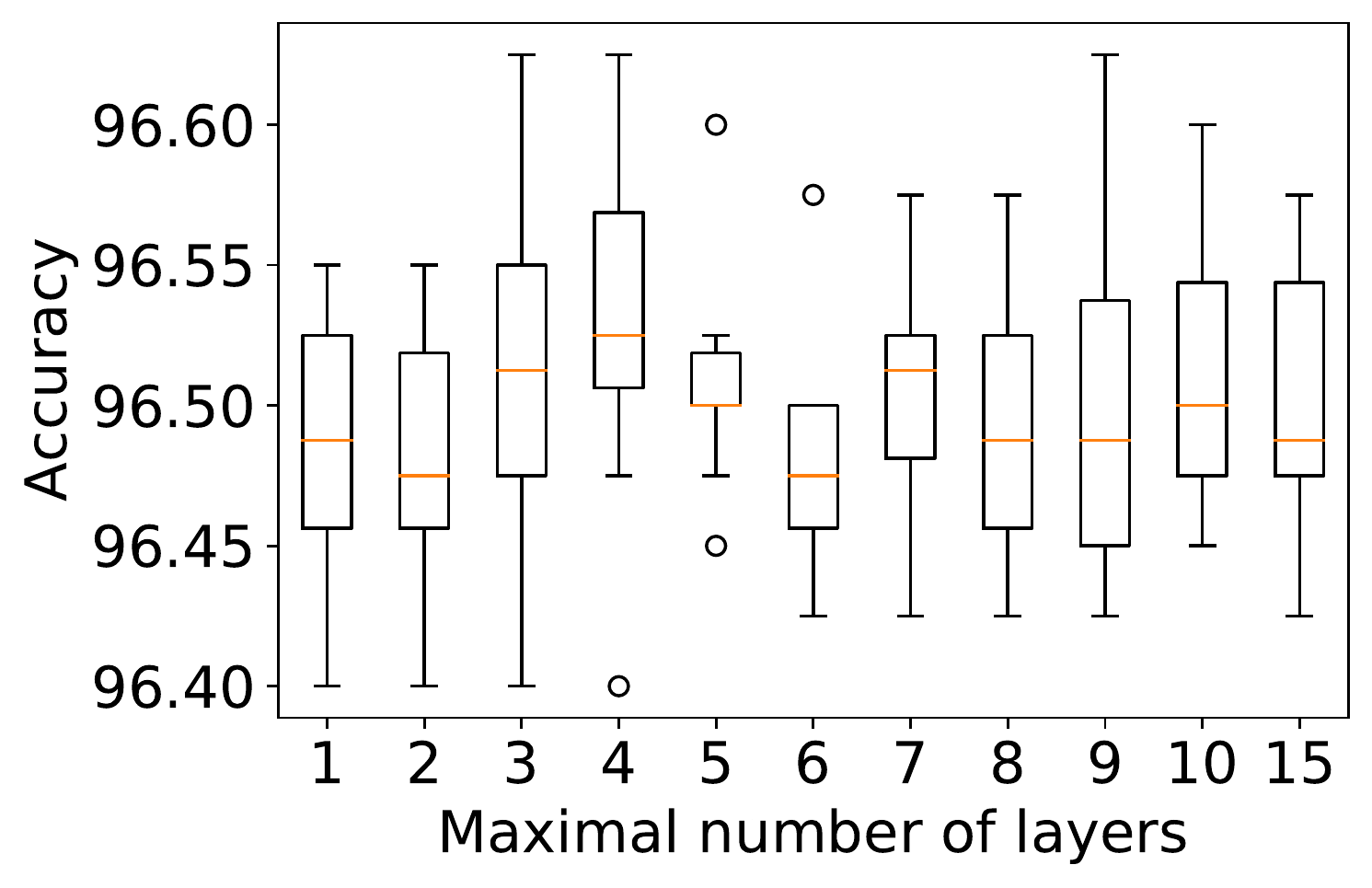}
    \caption{Letter dataset. Boxplots over 10 tries of the accuracy of a DF with 1 forest by layer (left) or 4 forests by layer (right), with respect to the DF maximal number of layers.}
    \label{fig:letter_accuracy_depth}
\end{figure}

\begin{figure}[!htp]
    \centering
    \includegraphics[width = 0.4\textwidth]{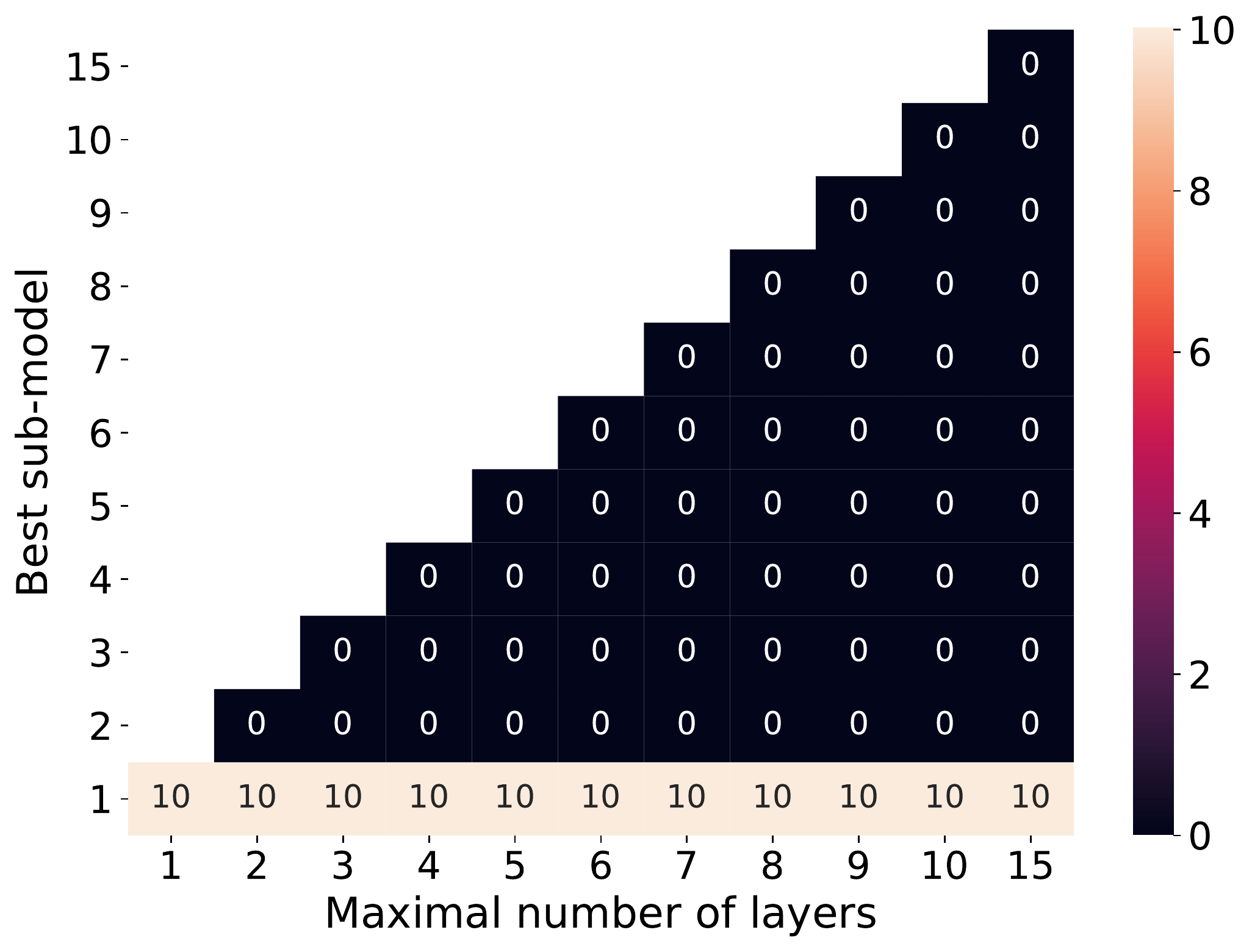}
    \includegraphics[width = 0.4\textwidth]{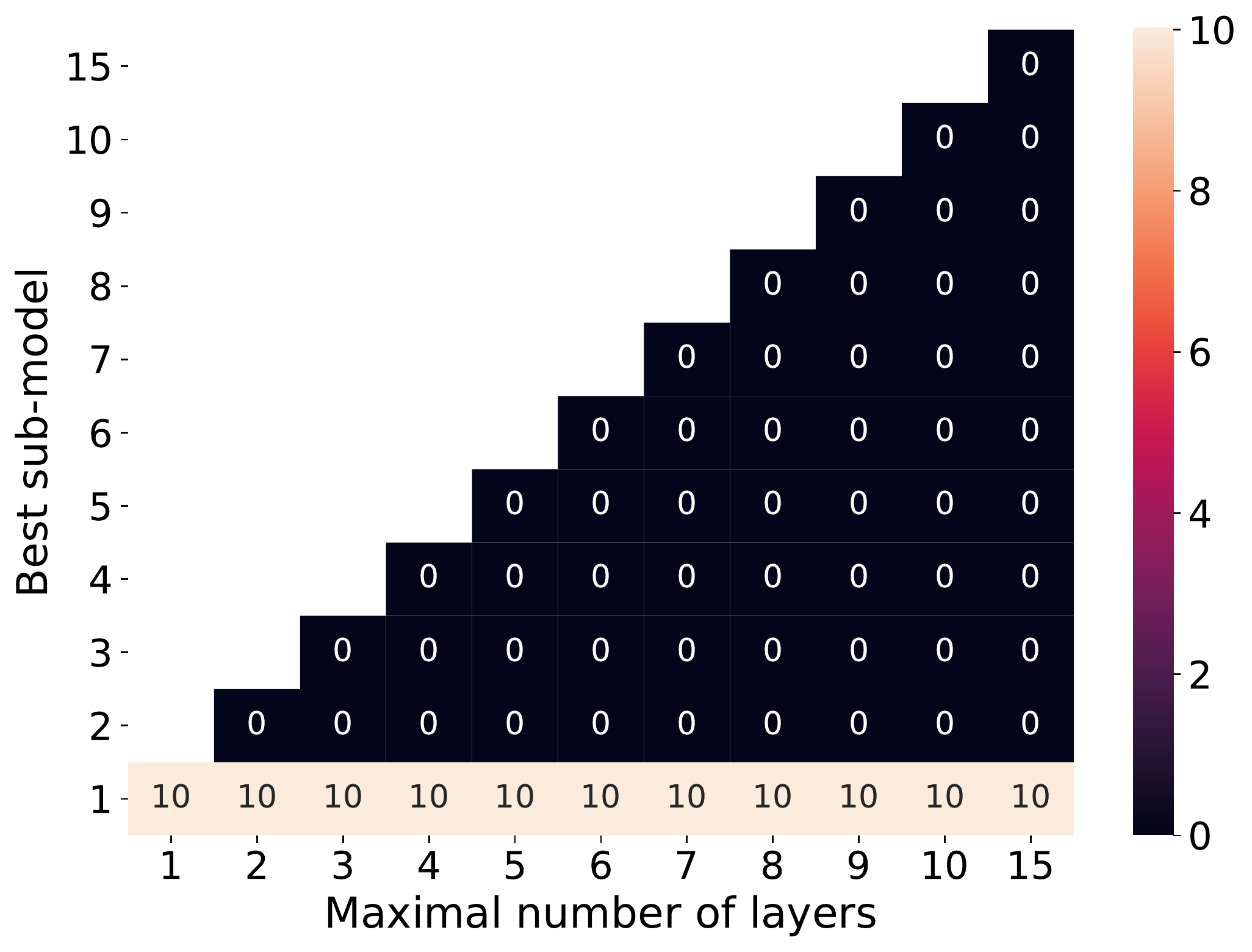}
    \caption{Letter dataset. Heatmap counting the index of the sub-optimal model over 10 tries of a default DF with 1 (Breiman) forest per layer (left) or 4 forests (2 Breiman, 2 CRF) per layer (right), with respect to the maximal  number of layers.}
    \label{fig:letter_heatmap}
\end{figure}

\begin{figure}[!htp]
    \centering
    \includegraphics[width = 0.4\textwidth]{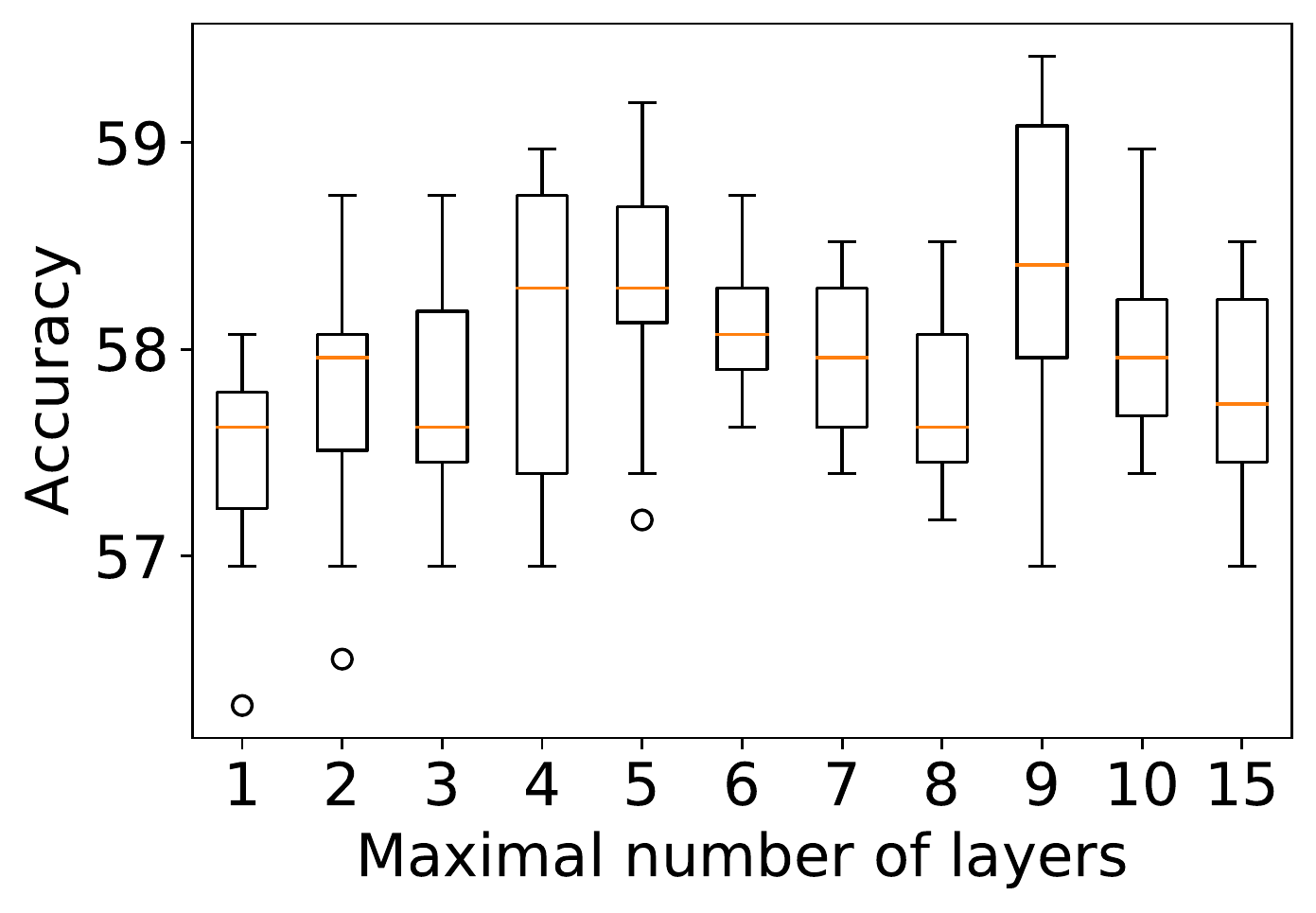}
    \includegraphics[width = 0.4\textwidth]{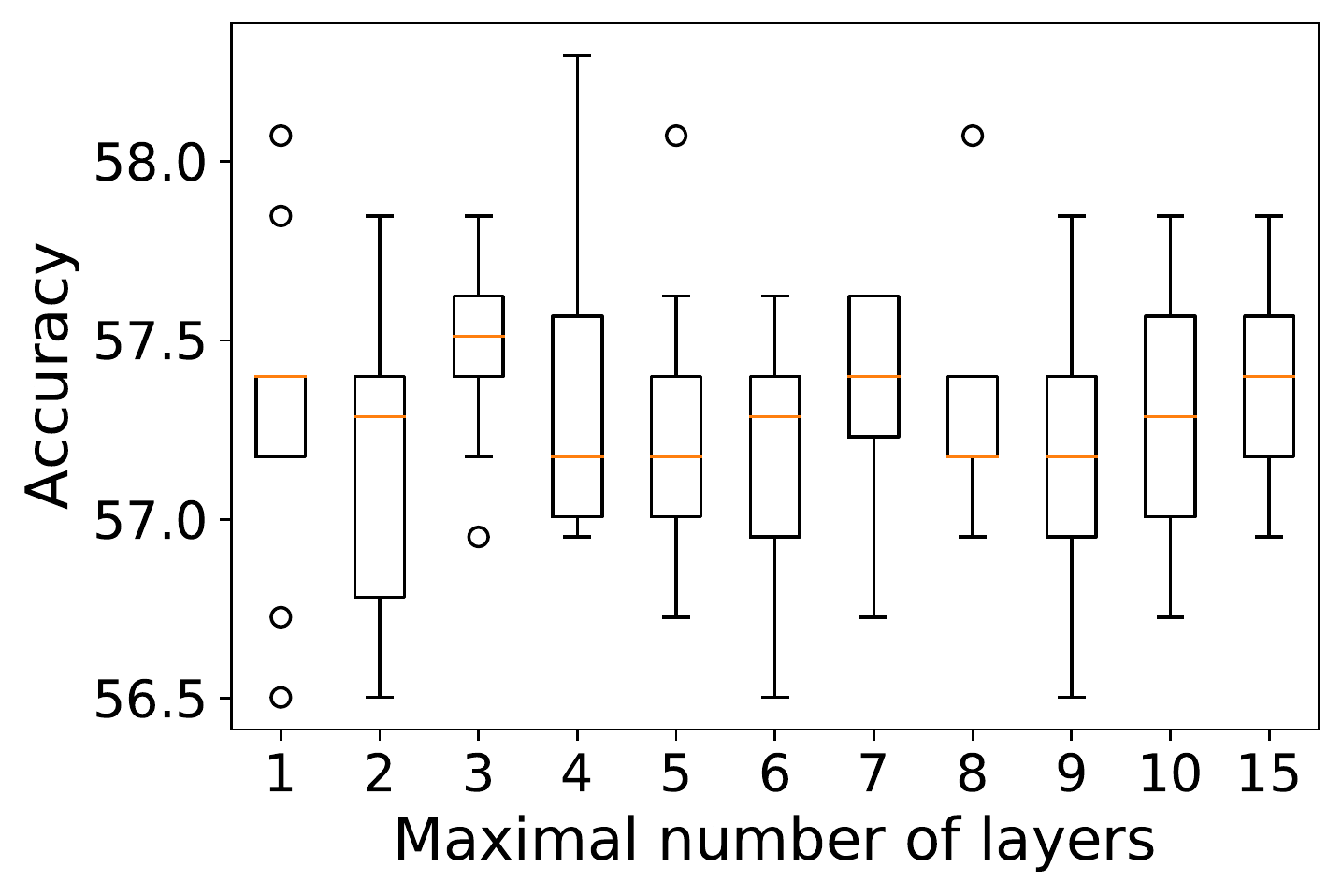}
    \caption{Yeast dataset. Boxplots over 10 tries of the accuracy of a DF with 1 forest by layer (left) or 4 forests by layer (right), with respect to the DF maximal number of layers.}
    \label{fig:yeast_acc_depth}
\end{figure}

\begin{figure}[!htp]
    \centering
    \includegraphics[width = 0.4\textwidth]{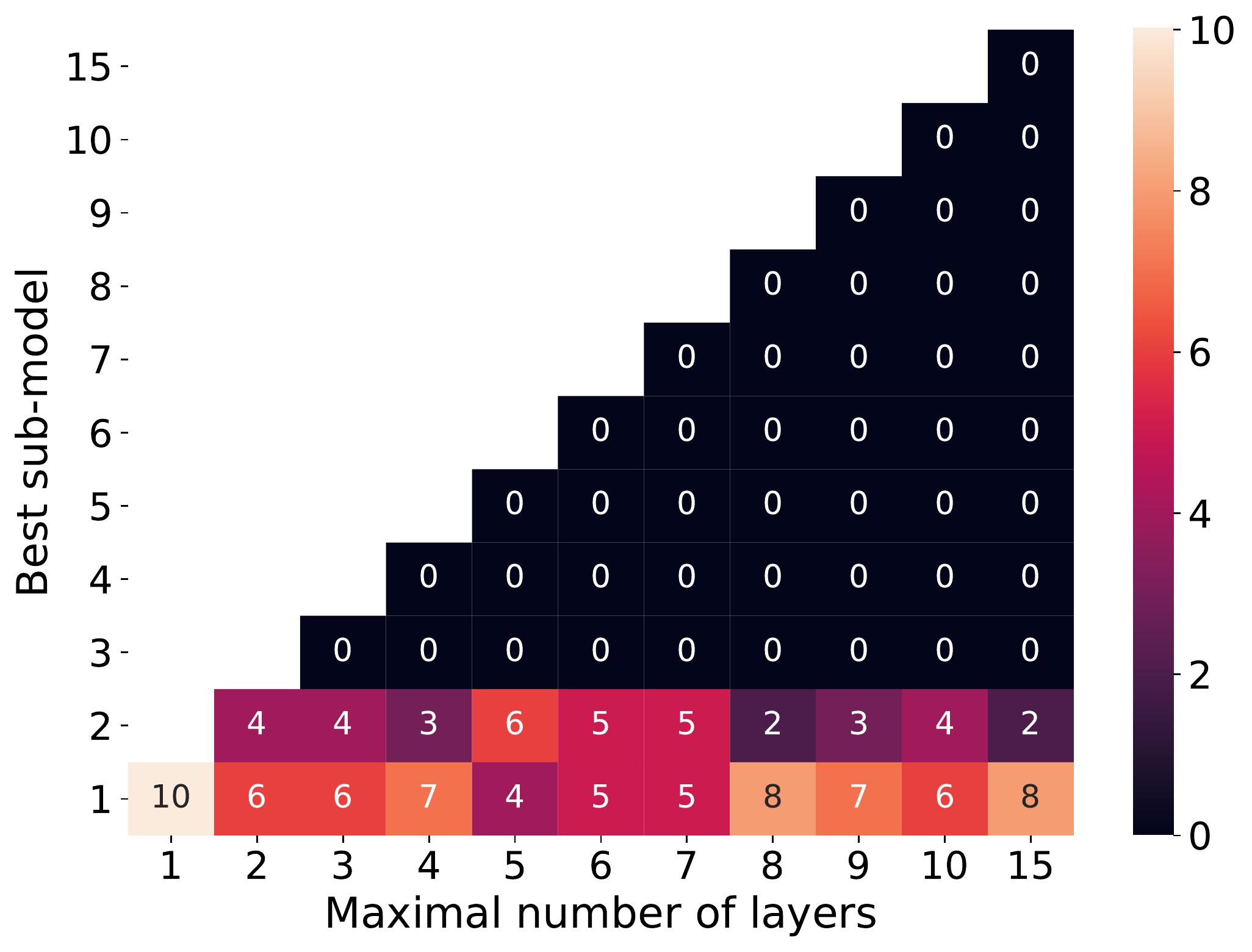}
    \includegraphics[width = 0.4\textwidth]{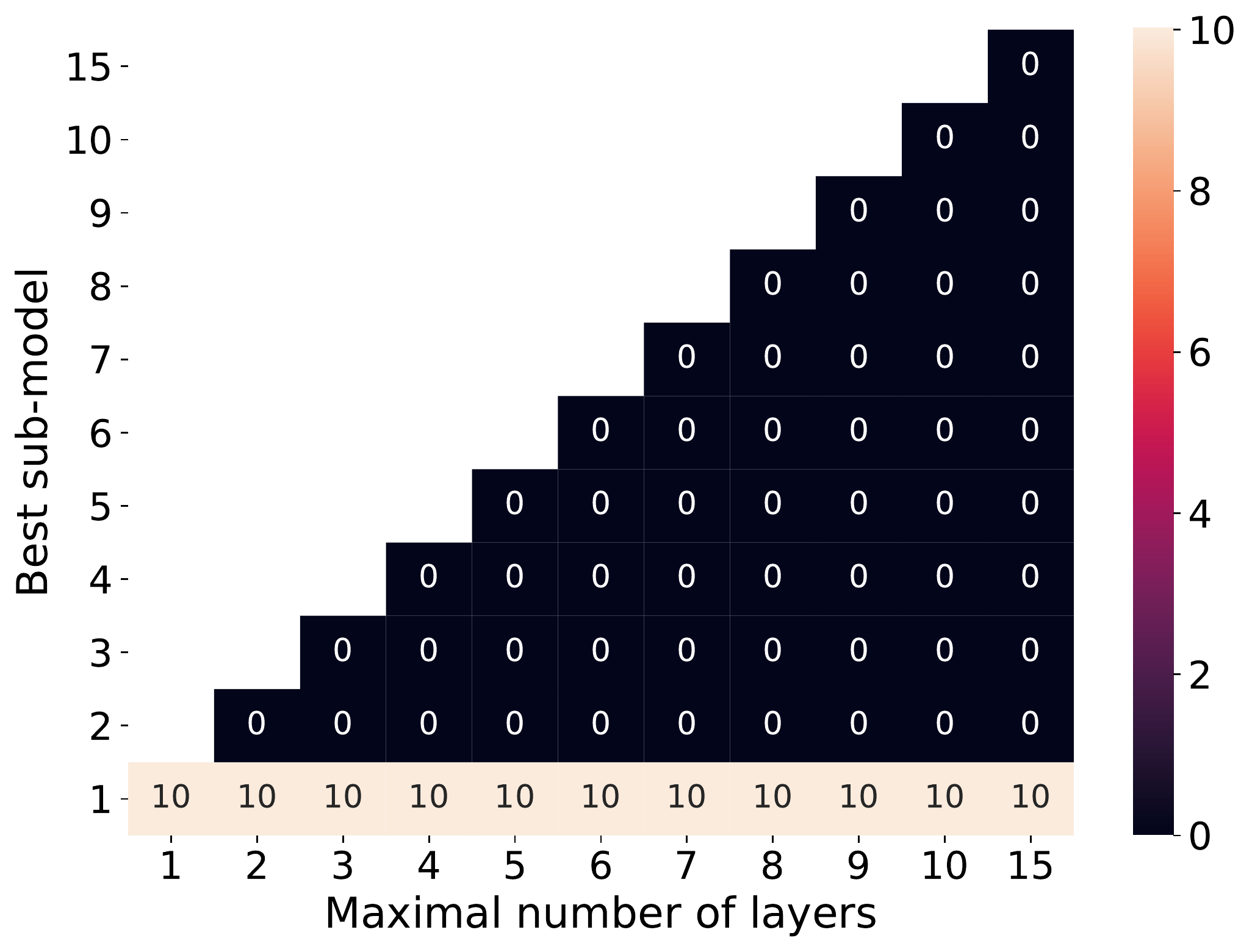}
    \caption{Yeast dataset. Heatmap counting the index of the sub-optimal model over 10 tries of a default DF with 1 (Breiman) forest per layer (left) or 4 forests (2 Breiman, 2 CRF) per layer (right), with respect to the maximal  number of layers.}
    \label{fig:yeast_heatmap}
\end{figure}

\begin{figure}[!htp]
    \centering
    \includegraphics[width = 0.4\textwidth]{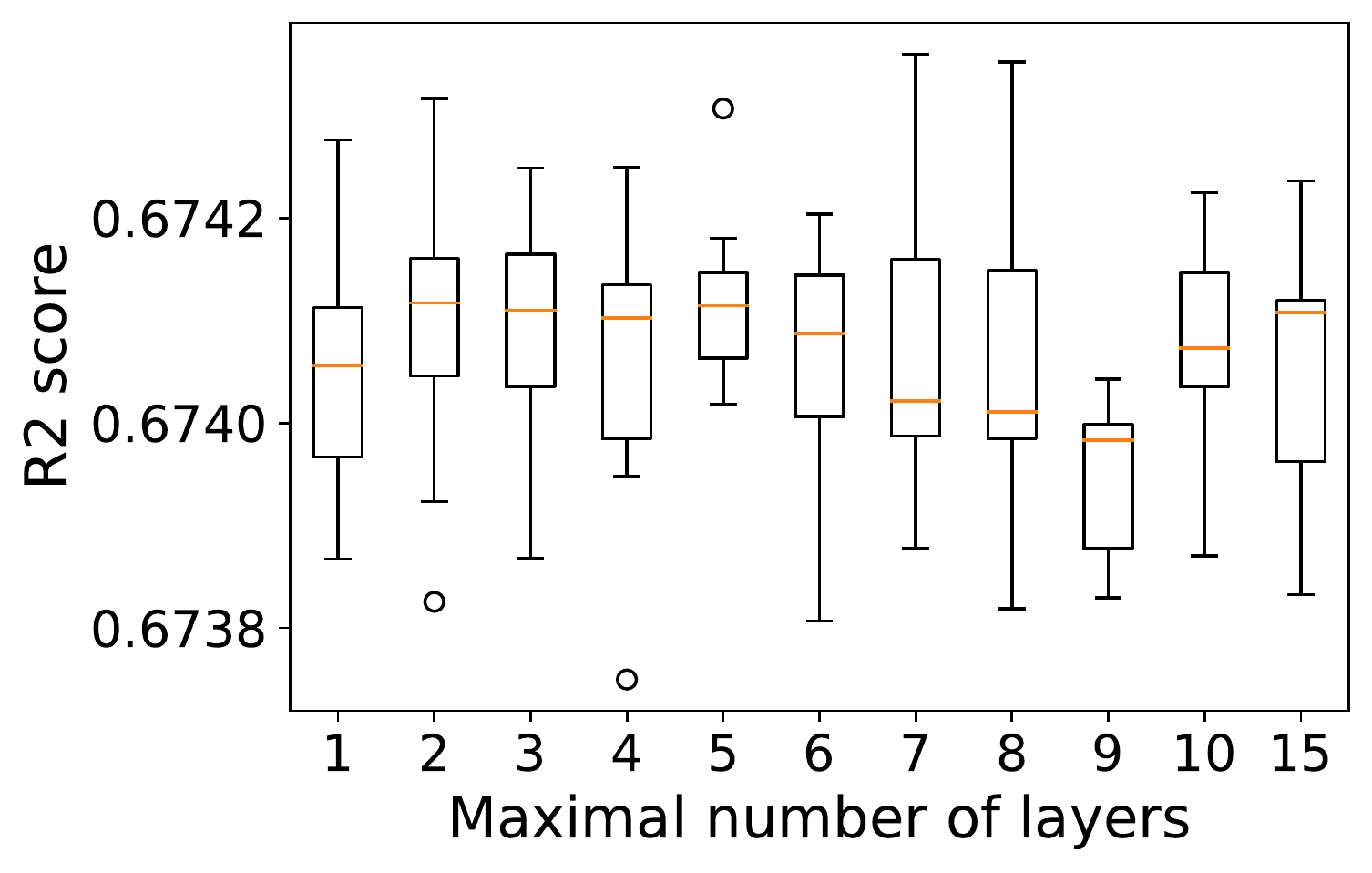}
    \includegraphics[width = 0.4\textwidth]{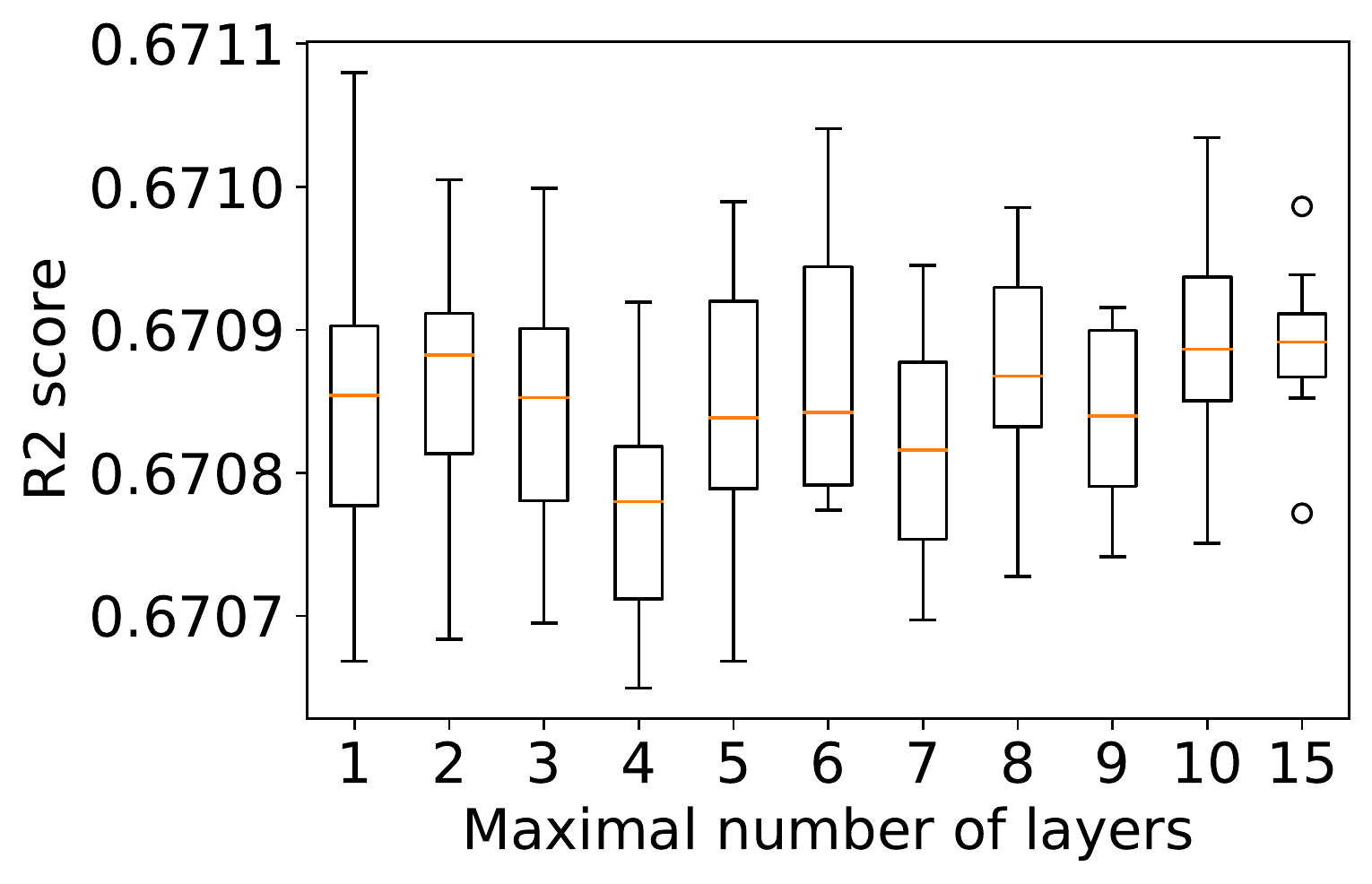}
    \caption{Airbnb dataset. Boxplots over 10 tries of the accuracy of a DF with 1 forest by layer (left) or 4 forests by layer (right), with respect to the DF maximal number of layers.}
    \label{fig:airbnb_acc_depth}
\end{figure}

\begin{figure}[!htp]
    \centering
    \includegraphics[width = 0.4\textwidth]{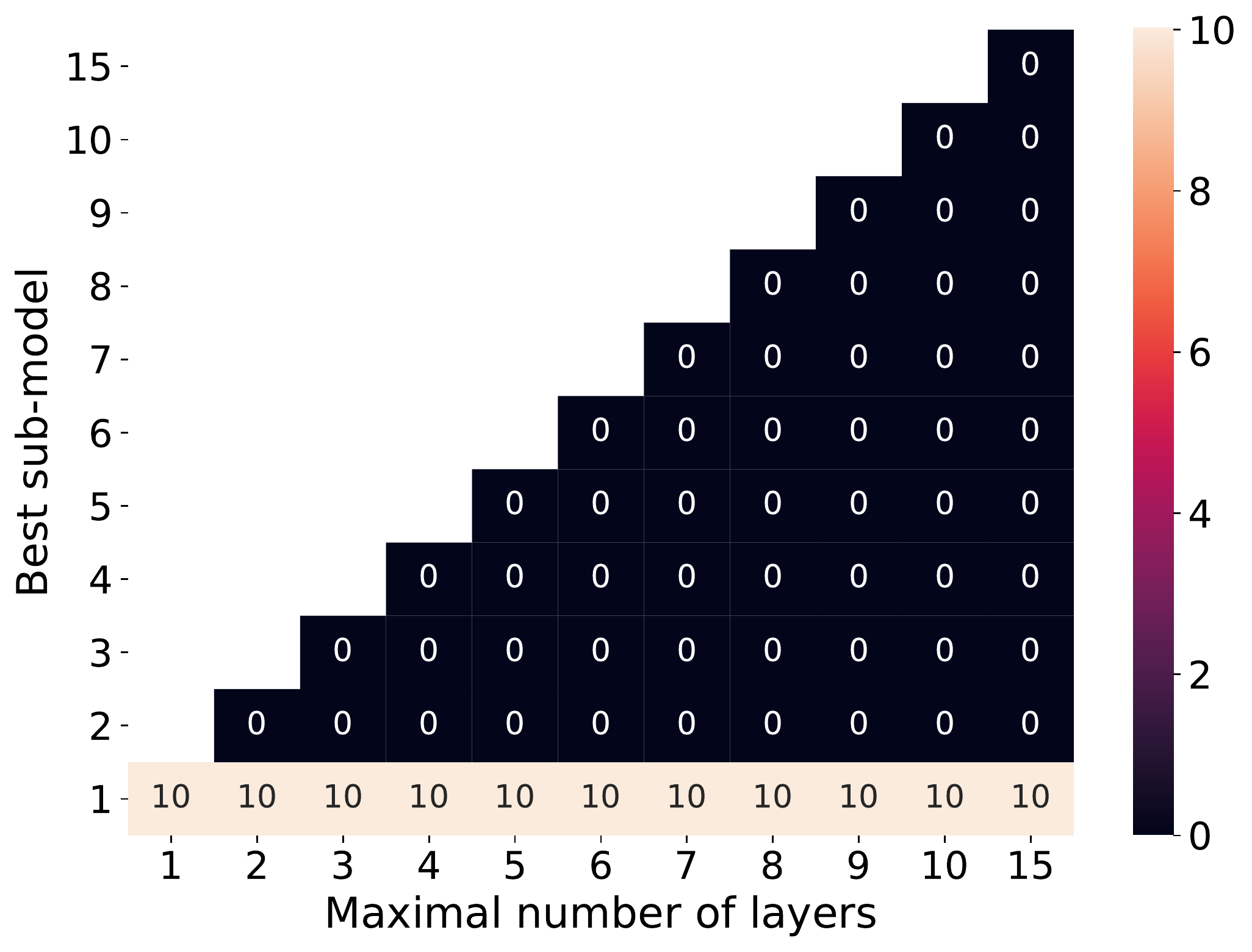}
    \includegraphics[width = 0.4\textwidth]{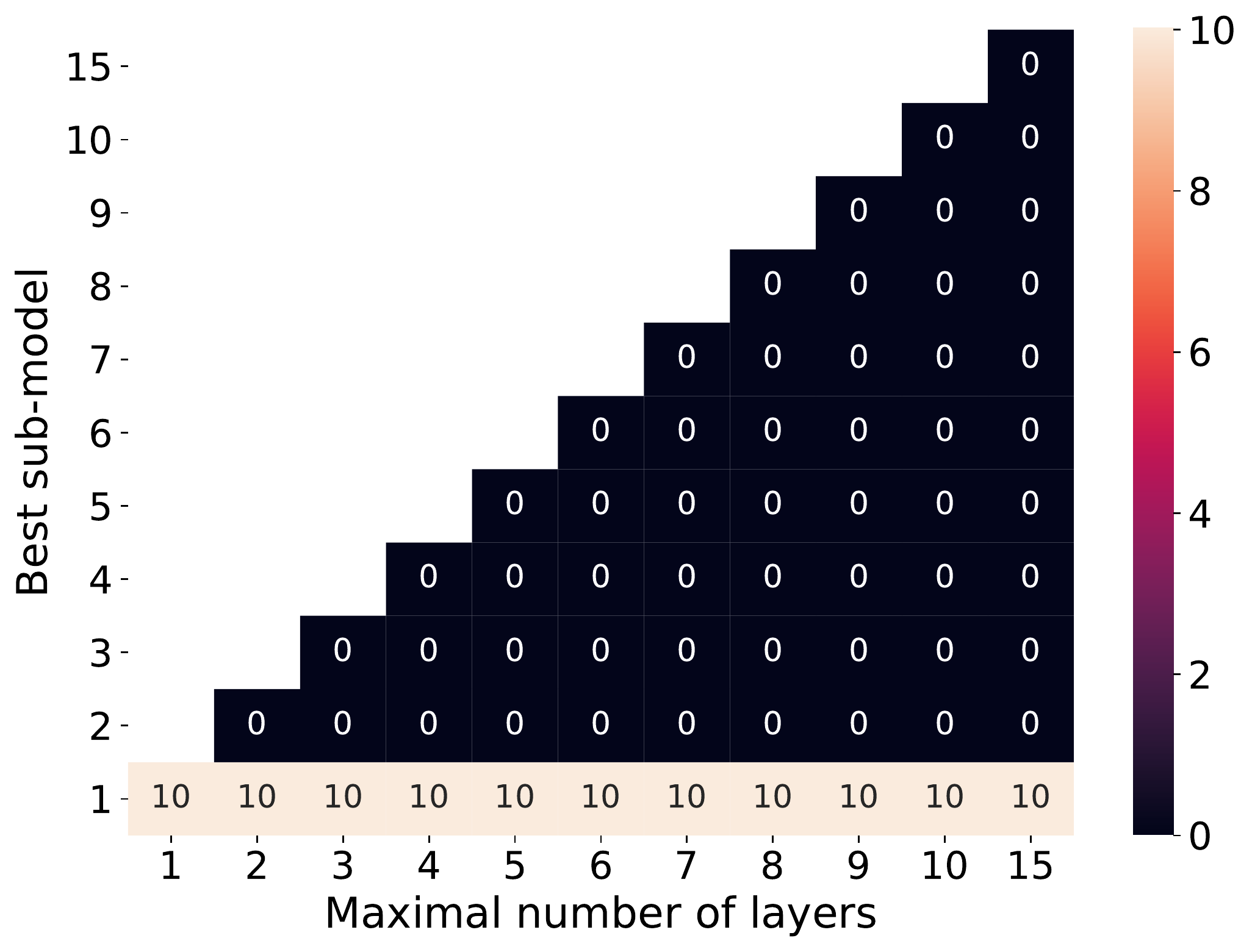}
    \caption{Airbnb datase. Heatmap counting the index of the sub-optimal model over 10 tries of a default DF with 1 (Breiman) forest per layer (left) or 4 forests (2 Breiman, 2 CRF) per layer (right), with respect to the maximal  number of layers.}
    \label{fig:airbnb_heatmap}
\end{figure}

\begin{figure}[!htp]
    \centering
    \includegraphics[width = 0.4\textwidth]{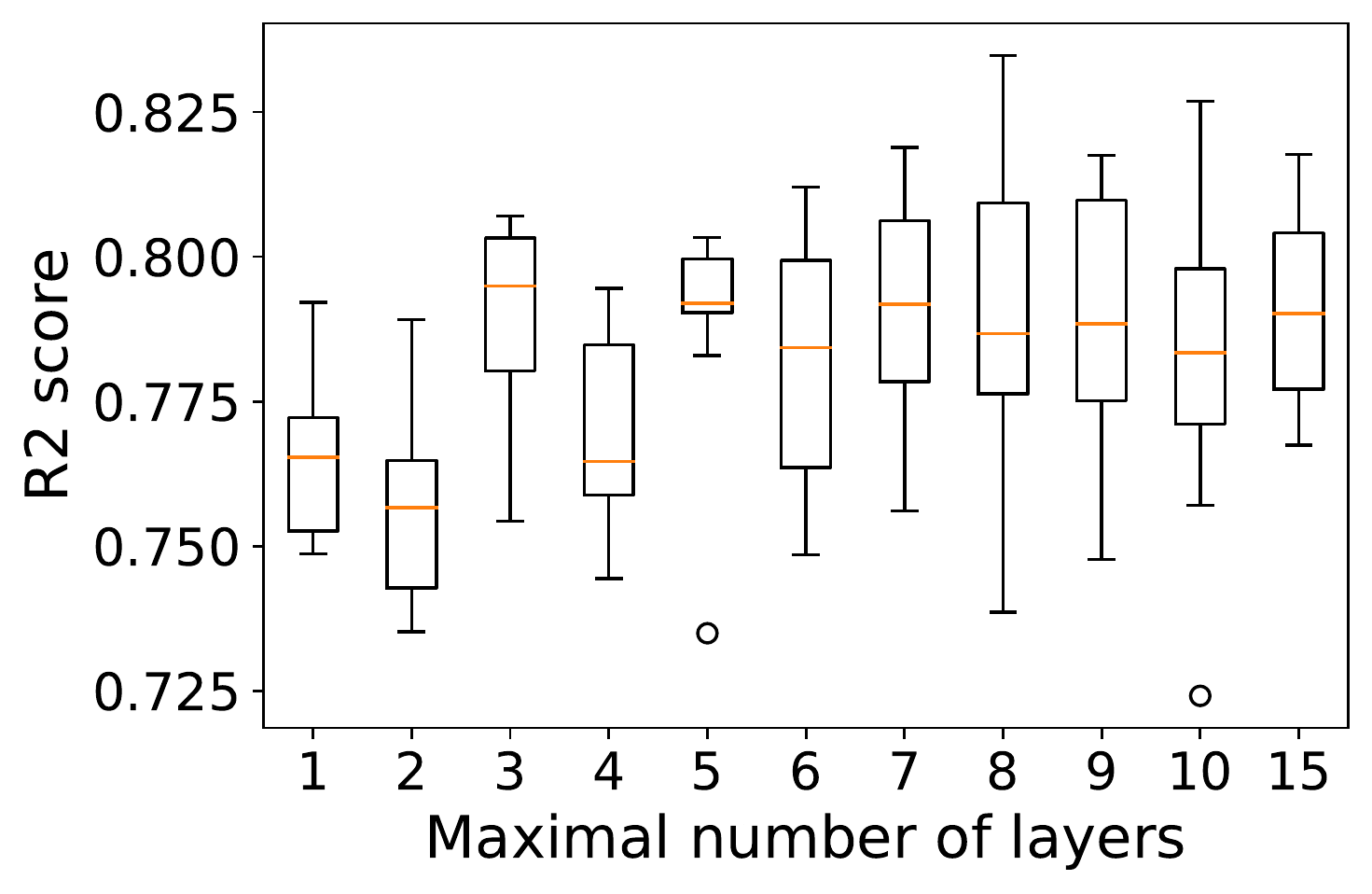}
    \includegraphics[width = 0.4\textwidth]{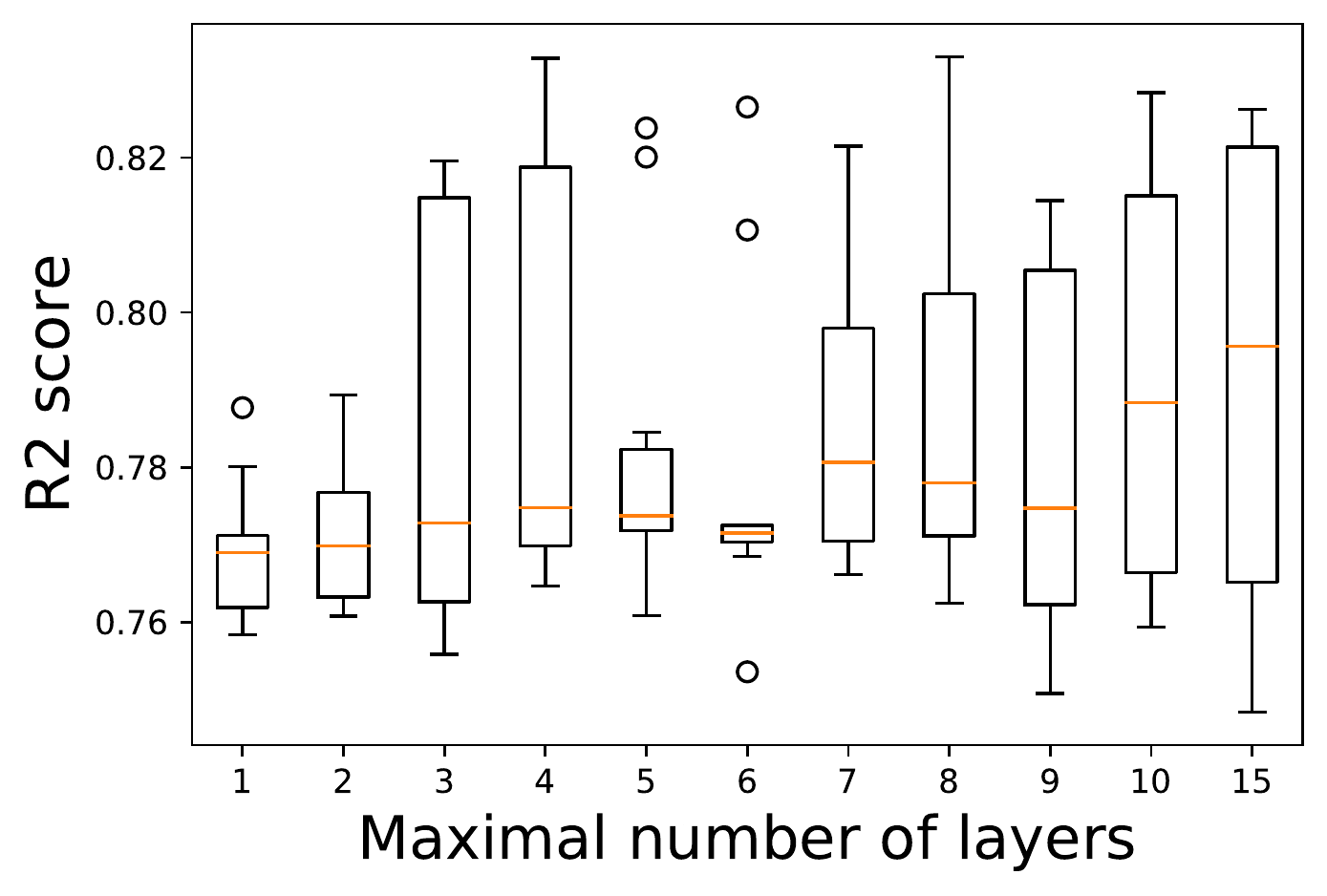}
    \caption{Housing dataset. Boxplots over 10 tries of the accuracy of a DF with 1 forest by layer (left) or 4 forests by layer (right), with respect to the DF maximal number of layers.}
    \label{fig:housing_acc_depth}
\end{figure}

\begin{figure}[!htp]
    \centering
    \includegraphics[width = 0.4\textwidth]{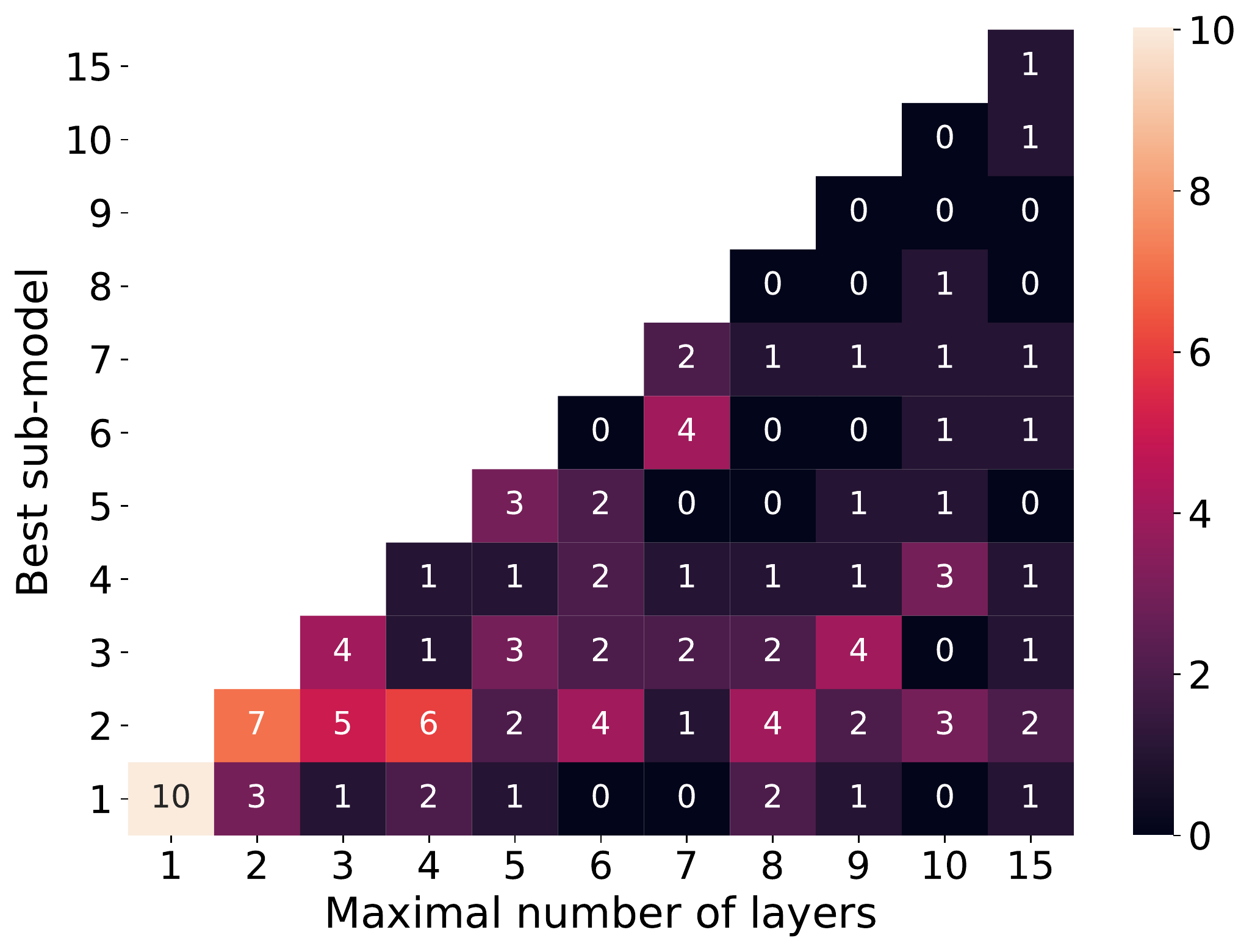}
    \includegraphics[width = 0.4\textwidth]{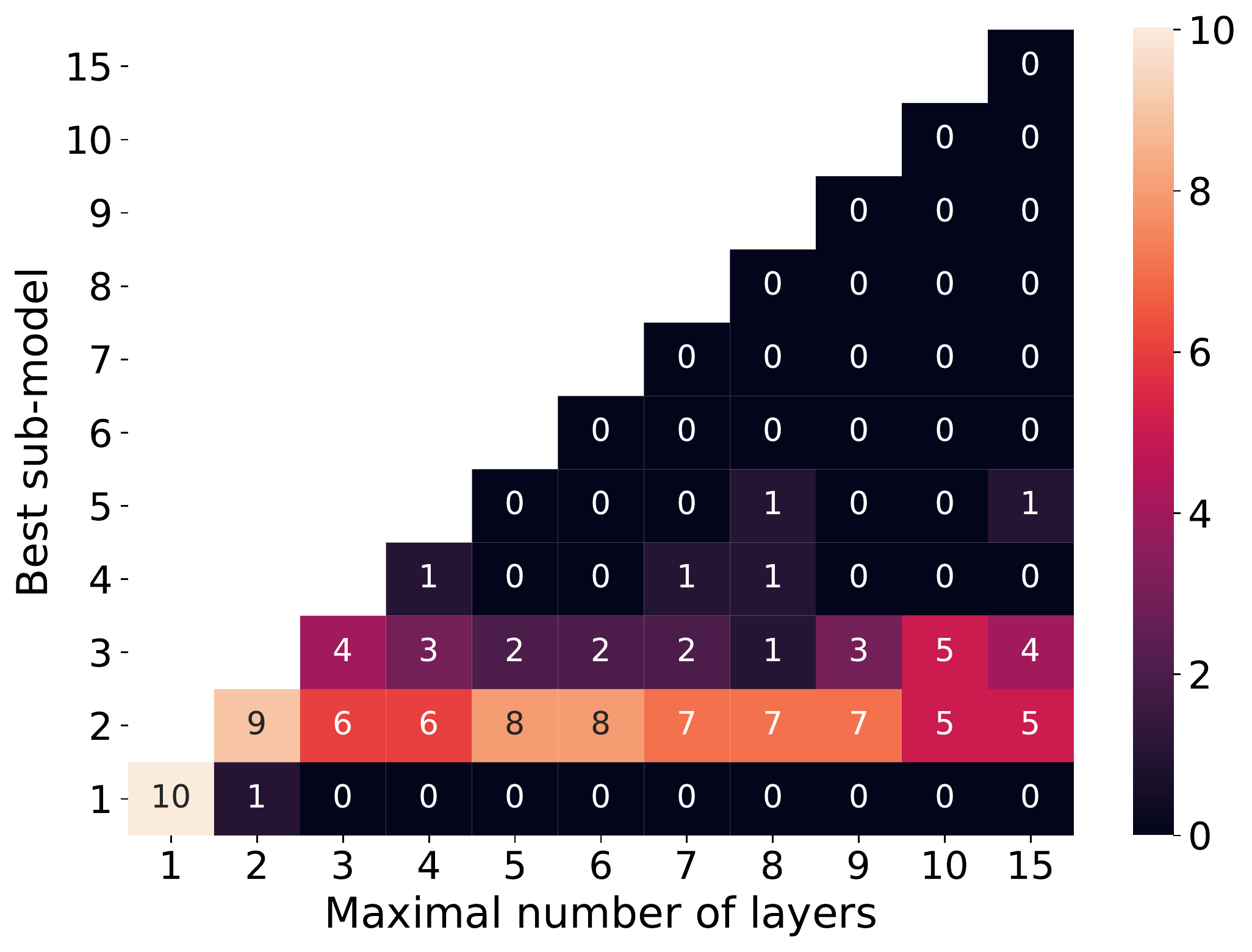}
    \caption{Housing dataset. Heatmap counting the index of the sub-optimal model over 10 tries of a default DF with 1 (Breiman) forest per layer (left) or 4 forests (2 Breiman, 2 CRF) per layer (right), with respect to the maximal  number of layers.}
    \label{fig:housing_heatmap}
\end{figure}

\clearpage

\subsection{Additional figures to Section \ref{sec:theoretical_results}}
\label{app:sec:theoretical_results}
\begin{figure}[!ht]
    \centering
    \includegraphics[width = 0.7\textwidth]{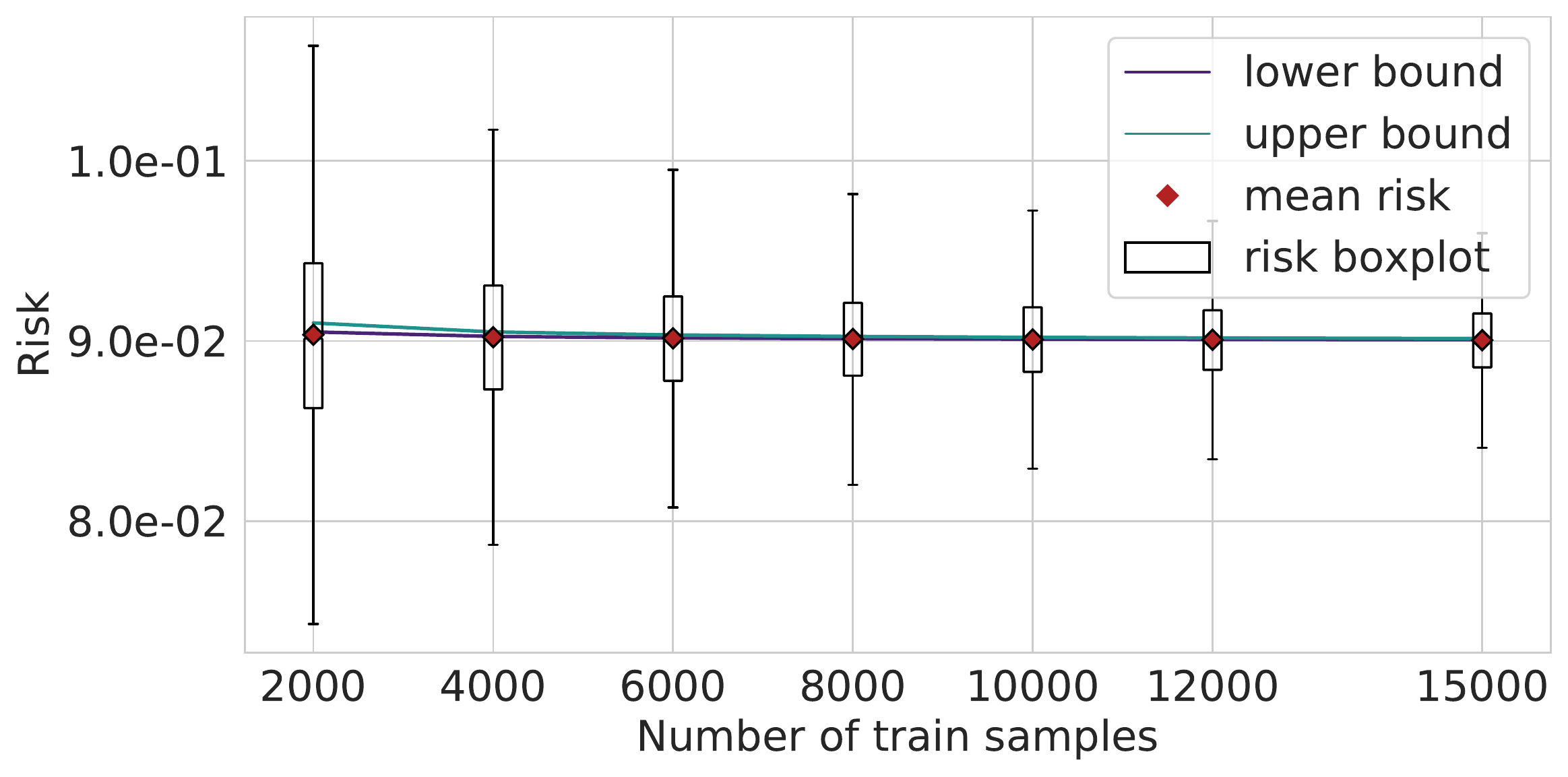}
    \caption{Illustration of the theoretical bounds for a single tree of Proposition \ref{prop:balanced_chessboard} 1. for a chessboard with parameters $k^\star=4$, $N_{\mathcal{B}}=2^{k^\star-1}$, and $p=0.8$.  
    The single tree is of depth $k=2$. 
    We draw a sample of size $n$ (x-axis), and a single tree $r_{k,0,n}$ is fitted for which the theoretical risk is evaluated. Each boxplot is built out of 20 000 repetitions. The outliers are not shown for the sake of presentation.}
    \label{fig:rtree_bound_small_depth}
\end{figure}

\begin{figure}[!ht]
    \centering
    \includegraphics[width = 0.7\textwidth]{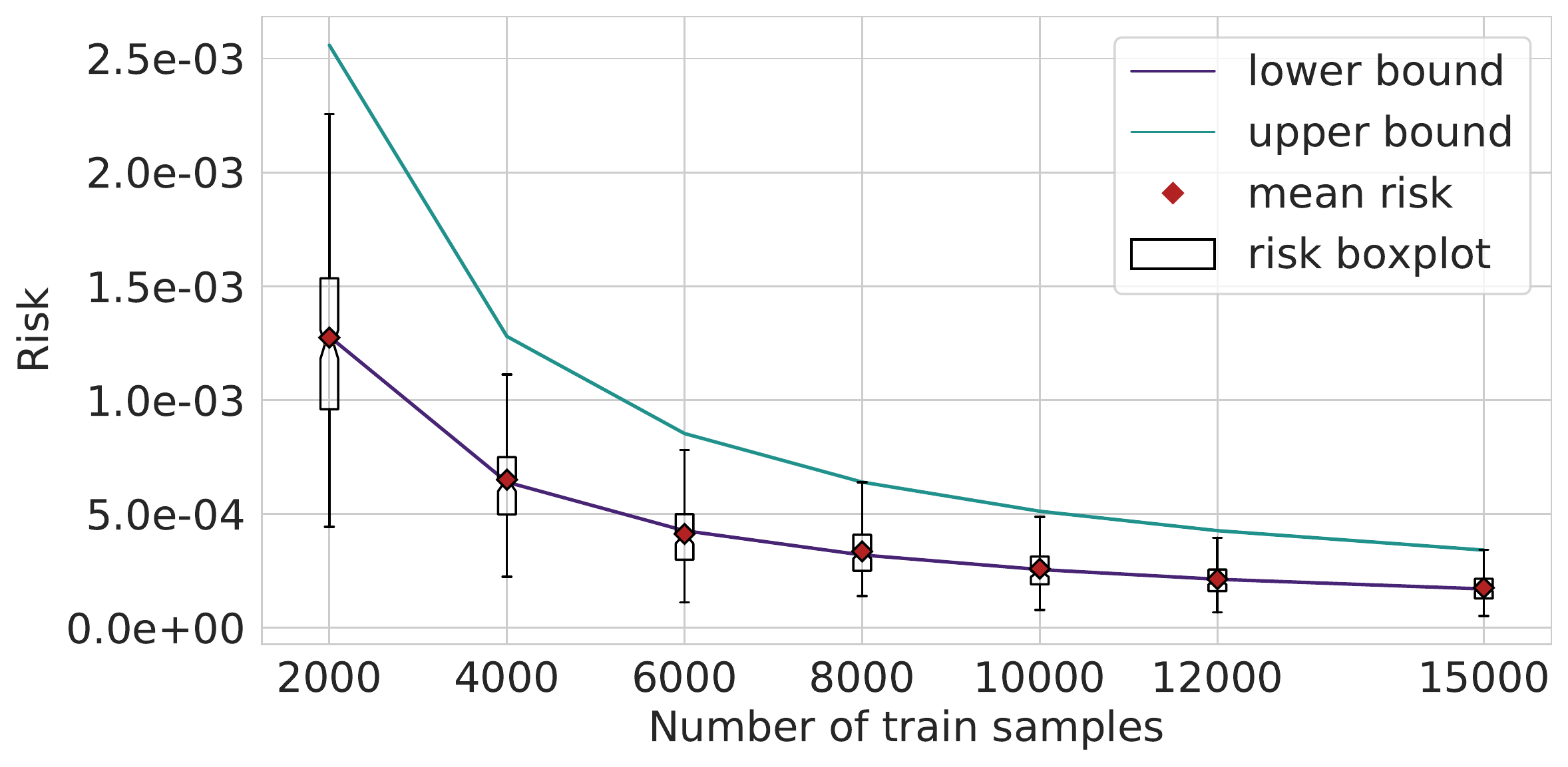}
    \caption{Illustration of the theoretical bounds for a single tree of Proposition \ref{prop:new_risk_large_k} 1. for a chessboard with parameters $k^\star=4$, $N_{\mathcal{B}}=2^{k^\star-1}$ and $p =0.8$. The single tree is of depth $k=4$. We draw a sample of size $n$ (x-axis), and a single tree $r_{k,0,n}$ is fitted for which the theoretical risk is evaluated. Each boxplot is built out of 20 000 repetitions. The outliers are not shown for the sake of presentation.}
    \label{fig:rtree_bound_perfect_depth}
\end{figure}

\begin{figure}[!ht]
    \centering
    \includegraphics[width = 0.7\textwidth]{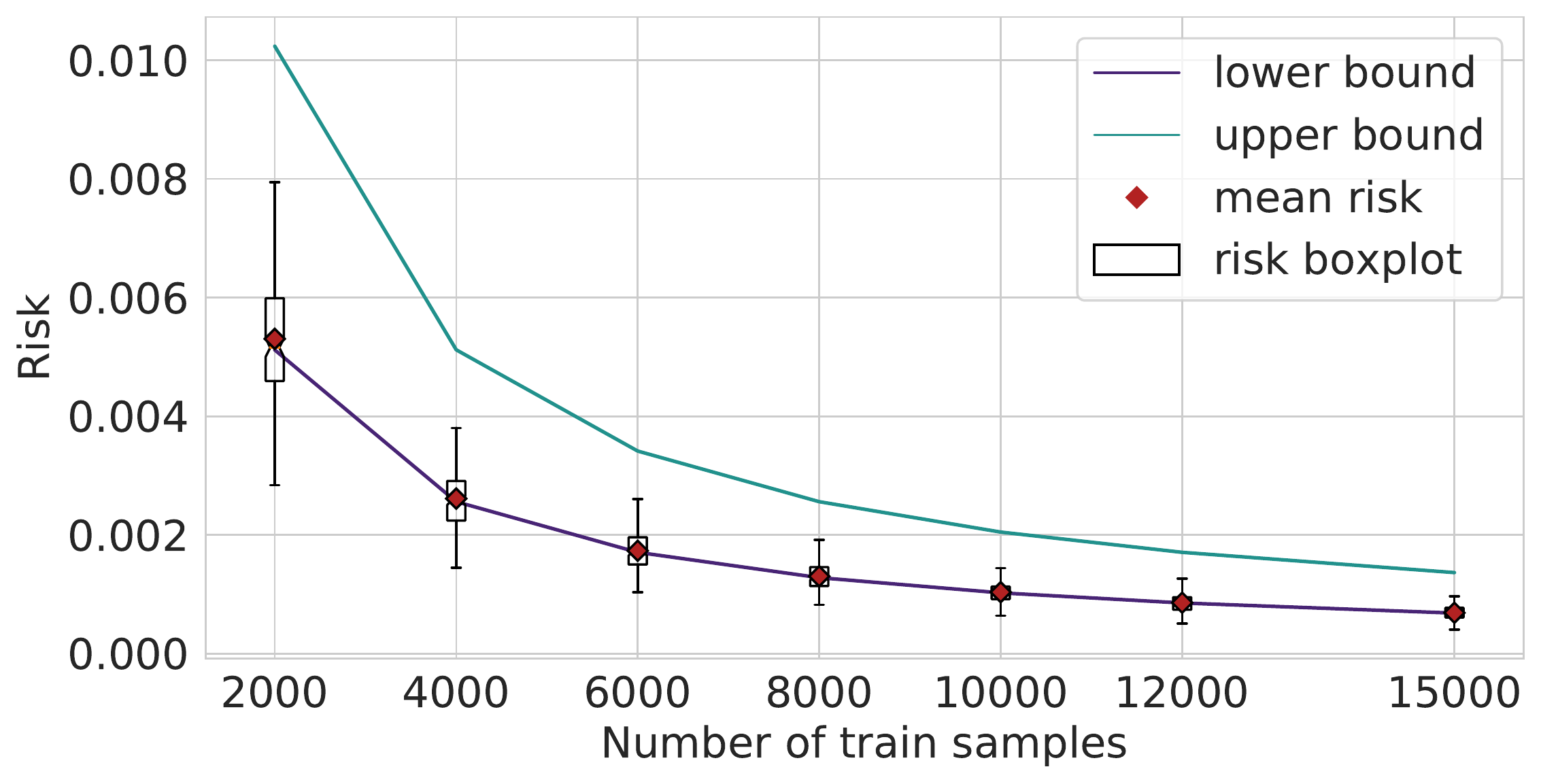}
    \caption{Illustration of the theoretical bounds for a single tree of Proposition \ref{prop:new_risk_large_k} 1. for a chessboard with parameters $k^\star=4$, $N_{\mathcal{B}}=2^{k^\star-1}$ and $p =0.8$. The single tree is of depth $k=6$. We draw a sample of size $n$ (x-axis), and a single tree $r_{k,0,n}$ is fitted for which the theoretical risk is evaluated. Each boxplot is built out of 20 000 repetitions. The outliers are not shown for the sake of presentation.}
    \label{fig:rtree_bound_big_depth}
\end{figure}

\begin{figure}[!ht]
    \centering
    \includegraphics[width = 0.7\textwidth]{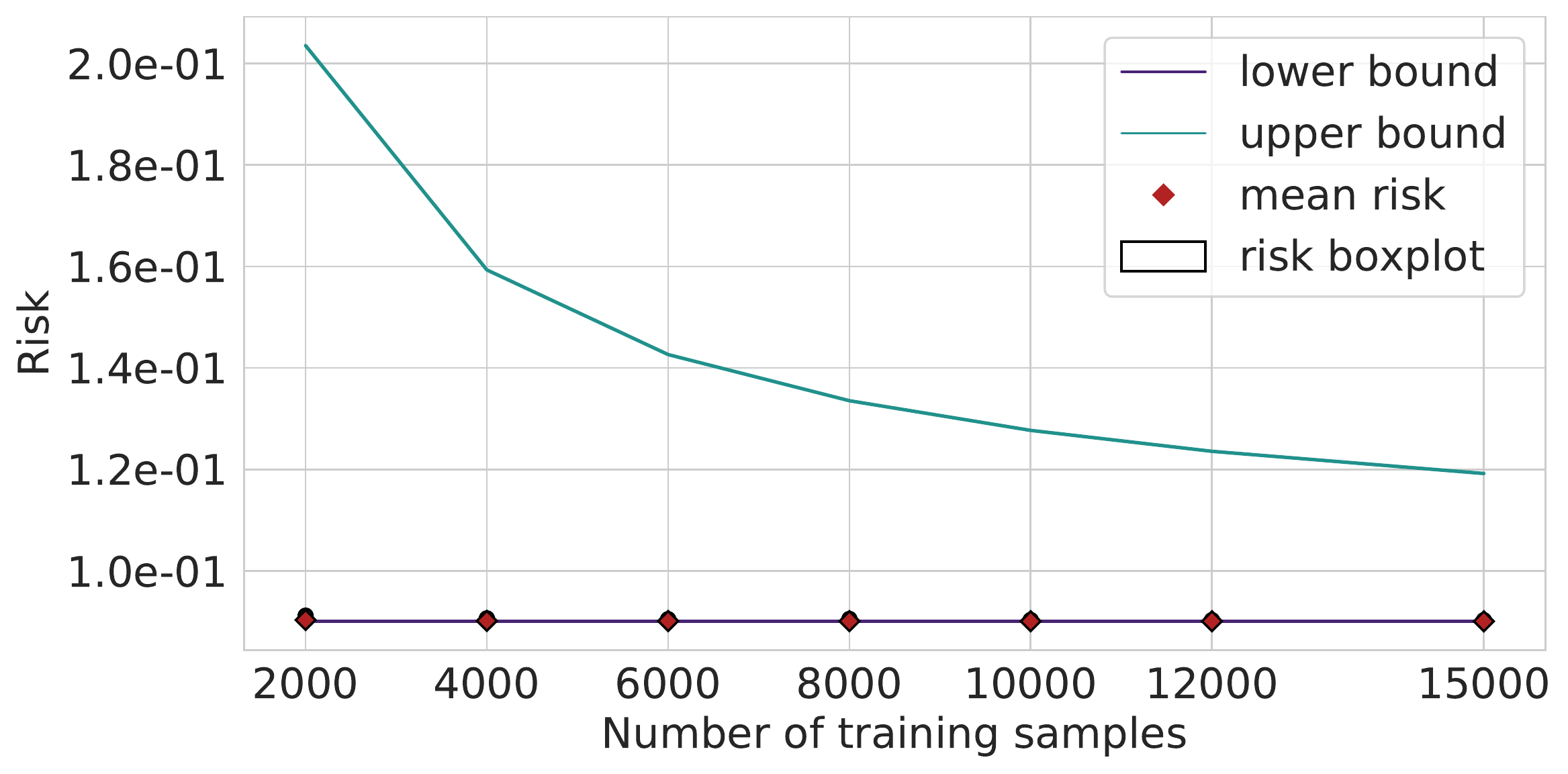}
    \caption{Illustration of the theoretical bounds for a shallow tree network of Proposition \ref{prop:balanced_chessboard} 2. for a chessboard with parameters $k^\star=4$, $N_{\mathcal{B}}=2^{k^\star-1}$ and $p =0.8$. The first-layer tree is of depth $k=2$. We draw a sample of size $n$ (x-axis), and a single tree $r_{k,0,n}$ is fitted for which the theoretical risk is evaluated. Each boxplot is built out of 20 000 repetitions. The outliers are not shown for the sake of presentation. Note that in such a case, the theoretical lower bound is constant and equal to the bias term.}
    \label{fig:etree_bound_small_depth}
\end{figure}

\begin{figure}[!ht]
    \centering
    \includegraphics[width = 0.7\textwidth]{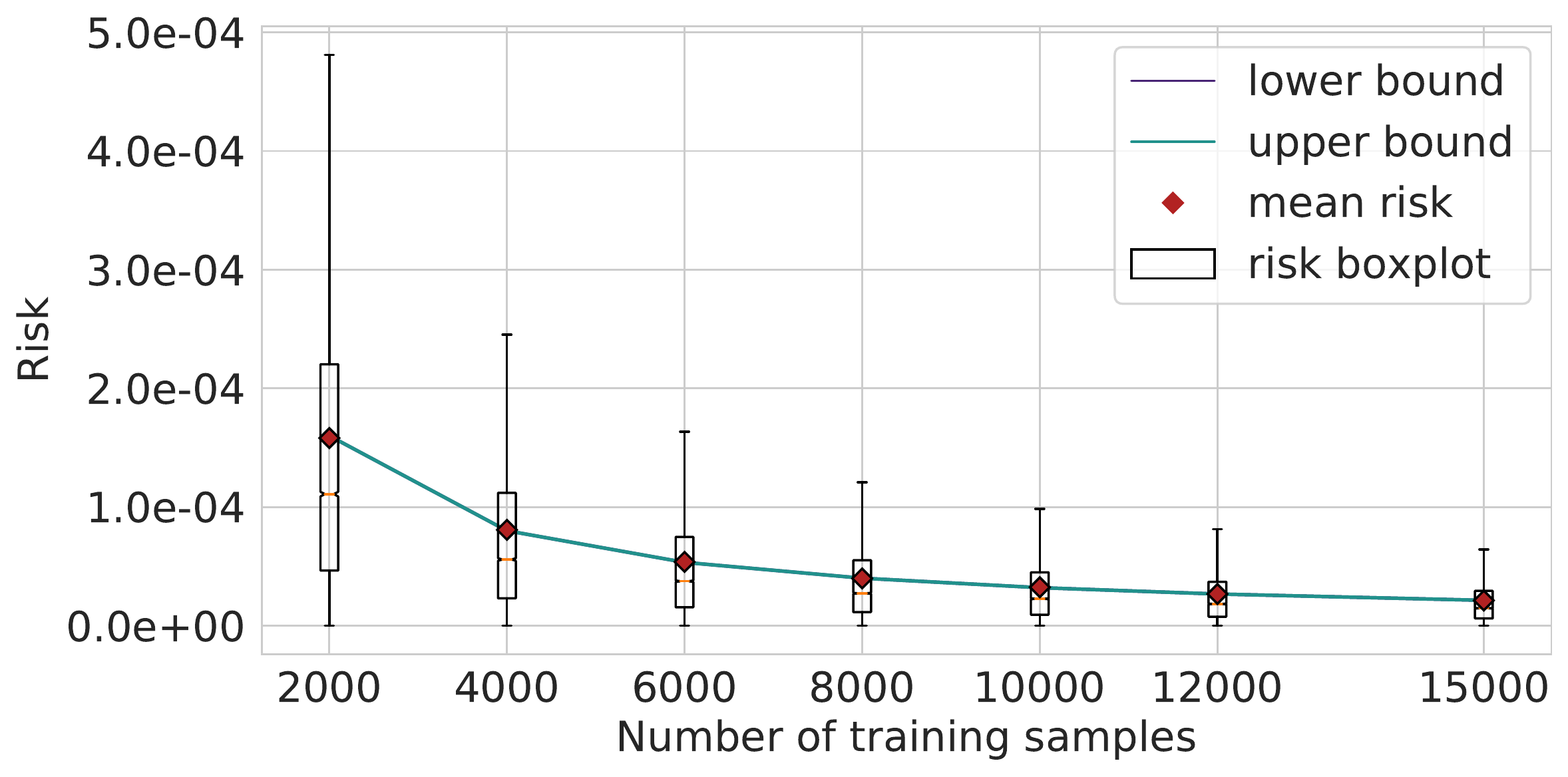}
    \caption{Illustration of the theoretical bounds for a shallow tree network of Proposition \ref{prop:new_risk_large_k} 2. for a chessboard with parameters $k^\star=4$, $N_{\mathcal{B}}=2^{k^\star-1}$ and $p =0.8$. The first-layer tree is of depth $k=4$. We draw a sample of size $n$ (x-axis), and a single tree $r_{k,0,n}$ is fitted for which the theoretical risk is evaluated. Each boxplot is built out of 20 000 repetitions. The outliers are not shown for the sake of presentation. Note that in such a case, the theoretical lower bound is constant and equal to the bias term. Note that the lower bound and the upper bound are merged.}
    \label{fig:etree_bound_perfect_depth}
\end{figure}

\begin{figure}[!ht]
    \centering
    \includegraphics[width = 0.7\textwidth]{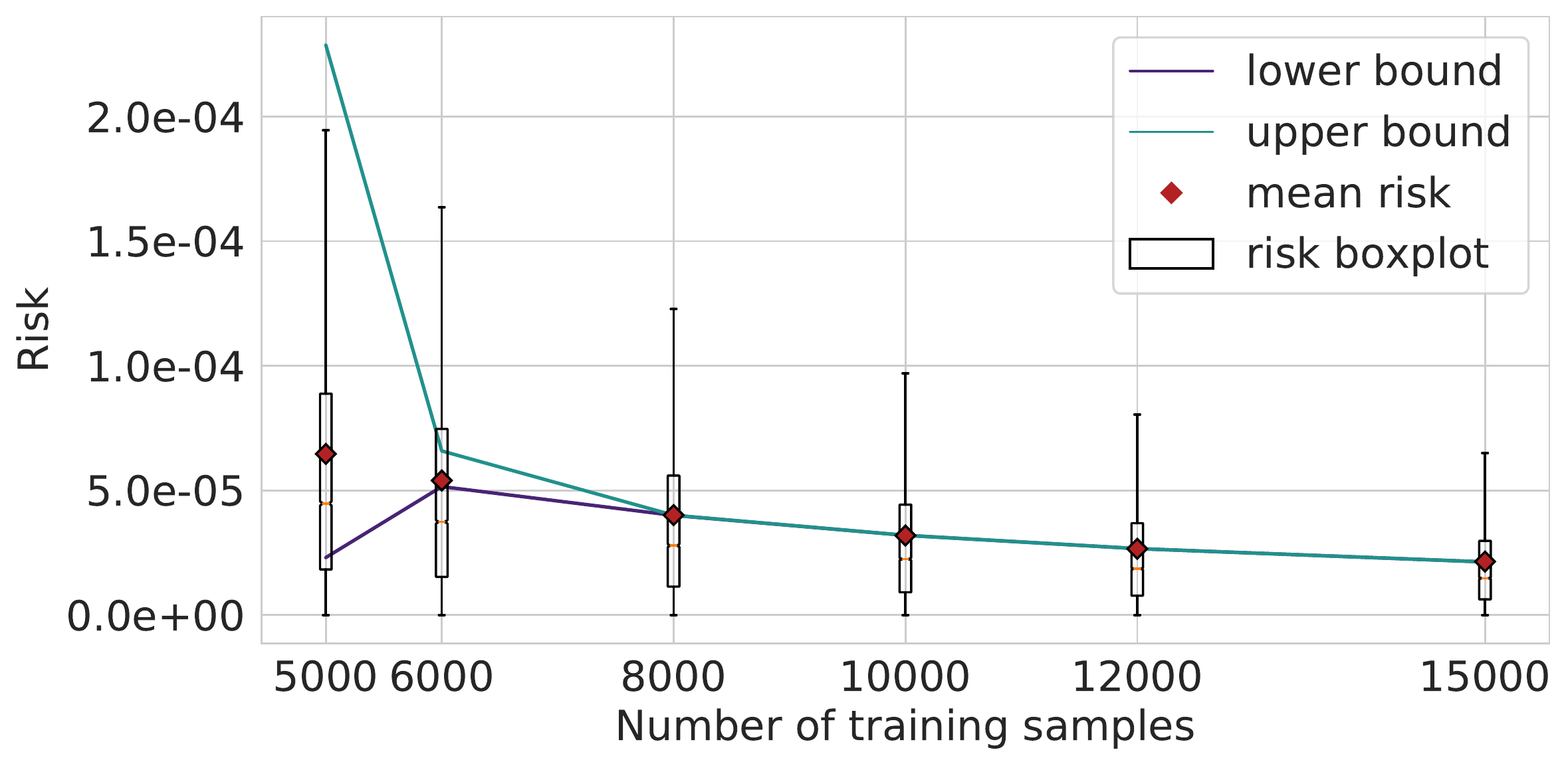}
    \caption{Illustration of the theoretical bounds for a shallow tree network of Proposition \ref{prop:new_risk_large_k} 2. for a chessboard with parameters $k^\star=4$, $N_{\mathcal{B}}=2^{k^\star-1}$ and $p =0.8$. The first-layer tree is of depth $k=6$. We draw a sample of size $n$ (x-axis), and a single tree $r_{k,0,n}$ is fitted for which the theoretical risk is evaluated. Each boxplot is built out of 20 000 repetitions. The outliers are not shown for the sake of presentation. Note that in such a case, the theoretical lower bound is constant and equal to the bias term. }
    \label{fig:etree_bound_big_depth_app}
\end{figure}

\clearpage

\section{Technical results on binomial random variables}

\begin{lemma} \label{lem:tech_res_binomial}
Let $Z$ be a binomial $\mathfrak{B}(n,p)$, $p \in (0,1], n>0$. Then, \\

\begin{enumerate}[(i)]
    \item $$\frac{1-(1-p)^n}{(n+1)p} \leq \Esp{ \frac{\ind{Z>0}}{Z}} \leq \frac{2}{(n+1)p} $$ 
    \item $$\Esp{ \frac{1}{1+Z}} \leq \frac{1}{(n+1)p} $$
    \item $$ \Esp{\frac{1}{1+Z^2}} \leq \frac{3}{(n+1)(n+2)p^2}$$ 
    \item $$\Esp{\frac{\ind{Z>0}}{\sqrt{Z}}} \leq \frac{2}{\sqrt{np}} $$
    \vspace{0.5cm}
    \item \label{lem:ahmed_formula} Let $k$ be an integer $\leq n$. Then, $$ \Esp{Z \mid Z \geq k } = np + (1-p)k \frac{\Prob{Z=k}}{\displaystyle \sum_{i=k}^n \Prob{Z=i}}.$$
    \item \label{lem:tech_res_bin_cond_exp} Let Z be a binomial $\mathfrak{B}(n,\frac{1}{2})$, $n>0$. Then, $$\Esp{Z \mid Z \leq \lfloor \frac{n+1}{2} \rfloor -1} \geq \frac{n}{2} - \left( \frac{\sqrt{n}}{\sqrt{\pi}} + \frac{2 \sqrt{2n}}{\pi \sqrt{2n+1}}  \right).$$
    
    \item \label{lem:tech_res_bin_cond_exp_b}Let Z be a binomial $\mathfrak{B}(n,\frac{1}{2})$, $n>0$. Then,
    $$\Esp{Z \mid Z \geq \lfloor \frac{n+1}{2} \rfloor} \leq \frac{n}{2} +1 + \frac{1}{\sqrt{\pi(n+1)}}.$$
\end{enumerate}

\end{lemma}

\bigskip

\begin{proof}
The reader may refer to the Lemma 11 of \cite{biau2012} to see the proof of (ii), (iii) and the right-hand side of (i). The left-hand side inequality of (i) can be found in the Section 1 of \cite{cribari2000note}. 



\underline{\textbf{(iv)}} The first two inequalities rely on simple analysis : 

\begin{align*}
     \Esp{\frac{\ind{Z>0}}{\sqrt{Z}}} &\leq \Esp{\frac{2}{1+\sqrt{Z}}} \\
     &\leq \Esp{\frac{2}{\sqrt{1+Z}}}.
\end{align*}

To go on, we adapt a transformation from Section 2 of \cite{cribari2000note} to our setting:
\begin{align*}
    \Esp{\frac{2}{\sqrt{1+Z}}} &= \frac{2}{\Gamma(1/2)} \int_0^\infty \frac{e^{-t}}{\sqrt{t}} \Esp{e^{-tZ}} dt  \\
    &=\frac{2}{\Gamma(1/2)} \int_0^\infty \frac{e^{-t}}{\sqrt{t}} (1-p+pe^{-t})^n dt \\
    &= \frac{2}{\Gamma(1/2)} \int_0^{-\log(1-p)} g(r) e^{-rn} dr,
\end{align*}
with $g(r) := p^{-1} e^{-r}\left(-\log(1 + \frac{1-e^{-r}}{p})\right)^{-1/2}$ after the change of variable $(1-p+pe^{-t})= e^{-r}$.

Let's prove that 
\begin{align}
    \label{eq:tech_lem_6}
    g(r) \leq \frac{1}{\sqrt{rp}} .
\end{align}
It holds that $\log(1+x) \leq \frac{2x}{2 +x}$ when $-1 < x \leq 0$, therefore
\begin{align*}
    g(r)^2 &=  p^{-2} e^{-2r}\left(-\log(1 + \frac{1-e^{-r}}{p})\right)^{-1} 
    \leq p^{-2} e^{-2r} \frac{2p + e^{-r} -1}{2(1-e^{-r})}.
\end{align*}
Furthermore,
\begin{align*}
    2p &\geq 2p \left( e^{-r} + re^{-2r} \right) \\
    &\geq 2p \left( e^{-r} + re^{-2r} \right) + r \left(e^{-3r} - e^{-2r} \right) \\
    &= r e^{-2r} (2p -1 + e^{-r}) + 2pe^{-r},
\end{align*}
and then dividing by $rp^2$, 
\begin{align*}
    \frac{2}{rp}(1 - e^{-r}) &\geq  \frac{1}{p^2} e^{-2r} (2p -1 + e^{-r}) 
    \qquad  \iff \qquad 
    \frac{1}{rp} \geq p^{-2} e^{-2r} \frac{2p + e^{-r} -1}{2(1-e^{-r})},
\end{align*}
which proves \eqref{eq:tech_lem_6}.

Equation \eqref{eq:tech_lem_6} leads to  
\begin{align}
    \Esp{\frac{2}{\sqrt{1+Z}}} \leq \frac{2}{\Gamma(1/2)} \int_0^{-\log(1-p)} \frac{1}{\sqrt{pr}} e^{-rn} dr .
\end{align}
Note that $\Gamma(1/2) = \sqrt{\pi}$.
After the change of variable $ u = \sqrt{rn}$, we obtain : 
\begin{align*}
    \Esp{\frac{2}{\sqrt{1+Z}}} &\leq \frac{4}{\sqrt{np \pi}} \int_0^{\sqrt{- n \log(1-p)}} e^{-u^2} du 
    \leq \frac{4}{\sqrt{np \pi}} \int_0^\infty e^{-u^2} du 
    \leq \frac{2}{\sqrt{np}}
\end{align*}
which ends the proof of (iv). 

\underline{\textbf{(v).(a)}}
We recall that $p=1/2$. An explicit computation of the expectation yields :

\begin{align*}
    \Esp{Z \mid Z < \lfloor \frac{n+1}{2} \rfloor} &= \frac{1}{\Prob{Z \leq \lfloor \frac{n+1}{2}\rfloor - 1 }} \displaystyle \sum_{i=1}^{\lfloor \frac{n+1}{2}\rfloor -1} \frac{i}{2^n} \binom{n}{i} \\
    &= \frac{2}{1} \frac{n}{2^n} \left( \frac{2^n}{2} - \frac{1}{2} \binom{n-1}{\frac{n-1}{2}}\right) \ind{n \% 2 = 1} \\
    &+ \frac{n}{\frac{1}{2}- \frac{1}{2}\Prob{Z = n/2}} \left(\displaystyle \sum_{i=1}^{n/2} i \binom{n}{i} - \frac{n}{2}\binom{n}{n/2} \right) \frac{\ind{n \% 2 = 0}}{2^n} \\
    &= n \left( \frac{1}{2} - \frac{1}{2^n} \binom{n-1}{\frac{n-1}{2}}\right) \ind{n \% 2 = 1} + \frac{n \cdot \ind{n \% 2 = 0}}{1- \Prob{Z = n/2}} \left( \frac{1}{2} - \frac{1}{2^n} \binom{n}{n/2} \right).
\end{align*}

We use that for all $m \in 2 \mathbb{N}^*$, 
\begin{align}
\label{ineq:binomial_coeff_pi}  
\binom{m}{m/2} \leq \frac{2^m}{\sqrt{\pi (m/2 + 1/4)}} 
\end{align}
and 
$$ \frac{1}{1-\Prob{Z = m/2}} \geq 1+ \frac{\sqrt{2}}{\sqrt{\pi n}}$$ where the last inequality can be obtained via a series expansion at $n=\infty$. Replacing the terms by their bounds, we have : 

\begin{align*}
    \Esp{Z \mid Z < \lfloor \frac{n+1}{2} \rfloor} &\geq n \left( \left( \frac{1}{2} - \frac{1}{\sqrt{\pi (2m -1)}}\right)\ind{n \% 2 =1}  + \left(1 + \frac{\sqrt{2}}{\sqrt{\pi n }}\right) \left(\frac{1}{2} - \frac{2}{\sqrt{\pi (2n+1)}} \right)\ind{n \% 2 =0}\right) \\
    &\geq n \left( \frac{1}{2} - \frac{1}{\sqrt{n \pi }} - \frac{2\sqrt{2}}{\pi \sqrt{n(2n+1)}}\right) \\
    &\geq \frac{n}{2} + \sqrt{n} \left( \frac{1}{\sqrt{\pi}} - \frac{2\sqrt{2}}\pi {\sqrt{(2n+1)}} \right) 
\end{align*}
which ends the proof of this item (v)(a). 

\underline{\textbf{(v).(b)}} We also begin with an explicit computation of the expectation :
\begin{align*}
    \Esp{Z \mid Z \geq \lfloor \frac{n+1}{2} \rfloor} &= \frac{1}{\Prob{Z \geq \lfloor \frac{n+1}{2}\rfloor}} \displaystyle \sum_{i=\lfloor \frac{n+1}{2}\rfloor}^{n} \frac{i}{2^n} \binom{n}{i} \\
    &= \frac{2}{1} \frac{1}{2^n} \left( 2^{n-2} +2^{n-1} + \frac{1}{2} \binom{n-1}{\frac{n-1}{2}}\right) \ind{n \% 2 = 1} \\
    &+ \frac{n}{\frac{1}{2}+ \frac{1}{2}\Prob{Z = n/2}} \left(\displaystyle \sum_{i=\lfloor \frac{n+1}{2}\rfloor}^{n} i \binom{n}{i} \right) \frac{\ind{n \% 2 = 0}}{2^n} \\
    &=  \left( \frac{n}{2} +1  + \frac{1}{2^n} \binom{n-1}{\frac{n-1}{2}}\right) \ind{n \% 2 = 1} + \frac{n \cdot \ind{n \% 2 = 0}}{1+ \Prob{Z = n/2}} \left( \frac{1}{2} + \frac{1}{2^n} \binom{n}{n/2} \right).
\end{align*}
The computation of the upper bound relies on the following inequalities : $\forall m \in 2\mathbb{N}^*, $
\begin{align}
    \binom{2m}{m} &\leq \frac{2^{2m}}{\sqrt{\pi(m + 1/4)}} \label{eq:bin_coef_inequality}
\end{align}
as well as $$ \frac{1}{1+\Prob{Z = n/2}} \leq 1 - \frac{\sqrt{2}}{\sqrt{\pi n }} + \frac{2}{\pi n } $$ where the last bound can be found via a series expansion at $n=\infty$. Replacing all terms by their bound and simplifying roughly gives the result.
\end{proof}

\begin{lemma}[Uniform Bernoulli labels: risk of a single tree] \label{lem:unif_tree}
Let $K$ be a compact in $\mathbb{R}^d, d \in \mathbb{N}$. Let $X, X_1, ..., X_n, n \in N^*$ be i.i.d.\ random variables uniformly distributed over K, $Y, Y_1, ..., Y_n$ i.i.d Bernoulli variables of parameter $p \in [0,1]$ which can be considered as the labels of $X, X_1,...,X_n$. We denote by $r_{0,k,n}, k \in \N^*$ a single tree of depth k. Then we have, for all $k \in N^*$, 

(i)
\begin{align}
    \mathbb{E} \left[ (r_{0,0,n}(X) -r(X))^2 \right] = \frac{p(1-p)}{n}
\end{align}

(ii)
\begin{align}
    2^{k} \cdot \frac{p(1-p)}{n} + \left(p^2 - \frac{2^k}{n} \right) (1-2^{-k})^n \leq \hspace{0.1cm} \mathbb{E} \left[ (r_{0,k,n}(X) -r(X))^2 \right] \hspace{0.1cm} \leq 2^{k+1} \cdot \frac{p(1-p)}{n} + p^2 (1-2^{-k})^n
\end{align}

\end{lemma}

\begin{proof}

(i) In the case $k=0$, $r_{0,0,n}$ simply computes the mean of all the $(Y_i)$'s over $K$:

\begin{align}
    \Esp{(r_{0,0,n}(X) - r(X))^2} &= \Esp{\left(\frac{1}{n} \displaystyle \sum_i Y_i -p\right)^2} \\
    &= \Esp{\frac{1}{n^2} \displaystyle \sum_i (Y_i -p)^2} \hspace{3cm} (Y_i \text{ independent}) \\
    &= \frac{p(1-p)}{n}.
\end{align}

(ii)
\begin{align}
    \Esp{(r_{0,k,n}(X) - r(X))^2} &= \Esp{\left( \frac{1}{\cardcellof{X}}  \displaystyle \sum_{X_i \in \cellof{X}} Y_i -p \right)^2 \ind{\cardcellof{X} >0}} + p^2 \Prob{\cardcellof{X} = 0} \nonumber \\
    &= \Esp{\frac{\ind{\cardcellof{X} >0}}{\cardcellof{X}^2}  \displaystyle \sum_{X_i \in \cellof{X}} (Y_i -p)^2 } + p^2 \Prob{\cardcellof{X} = 0} \\
    &= p(1-p) \Esp{\frac{\ind{\cardcellof{X} >0}}{\cardcellof{X}}} + p^2 (1-2^{-k})^n
\end{align}

Noticing that $\cardcellof{X}$ is a binomial $\mathfrak{B}(n, \frac{1}{2^k})$, we obtain the upper bound using Lemma \ref{lem:tech_res_binomial} (i) :

\begin{align}
    \Esp{\frac{\ind{\cardcellof{X} >0}}{\cardcellof{X}}} &\leq 2 \cdot  \frac{2^k}{n}
\end{align}

the lower bound is immediately obtained by applying Lemma \ref{lem:tech_res_binomial}, (i):

\begin{align}
    \Esp{\frac{\ind{\cardcellof{X} >0}}{\cardcellof{X}}} &\geq  \frac{2^k}{n} \left(1-(1-2^{-k})^n \right) 
\end{align}
\end{proof}

\section{Proof of Lemma \ref{lem:counter_example}}
\label{app:proof_counter_example}

Note that
\begin{align}
    \mathbb{E} [(\hat{r}_{k,1,\infty}(X) - r(X))^2] & = 
    \mathbb{E} [(\hat{r}_{k,1,\infty}(X) - r(X))^2 \mathbb{1}_{X \in \mathcal{B}}]\\
    & + \mathbb{E} [(\hat{r}_{k,1,\infty}(X) - r(X))^2\mathbb{1}_{X \in \mathcal{W}}].
    \label{eq:proof_ce_1}
\end{align}
Now, we analyze the first term in Equation~\eqref{eq:proof_ce_1}. We have
\begin{align}
\mathbb{E} [(\hat{r}_{k,1,\infty}(X) - r(X))^2 \mathbb{1}_{X \in \mathcal{B}}]  & = \mathbb{E} [\mathbb{E}  [(\hat{r}_{k,1,\infty}(X) - r(X))^2 \mathbb{1}_{X \in \mathcal{B}} |N_{\mathcal{B}}]]\\
&= \sum_{i,j} \mathbb{E} [\mathbb{E}  [\Big(\frac{1}{N_{\mathcal{B}}} \sum_{i',j' \in \mathcal{B}} p_{i'j'} - p_{ij}\Big)^2 \mathbb{1}_{X \in C_{ij} \cap \mathcal{B}} |N_{\mathcal{B}}]] \nonumber\\
&= \sum_{i,j} \Esp{\mathbb{E} [\mathbb{1}_{X \in C_{ij}} \mathbb{1}_{X \in \mathcal{B}} \mathbb{1}_{N_{\mathcal{B}} > 0} \Big(\frac{1}{N_{\mathcal{B}}} \sum_{i',j' \in \mathcal{B}} p_{i'j'} - p_{ij}\Big)^2 | N_{\mathcal{B}} ]} \nonumber \\
&+ \sum_{i,j} \Esp{\Esp{\ind{X \in C_{i,j} \cap \mathcal{B}} \ind{N_\mathcal{B} = 0} p_{i,j}^2 \mid N_{\mathcal{B}}}} \label{eq:lem_decompo_empty}.
\end{align}
We begin with the second term in Equation \eqref{eq:lem_decompo_empty}. We have, for all $i,j$,
\begin{align}
    \Esp{\Esp{\ind{X \in C_{i,j} \cap \mathcal{B}} \ind{N_\mathcal{B} = 0} p_{i,j}^2 \mid N_{\mathcal{B}}}} &= \Esp{\ind{X \in C_{i,j}}\ind{N_\mathcal{B} = 0 } \Esp{ p_{i,j}^2 \ind{X \in \mathcal{B}}\mid X, N_\mathcal{B}}} \\
    &= \Esp{\ind{X \in C_{i,j}}\ind{N_\mathcal{B} = 0 } \Esp{ p_{i,j}^2 \ind{p_{i,j} \geq \frac{1}{2}}}}.
\end{align}
As $p_{i,j}$ is drawn uniformly in $[0,1]$,
\begin{align*}
    \Esp{ p_{i,j}^2 \ind{p_{i,j} \geq \frac{1}{2}}} &= \Esp{p_{i,j}^2 \mid p_{i,j} \geq \frac{1}{2}} \Prob{p_{i,j} \geq \frac{1}{2}} \\
    &= \frac{7}{24}.
\end{align*}
Therefore,
\begin{align}
    \Esp{\Esp{\ind{X \in C_{i,j} \cap \mathcal{B}} \ind{N_\mathcal{B} = 0} p_{i,j}^2 \mid N_{\mathcal{B}}}} &= \frac{7}{24} \Prob{X \in C_{i,j}} \Prob{N_\mathcal{B} = 0} \\
    &= \frac{1}{2^{2^{k^\star}}} \frac{7}{24}.
\end{align}
Regarding the first term of Equation \eqref{eq:lem_decompo_empty}, 
\begin{align}
    &\mathbb{E} [\mathbb{1}_{X \in C_{ij}} \mathbb{1}_{X \in \mathcal{B}} \mathbb{1}_{N_{\mathcal{B}} > 0} \Big(\frac{1}{N_{\mathcal{B}}} \sum_{i',j' \in \mathcal{B}} p_{i'j'} - p_{ij}\Big)^2 | N_{\mathcal{B}} ]  \nonumber \\
    &= \Esp{\mathbb{1}_{X \in C_{ij}} \Esp{\mathbb{1}_{X \in \mathcal{B}} \mathbb{1}_{N_{\mathcal{B}} > 0} \left(\frac{1}{N_\mathcal{B}} \sum_{i',j'}(p_{i',j'}-p_{i,j}) \right)^2 \mid X, N_\mathcal{B}}} \\
    &= \Esp{\mathbb{1}_{X \in C_{ij}}} \Esp{ \ind{p_{i,j} \geq \frac{1}{2}} \mathbb{1}_{N_{\mathcal{B}} > 0} \left(\frac{1}{N_\mathcal{B}} \sum_{i',j' \in \mathcal{B}}(p_{i',j'}-p_{i,j}) \right)^2 \mid N_{\mathcal{B}}} \\
    &= \Esp{\mathbb{1}_{X \in C_{ij}}}  \Esp{\frac{1}{N_{\mathcal{B}}^2} \left( \sum_{ \substack{i',j',i'',j'' \in \mathcal{B} \\ (i',j')\neq (i'',j'') }}(p_{i'j'} - p_{ij})(p_{i''j''} - p{ij}) + \sum_{ \substack{i',j' \in \mathcal{B}, \\(i',j' \neq (i,j))}} (p_{i'j'} -p_{ij})^2  \right) \mid N_{\mathcal{B}}, N_\mathcal{B} >0, p_{i,j} \geq \frac{1}{2}} \nonumber\\
    &\cdot \Prob{p_{i,j} \geq \frac{1}{2} \cap N_\mathcal{B} >0}.
\end{align}
Recall that $p_{ij}$ is drawn uniformly over $[0,1]$. Therefore, $\Prob{p_{i,j} \geq \frac{1}{2} \cap N_\mathcal{B} >0}  = \Prob{p_{i,j} \geq \frac{1}{2}} = \frac{1}{2}$. Thus,
\begin{align*}
    &\mathbb{E} [\mathbb{1}_{X \in C_{ij}} \mathbb{1}_{X \in \mathcal{B}} \mathbb{1}_{N_{\mathcal{B}} > 0} \Big(\frac{1}{N_{\mathcal{B}}} \sum_{i',j' \in \mathcal{B}} p_{i'j'} - p_{ij}\Big)^2 | N_{\mathcal{B}} ] \nonumber \\
    &= \frac{1}{\expkstar{}} \Esp{\frac{1}{N_\mathcal{B}^2} (N_\mathcal{B}-1)(N_\mathcal{B}-2) \text{Var}(p_{i,j} \mid p_{i,j} \geq \frac{1}{2}) + \sum_{ \substack{i',j' \in \mathcal{B}, \\(i',j' \neq (i,j))}} 2 \text{Var}(p_{i,j} \mid p_{i,j} \geq \frac{1}{2})  \mid N_{\mathcal{B}}, N_\mathcal{B} >0, p_{i,j} \geq \frac{1}{2}} \\
    &= \frac{1}{\expkstar{}} \Esp{ \frac{1}{N_\mathcal{B}^2} \left( (N_\mathcal{B}-1)(N_\mathcal{B}-2) \frac{1}{48} + 2 \frac{1}{48} (N_\mathcal{B} - 1) \right) \mid N_{\mathcal{B}}, N_\mathcal{B} >0}\Prob{p_{i,j} \geq \frac{1}{2} \cap N_\mathcal{B} >0} \\
    &= \frac{1}{\expkstar{+1}} \Esp{\frac{1}{48 N_\mathcal{B}^2} \left(N_\mathcal{B}^2 - N_\mathcal{B} \right) \mid N_{\mathcal{B}}, N_\mathcal{B} >0}\\
    &= \frac{1}{48 \cdot \expkstar{+1}} \left(1 - \Esp{\frac{1}{N_\mathcal{B}} \mid N_\mathcal{B}, N_\mathcal{B} > 0 } \right).
\end{align*}
We now have:
\begin{align*}
    \Esp{\mathbb{E} [\mathbb{1}_{X \in C_{ij}} \mathbb{1}_{X \in \mathcal{B}} \mathbb{1}_{N_{\mathcal{B}} > 0} \Big(\frac{1}{N_{\mathcal{B}}} \sum_{i',j' \in \mathcal{B}} p_{i'j'} - p_{ij}\Big)^2 | N_{\mathcal{B}} ]} &= \Esp{\frac{1}{48 \cdot \expkstar{+1}} \left(1 - \Esp{\frac{1}{N_\mathcal{B}} \mid N_\mathcal{B}, N_\mathcal{B} > 0 } \right)} \\
    &= \frac{1}{48 \cdot \expkstar{+1}} \left(1 - \Esp{\frac{1}{N_\mathcal{B}} \mid N_\mathcal{B} > 0 } \right).
\end{align*}
Notice that $N_\mathcal{B}$ is a binomial variable of parameters $\expkstar{},1/2$. Thus we can apply Lemma \ref{lem:tech_res_binomial} to deduce
\begin{align*}
    \Esp{\frac{1}{N_\mathcal{B}} \mid N_\mathcal{B} > 0 } &= \Esp{\frac{\ind{Z>0}}{Z}} \frac{1}{\Prob{Z>0}} \\
    &\leq \frac{4}{\expkstar{}+1} \frac{1}{\Prob{Z>0}}
\end{align*}
Moreover, as $\Prob{Z>0} \geq \frac{1}{2}$, we have:
\begin{align*}
    \Esp{\mathbb{E} [\mathbb{1}_{X \in C_{ij}} \mathbb{1}_{X \in \mathcal{B}} \mathbb{1}_{N_{\mathcal{B}} > 0} \Big(\frac{1}{N_{\mathcal{B}}} \sum_{i',j' \in \mathcal{B}} p_{i'j'} - p_{ij}\Big)^2 | N_{\mathcal{B}} ]} &\geq  \frac{1}{48 \cdot \expkstar{+1}} \left(1 - \frac{8}{\expkstar{}+1}\right)
\end{align*}
In the end, the first term of Equation \eqref{eq:proof_ce_1} verifies
\begin{align*}
    \mathbb{E} [(\hat{r}_{k,1,\infty}(X) - r(X))^2 \mathbb{1}_{X \in \mathcal{B}}] \geq \frac{1}{2} \Big( \frac{1}{48} \Big(  1 - \frac{8}{2^{k^{\star}}-1}\Big)
    \Big) + \frac{1}{2^{2^{k^{\star}}}} \frac{7}{24}.
\end{align*}
Similar computations show that the second term of Equation \eqref{proof_ce_1} verifies:
\begin{align}
    \mathbb{E} [(\hat{r}_{k,1,\infty}(X) - r(X))^2\mathbb{1}_{X \in \mathcal{W}}].
    \label{proof_ce_1} &\geq \frac{1}{2} \Big( \frac{1}{48} \Big(  1 - \frac{8}{2^{k^{\star}}-1}\Big)
    \Big) + \Esp{\Esp{\ind{X \in C_{ij \cap \mathcal{W}} } \ind{N_\mathcal{W} = 0}p_{ij}^2 \mid N_\mathcal{W}}} \\
    &\geq \frac{1}{2} \Big( \frac{1}{48} \Big(  1 - \frac{8}{2^{k^{\star}}-1}\Big)\Big)  + \frac{1}{\expkstar{}} \frac{1}{2^{2^{k^\star}}} \Esp{p_{ij}^2 \mid p_{ij} < \frac{1}{2}} \\
    &\geq  \frac{1}{2} \Big( \frac{1}{48} \Big(  1 - \frac{8}{2^{k^{\star}}-1}\Big)\Big) + \frac{1}{2^{2^{k^{\star}}}} \frac{1}{12}.
\end{align}
All in all, we have
\begin{align*}
    \mathbb{E} [(\hat{r}_{k,1,\infty}(X) - r(X))^2] \geq  \frac{1}{48} \Big(  1 - \frac{8}{2^{k^{\star}}-1}\Big)
     + \frac{1}{2^{2^{k^{\star}}}} \frac{9}{24}.
\end{align*}

\section{Proof of Lemma \ref{lem:depth1}}
\label{app:proof_lemma_depth1}

 
First, note that since we are in an infinite sample regime, the risk of our estimators is equal to their bias term. We can thus work with the true distribution instead of a finite data set.
\begin{enumerate}[(i)]

\item The risk of a second-layer tree cutting $k'$ times, $k'\geq k^\star$ along the raw features equals $0$ (thus being minimal) as each leaf is included in a cell. We now exhibit one configuration for which any second-layer tree of depth $k'<k^\star$ is biased. We consider the balanced chessboard with parameters $k^\star$, $N_{\mathcal{B}}=2^{k^\star-1}$ and $p$, defined in Proposition \ref{prop:balanced_chessboard} and shown in Figure \ref{fig:Damier}. For all $k<k^\star$, each leaf of the first tree contains exactly half black and half white cells, thus predicting $1/2$ and having a risk of $(p-\frac{1}{2})^2$. Therefore a second-layer tree building on raw features only would predict $1/2$ everywhere and would also be biased. If the second-layer tree performs a cut on the new feature provided by the first-layer tree, it creates two leaves: all the leaves where the prediction of the first tree is greater than or equal to $1/2$ are gathered in the right leaf, all the other leaves are gathered in the left leaf. The left leaf is empty and the prediction of the second-layer tree is also $1/2$ everywhere. Any new cut along the new feature would create one leaf predicting $1/2$ on $[0,1]^2$ and other leaves being empty. In any case, the second-layer tree is biased. Thus the minimal risk for all configurations is obtained by a second-layer tree of depth $k'\geq k^\star$ which cuts along the raw features only.

\item When $k \geq k^{\star}$,  the first tree is unbiased since each of its leaves is included in only one chessboard data cell. Splitting on the new feature in the second-layer tree induces a separation between cells for which $\mathbb{P}[Y=1| X \in C] = p$ and cells for which $\mathbb{P}[Y=1| X \in C] = 1-p$ since $p \neq 1/2$. Taking the expectation of $Y$ on these two regions leads to a shallow tree network of risk zero. 
\end{enumerate}

\section{Proof of Proposition \ref{prop:balanced_chessboard}}
\subsection{Proof of statement 1.: risk of a single tree}

Recall that if a cell is empty, the tree prediction in this cell is set (arbitrarily) to zero. Thus, 
\begin{align}
    \nonumber
    & \Esp{(r_{k,0,n}(X) - r(X))^2} \\
    &= \Esp{(r_{k,0,n}(X) - r(X))^2 \ind{\cardcellof{X} >0}} + \Esp{(r(X))^2 \ind{\cardcellof{X} =0}}, \\
    &= \Esp{\left( \frac{1}{\cardcellof{X}}  \displaystyle \sum_{X_i \in \cellof{X}} Y_i -r(X) \right)^2 \ind{\cardcellof{X} >0} } +\Esp{(r(X))^2 \ind{\cardcellof{X} =0}},
    \label{eq_proof1_1}
\end{align}
where
\begin{align}
\Esp{(r(X))^2 \ind{\cardcellof{X} =0}} & = \Esp{(r(X))^2 \ind{\cardcellof{X} =0} \ind{X \in \mathcal{B}}} + \Esp{(r(X))^2 \ind{\cardcellof{X} =0} \ind{X \in \mathcal{W}}} \\
& = \left( \frac{p^2}{2} + \frac{(1-p)^2}{2} \right) \Prob{\cardcellof{X} = 0}\\
& = (p^2 + (1-p)^2) \frac{(1-2^{-k})^n}{2}.
\end{align}

We now study the first term in \eqref{eq_proof1_1}, by considering that $X$ falls into $\mathcal{B}$ (the same computation holds when $X$ falls into $\mathcal{W}$). 
Letting $(X', Y')$ a generic random variable with the same distribution as $(X,Y)$, one has

\begin{align}
    &\Esp{\left( \frac{1}{\cardcellof{X}}  \displaystyle  \sum_{X_i \in \cellof{X}} Y_i -p \right)^2 \ind{\cardcellof{X} >0} \ind{X \in \mathcal{B}} } \\
    &= \frac{1}{2}\Esp{\left( \frac{1}{\cardcellof{X}}  \displaystyle \sum_{X_i \in \cellof{X}} \left(Y_i - \Esp{Y' | X' \in \cellof{X}}\right) \right)^2  \ind{\cardcellof{X} >0}} \\
    \notag
    &\qquad \qquad +  \Esp{\left( \Esp{Y' | X' \in \cellof{X}} -p \right)^2 \ind{X \in \mathcal{B}} \ind{\cardcellof{X} >0}} \\
    \label{eq:risk_singletree_B}
    \notag
    &= \frac{1}{2} \Esp{ \frac{\ind{\cardcellof{X} >0}}{\cardcellof{X}^2} \Esp{\left( \displaystyle \sum_{X_i \in \cellof{X}} \left( Y_i - \Esp{Y' | X' \in \cellof{X}} \right) \right)^2 \left\vert\right. \cardcellof{X}}} \\
    & \qquad \qquad + \frac{1}{2}\left(p-\frac{1}{2}\right)^2  \Prob{\cardcellof{X} >0},
\end{align}
where we used the fact that $\Esp{Y' | X' \in \cellof{X}} = 1/2$ as in any leaf there is the same number of black and white cells. Moreover, conditional to $\cardcellof{X}$,  $  \sum_{X_i \in \cellof{X}} Y_i $ is a binomial random variable with parameters $\mathfrak{B}(\cardcellof{X}, \frac{1}{2})$. 
Hence we obtain : 

\begin{align}
    &\Esp{ \frac{\ind{\cardcellof{X} >0}}{\cardcellof{X}^2} \Esp{\left( \displaystyle \sum_{X_i \in \cellof{X}} \left( Y_i - \Esp{Y' | X' \in \cellof{X}} \right) \right)^2 | \cardcellof{X}}} \\ 
    &= \frac{1}{4} \Esp{\frac{\ind{\cardcellof{X} >0}}{\cardcellof{X}}} .
\end{align}

The same computation holds when $X$ falls into $\mathcal{W}$. Indeed, the left-hand side term in \eqref{eq:risk_singletree_B} is unchanged,  as for the right-hand side term, note that $(\frac{1}{2}-p)^2 = (\frac{1}{2} -(1-p))^2$. Consequently, 
\begin{align}
&\Esp{\left( \frac{1}{\cardcellof{X}}  \displaystyle \sum_{X_i \in \cellof{X}} Y_i -r(X) \right)^2 \ind{\cardcellof{X} >0} }\\ &=  \frac{1}{4}\Esp{\frac{\ind{\cardcellof{X} >0}}{\cardcellof{X}}} + \left(p-\frac{1}{2}\right)^2 (1-(1-2^{-k})^n).  
\label{eq4_proof_prop1}
\end{align}
Injecting \eqref{eq4_proof_prop1} into \eqref{eq_proof1_1}, we have
\begin{align}
    & \Esp{(r_{k,0,n}(X) - r(X))^2} \\
    &= \frac{1}{4}\Esp{\frac{\ind{\cardcellof{X} >0}}{\cardcellof{X}}} + \left(p-\frac{1}{2}\right)^2 (1-(1-2^{-k})^n) + (p^2 + (1-p)^2) \frac{(1-2^{-k})^n}{2}\\
    & = \frac{1}{4}\Esp{\frac{\ind{\cardcellof{X} >0}}{\cardcellof{X}}} + \left(p-\frac{1}{2}\right)^2  + \left(p^2 + (1-p)^2 - 2 \left(p - \frac{1}{2} \right)^2 \right) \frac{(1-2^{-k})^n}{2}\\
    & = \frac{1}{4}\Esp{\frac{\ind{\cardcellof{X} >0}}{\cardcellof{X}}} + \left(p-\frac{1}{2}\right)^2  +  \frac{(1-2^{-k})^n}{4}.
\end{align}
Noticing that $\cardcellof{X}$ is a binomial random variable $\mathfrak{B}(n, \frac{1}{2^k})$, we obtain the upper and lower bounds with Lemma \ref{lem:tech_res_binomial} (i): 
\begin{align}
    \Esp{\frac{\ind{\cardcellof{X} >0}}{\cardcellof{X}}} & \leq  \frac{2^{k+1}}{n+1},
\end{align}

and, 

\begin{align}
    \Esp{\frac{\ind{\cardcellof{X} >0}}{\cardcellof{X}}} \geq \left(1-(1-2^{-k})^n \right)  \frac{2^k}{n+1}.
\end{align}

Gathering all the terms gives the result, $$ \Esp{(r_{k,0,n}(X) - r(X))^2}\leq \left(p-\frac{1}{2}\right)^2  + \frac{2^k}{2(n+1)} + \frac{(1-2^{-k})^n}{4} $$
and 
$$ \Esp{(r_{k,0,n}(X) - r(X))^2} \geq \left(p-\frac{1}{2}\right)^2  + \frac{2^k}{4(n+1)} +   \frac{(1-2^{-k})^n}{4} \left( 1 - \frac{2^k}{n+1} \right).
$$

\subsection{Proof of statement 2.: risk of a shallow tree network}

Let $k \in \mathbb{N}$. Denote by $\mathcal{L}_{k} = \{L_{i,k}, i  = 1, \hdots, 2^k \} $ the set of all leaves of the encoding tree (of depth $k$). We let  $\mathcal{L}_{\tilde{\mathcal{B}}_k}$ be the set of all cells of the encoding tree containing at least one observation, and such that the empirical probability of $Y$ being equal to one in the cell is larger than $1/2$, i.e. 
\begin{align*}
\tilde{\mathcal{B}}_k = \cup_{L \in \mathcal{L}_{\tilde{\mathcal{B}}_k}} \{x, x \in L\} 
\end{align*}
\begin{align*}
\mathcal{L}_{\tilde{\mathcal{B}}_k} = \{ L \in \mathcal{L}_{k}, N_n(L) >0, \frac{1}{N_n(L)} \displaystyle \sum_{X_i \in L} Y_i \geq \frac{1}{2} \}. 
\end{align*}
Accordingly, we let the part of the input space corresponding to $\mathcal{L}_{\tilde{\mathcal{B}}_k}$ as 
\begin{align*}
\tilde{\mathcal{B}}_k = \cup_{L \in \mathcal{L}_{\tilde{\mathcal{B}}_k}} \{x, x \in L\} 
\end{align*}
Similarly, 
\begin{align*}
\mathcal{L}_{\tilde{\mathcal{W}}_k} = \{ L \in \mathcal{L}_{k}, N_n(L) >0, \frac{1}{N_n(L)} \displaystyle \sum_{X_i \in L} Y_i < \frac{1}{2} \}. 
\end{align*}
and
\begin{align*}
\tilde{\mathcal{W}}_k = \cup_{L \in \mathcal{L}_{\tilde{\mathcal{W}}_k}} \{x, x \in L\} 
\end{align*}

\subsubsection{\textbf{Proof of 2. (upper-bound)}}

Recall that $k<k^\star$. In this case, each leaf of the encoding tree is contains half black square and half white square (see Figure \ref{fig:damier_small_depth}). Hence, the empirical probability of $Y$ being equal to one in such leaf is close to $1/2$. Recalling that our estimate is $r_{k,1,n}$, we have
\begin{align}
    \notag 
    &\Esp{(r_{k,1,n}(X) - r(X))^2 } \\
    \notag 
    &= \Esp{ (r_{k,1,n}(X) - p)^2 \ind{X \in \mathcal{B}} \ind{X \in \tilde{\mathcal{B}}_k}  } + \Esp{ (r_{k,1,n}(X) - p)^2 \ind{X \in \mathcal{B}} \ind{X \in \tilde{\mathcal{W}}_k} } \\
    \label{eq:risk_decompo_bw}
    &+ \Esp{ (r_{k,1,n}(X) - (1-p))^2 \ind{X \in \mathcal{W}} \ind{X \in \tilde{\mathcal{B}}_k}} + \Esp{ (r_{k,1,n}(X) - (1-p))^2 \ind{X \in \mathcal{W}} \ind{X \in \tilde{\mathcal{W}}_k}}\\
    \notag
    & + \Esp{ (r_{k,1,n}(X) - p)^2 \ind{X \in \mathcal{B}} (1 - \ind{X \in \tilde{\mathcal{B}}_k} - \ind{X \in \tilde{\mathcal{W}}_k}) } 
    + \Esp{ (r_{k,1,n}(X) - (1-p))^2 \ind{X \in \mathcal{W}} (1 - \ind{X \in \tilde{\mathcal{B}}_k} - \ind{X \in \tilde{\mathcal{W}}_k})}
\end{align}
Note that $X \notin \tilde{\mathcal{B}}_k \cup \tilde{\mathcal{W}}_k$ is equivalent to $X$ belonging to an empty cell. Besides, the prediction is null by convention in an empty cell.  Therefore, the sum of the last two terms in \eqref{eq:risk_decompo_bw} can be written as
\begin{align}
& \Esp{ p^2 \ind{X \in \mathcal{B}} \ind{N_n(C_n(X))=0}) } 
    + \Esp{  (1-p)^2 \ind{X \in \mathcal{W}} \ind{N_n(C_n(X))=0}) }
 = \frac{p^2+(1-p)^2}{2} \left(1-\frac{1}{2^k}\right)^n. \label{eq:empty_cell_risk}
\end{align}

To begin with we focus on the first two terms in \eqref{eq:risk_decompo_bw}. We deal with the last two terms at the very end as similar computations are conducted.

\begin{align}
    \notag
    & \Esp{ (r_{k,1,n}(X) - p)^2 \ind{X \in \mathcal{B}} \ind{X \in \tilde{\mathcal{B}}_k} } + \Esp{ (r_{k,1,n}(X) - p)^2 \ind{X \in \mathcal{B}} \ind{X \in \tilde{\mathcal{W}}_k} } \\
    \notag
    &  = \Esp{ \Esp{\left(\frac{1}{N_n(\tilde{\mathcal{B}}_k)} \displaystyle \sum_{X_i \in \tilde{\mathcal{B}}_k} Y_i -p\right)^2 \Big| \tilde{\mathcal{B}}_k }   \Prob{X \in \tilde{\mathcal{B}}_k, X \in \mathcal{B} | \tilde{\mathcal{B}}_k } } \nonumber 
    \\
    & \qquad + \Esp{ \Esp{ \left(\frac{1}{N_n(\tilde{\mathcal{W}}_k)} \displaystyle \sum_{X_i \in \tilde{\mathcal{W}}_k} Y_i -p\right)^2 \Bigg| \tilde{\mathcal{W}}_k } \Prob{X \in \tilde{\mathcal{W}}_k, X \in \mathcal{B} | \tilde{\mathcal{W}}_k}}.
    \label{eq:Prob21aRisk}
\end{align}
Regarding the left-hand side term in \eqref{eq:Prob21aRisk}, 
\begin{align}
    \label{eq:WtermProp41}
    \Esp{\left(\frac{1}{N_n(\tilde{\mathcal{B}}_k)} \displaystyle \sum_{X_i \in \tilde{\mathcal{B}}_k} Y_i -p\right)^2 \Big| \tilde{\mathcal{B}}_k } &\leq \left(p-\frac{1}{2}\right)^2,
\end{align}
since $p>1/2$ and, by definition of  $\tilde{\mathcal{B}}_k$, 
$$ \sum_{X_i \in \tilde{\mathcal{B}}_k } Y_i \geq N_n(\tilde{\mathcal{B}}_k)/2.$$

Now, regarding right-hand side term in \eqref{eq:Prob21aRisk}, we let 
$$ Z_{\tilde{\mathcal{W}}_k}= \Esp{ \sum_{X_i \in \tilde{\mathcal{W}}_k} Y_i \mid N_1,...,N_{2^{k}}, \tilde{\mathcal{W}}_k},$$ where $N_1,...,N_{2^{k}}$ denote the number of data points falling in each leaf $L_1, \hdots, L_{2^k}$ of the encoding tree. Hence,
\begin{align}
    \nonumber
    \Esp{\left( \frac{1}{N_n(\tilde{\mathcal{W}}_k)} \displaystyle \sum_{X_i \in \tilde{\mathcal{W}}_k} Y_i -p\right)^2 \Big| \tilde{\mathcal{W}}_k} 
    &=  \mathbb{E} \left[  \frac{1}{N_n(\tilde{\mathcal{W}}_k)^2} \mathbb{E}  \left[ \left(\displaystyle \sum_{X_i \in \tilde{\mathcal{W}}_k} Y_i - Z_{\tilde{\mathcal{W}}_k}\right)^2 + \left(  Z_{\tilde{\mathcal{W}}_k} - N_n(\tilde{\mathcal{W}}_k)p\right)^2 \right. \right.  \\
    & \left. \left.+ 2 \left (\displaystyle \sum_{X_i \in \tilde{\mathcal{W}}_k} Y_i - Z_{\tilde{\mathcal{W}}_k}\right)  \left(  Z_{\tilde{\mathcal{W}}_k} - N_n(\tilde{\mathcal{W}}_k)p \right) \mid N_1,...,N_{2^{k}}, \tilde{\mathcal{W}}_k \right] \Big| \tilde{\mathcal{W}}_k \right]
\end{align}
The cross-term is null according to the definition of $Z_{\tilde{\mathcal{W}}_k}$, and since $(Z_{\tilde{\mathcal{W}}_k} -N_n(\tilde{\mathcal{W}}_k))$ is $(N_1,...,N_{2^{k}}, \tilde{\mathcal{W}}_k)$-measurable. Therefore, 
\begin{align}
    \nonumber
    \Esp{\left( \frac{1}{N_n(\tilde{\mathcal{W}}_k)} \displaystyle \sum_{X_i \in \tilde{\mathcal{W}}_k} Y_i -p\right)^2 \Big| \tilde{\mathcal{W}}_k} 
    &=   \mathbb{E} \left[ \frac{1}{N_n(\tilde{\mathcal{W}}_k)^2} \mathbb{E}  \left[ \left(\displaystyle \sum_{X_i \in \tilde{\mathcal{W}}_k} Y_i - Z_{\tilde{\mathcal{W}}_k}\right)^2 \mid N_1,...,N_{2^{k}}, \tilde{\mathcal{W}}_k \right]  \Big| \tilde{\mathcal{W}}_k \right] 
    \\
    & \quad +  \mathbb{E} \left[ \frac{1}{N_n(\tilde{\mathcal{W}}_k)^2} \mathbb{E}  \left[  \left(  Z_{\tilde{\mathcal{W}}_k} - N_n(\tilde{\mathcal{W}}_k)p\right)^2
    \mid N_1,...,N_{2^{k}}, \tilde{\mathcal{W}}_k \right]  \Big| \tilde{\mathcal{W}}_k \right] \nonumber \\
    &= I_n + J_n,
    \label{eq_lower_bound_Jn}
\end{align}
where $I_n$ and $J_n$ can be respectively identified as variance and bias terms. Indeed, 
\begin{align*}
\mathbb{E}  \left[ \left(\displaystyle \sum_{X_i \in \tilde{\mathcal{W}}_k} Y_i - Z_{\tilde{\mathcal{W}}_k}\right)^2 \mid N_1,...,N_{2^{k}}, \tilde{\mathcal{W}}_k \right]
\end{align*}
is the variance of a binomial random variable  $B(N_n(\tilde{\mathcal{W}}_k),\frac{1}{2} )$ conditioned to be lower or equal to $N_n(\tilde{\mathcal{W}}_k) /2 $. According to Technical Lemma \ref{lem:cond_var}, we have
\begin{align}
    \label{eq:In_bound}
    I_n \leq \frac{1}{4} \Esp{  \frac{\ind{N_n(\tilde{\mathcal{W}}_k)>0}}{N_n(\tilde{\mathcal{W}}_k) \Prob{B(N_n(\tilde{\mathcal{W}}_k), 1/2) \leq N_n(\tilde{\mathcal{W}}_k)/2}} \Big| \tilde{\mathcal{W}}_k } \leq \frac{1}{2} \Esp{  \frac{\ind{N_n(\tilde{\mathcal{W}}_k)>0}}{N_n(\tilde{\mathcal{W}}_k) } \Big| \tilde{\mathcal{W}}_k} .
\end{align}
Regarding $J_n$, 
\begin{align}
     Z_{\tilde{\mathcal{W}}_k} - N_n(\tilde{\mathcal{W}}_k)p &= \Esp{\displaystyle \sum_{X_i \in \tilde{\mathcal{W}}_k} Y_i \mid N_1,...,N_{2^{k}}, \tilde{\mathcal{W}}_k} - N_n(\tilde{\mathcal{W}}_k)p \\
    &= \Esp{\displaystyle \sum_{j=1}^{2^{k}} \sum_{X_i \in L_j} Y_i \ind{L_j \subset \tilde{\mathcal{W}}_k} \mid N_1,...,N_{2^{k}}, \tilde{\mathcal{W}}_k} - N_n(\tilde{\mathcal{W}}_k)p \\
    \label{eq:biasProp21}
    &= \displaystyle \sum_{j=1}^{2^{k}} \left(  \Esp{\sum_{X_i \in L_j} Y_i \mid N_1,...,N_{2^{k}}, \tilde{\mathcal{W}}_k } - p N_j \right) \ind{L_j \subset \tilde{\mathcal{W}}_k},
\end{align}
since $ \ind{L_j \subset \tilde{\mathcal{W}}_k}$ is $\tilde{\mathcal{W}}_k$ -measurable and  $N_n(\tilde{\mathcal{W}}_k) = \displaystyle \sum_{i=1}^{2^{k}} N_j$. Noticing that 
\begin{align}
\Esp{ \sum_{X_i \in L_j} Y_i \mid N_1,...,N_{2^{k}}, \tilde{\mathcal{W}}_k } = \Esp{ \sum_{X_i \in L_j} Y_i \mid N_j, \tilde{\mathcal{W}}_k},   
\end{align}
we deduce
\begin{align}
     Z_{\tilde{\mathcal{W}}_k} - N_n(\tilde{\mathcal{W}}_k)p  &= \displaystyle \sum_{j=1}^{2^{k}} \left(  \Esp{\sum_{X_i \in L_j} Y_i \mid N_j, \tilde{\mathcal{W}}_k} -N_j p\right) \ind{L_j \subset \tilde{\mathcal{W}}_k}
\end{align}
and
\begin{align}
    (Z_{\tilde{\mathcal{W}}_k} - N_n(\tilde{\mathcal{W}}_k)p)^2  &= \left( \displaystyle \sum_{j=1}^{2^{k}} f_j \ind{L_j \subset \tilde{\mathcal{W}}_k} \right )^2 \label{eq:upperboundf_j}
\end{align}
with  $f_j = \left(  N_j p - \Esp{\sum_{X_i \in L_j}  Y_i \mid N_j, \tilde{\mathcal{W}}_k} \right)$. For all $j$, such that $L_j \subset \tilde{\mathcal{W}}_k$, $ \Esp{\sum_{X_i \in L_j} Y_i \mid N_j, \tilde{\mathcal{W}}_k}$ is a binomial random variable $\mathfrak{B}(N_n(\tilde{\mathcal{W}}_k),\frac{1}{2})$ conditioned to be lower or equal to $N_n(\tilde{\mathcal{W}}_k) /2$. Using Lemma   \ref{lem:tech_res_binomial} (\ref{lem:tech_res_bin_cond_exp}), we obtain : 
\begin{align}
    f_j &\leq N_j\left(p-\frac{1}{2}\right) + \sqrt{N_j} \left(\frac{1}{\sqrt{\pi}} + \frac{2\sqrt{2}}{\pi \sqrt{(2n+1)}} \right) \\
    \label{eq:fj_bound}
    &\leq N_j \left(p-\frac{1}{2}\right) + \sqrt{N_j} + \frac{2}{\pi} .
\end{align}
Therefore, 
\begin{align}
(Z_{\tilde{\mathcal{W}}_k} - N_n(\tilde{\mathcal{W}}_k)p)^2  & \leq \left(  N_n(\tilde{\mathcal{W}}_k) \left(p-\frac{1}{2}\right) + \sum_{j=1}^{2^k} \sqrt{N_j} \ind{L_j \subset \tilde{\mathcal{W}}_k} + \frac{2^{k+1}}{\pi}  \right )^2\\
& \leq \left(  N_n(\tilde{\mathcal{W}}_k) \left(p-\frac{1}{2}\right) + 2^{k/2} \sqrt{N_n(\tilde{\mathcal{W}}_k)}   + \frac{2^{k+1}}{\pi}  \right )^2, \label{eq:prop3:var1}
\end{align}
since, according to Cauchy-Schwarz inequality, 
\begin{align}
\sum_{j=1}^{2^k} \sqrt{N_j} \ind{L_j \subset \tilde{\mathcal{W}}_k} \leq 2^{k/2} N_n(\tilde{\mathcal{W}}_k)^{1/2} \label{eq:cs_sqrt_nb_points}.   
\end{align}
Overall 
\begin{align}
    \label{eq:J_n_upper_bound}
    J_n &\leq  \mathbb{E} \left[ \frac{1}{N_n(\tilde{\mathcal{W}}_k)^2} \mathbb{E}  \left[  \left(  N_n(\tilde{\mathcal{W}}_k) \left(p-\frac{1}{2}\right) + 2^{k/2} N_n(\tilde{\mathcal{W}}_k)^{1/2} + \frac{2^{k+1}}{\pi}  \right )^2
    \mid N_1,...,N_{2^{k}}, \tilde{\mathcal{W}}_k \right] \Big| \tilde{\mathcal{W}}_k \right] \\
    &\leq \left(p-\frac{1}{2}\right)^2 + 2^k \Esp{\frac{ \ind{N_n(\tilde{\mathcal{W}}_k) >0}}{N_n(\tilde{\mathcal{W}}_k)} \Big| \tilde{\mathcal{W}}_k} +\frac{2^{2k+2}}{\pi^2} \Esp{\frac{ \ind{N_n(\tilde{\mathcal{W}}_k) >0}}{N_n(\tilde{\mathcal{W}}_k)^2}\Big| \tilde{\mathcal{W}}_k} +2^{k/2+1} \left(p-\frac{1}{2}\right)\Esp{\frac{ \ind{N_n(\tilde{\mathcal{W}}_k) >0}}{N_n(\tilde{\mathcal{W}}_k)^{1/2}}\Big| \tilde{\mathcal{W}}_k} \\
    & \quad + \frac{2^{k+2}}{\pi}\left(p-\frac{1}{2}\right) \Esp{\frac{ \ind{N_n(\tilde{\mathcal{W}}_k) >0}}{N_n(\tilde{\mathcal{W}}_k)}\Big| \tilde{\mathcal{W}}_k} + \frac{2^{\frac{3k}{2}+2}}{\pi} \Esp{\frac{ \ind{N_n(\tilde{\mathcal{W}}_k) >0}}{N_n(\tilde{\mathcal{W}}_k)^{3/2}}\Big| \tilde{\mathcal{W}}_k}.
\end{align}
All together, we obtain
\begin{align*}
    I_n+J_n &\leq \left(p-\frac{1}{2}\right)^2 + \left(2^k+ \frac{1}{2} + \frac{2^{k+2}}{\pi}\left(p-\frac{1}{2}\right)\right) \Esp{\frac{ \ind{N_n(\tilde{\mathcal{W}}_k) >0}}{N_n(\tilde{\mathcal{W}}_k)} \Big| \tilde{\mathcal{W}}_k} +\frac{2^{2k+2}}{\pi^2} \Esp{\frac{ \ind{N_n(\tilde{\mathcal{W}}_k) >0}}{N_n(\tilde{\mathcal{W}}_k)^2}\Big| \tilde{\mathcal{W}}_k}  \\
    & \quad + 2^{k/2+1} \left(p-\frac{1}{2}\right)\Esp{\frac{ \ind{N_n(\tilde{\mathcal{W}}_k) >0}}{N_n(\tilde{\mathcal{W}}_k)^{1/2}}\Big| \tilde{\mathcal{W}}_k} + \frac{2^{\frac{3k}{2}+2}}{\pi} \Esp{\frac{ \ind{N_n(\tilde{\mathcal{W}}_k) >0}}{N_n(\tilde{\mathcal{W}}_k)^{3/2}}\Big| \tilde{\mathcal{W}}_k}
\end{align*}
We apply Lemma \ref{lem:tech_res_binomial}(i)(iv) to $N_n(\tilde{\mathcal{W}}_k)$ which is a binomial $\mathfrak{B}(n,p')$ where $p' = \mathbb{P}(X \in \tilde{\mathcal{W}}_k | \tilde{\mathcal{W}}_k)$  :
$$ \Esp{\frac{ \ind{N_n(\tilde{\mathcal{W}}_k) >0}}{N_n(\tilde{\mathcal{W}}_k)} \Big| \tilde{\mathcal{W}}_k} \leq \frac{2}{(n+1)p'},$$
\begin{align*}
\Esp{\frac{ \ind{N_n(\tilde{\mathcal{W}}_k) >0}}{N_n(\tilde{\mathcal{W}}_k)^{1/2}}\Big| \tilde{\mathcal{W}}_k}  \leq \frac{2}{\sqrt{n \cdot p'}}.
\end{align*}

We deduce that 

\begin{align*} 
 \label{eq:in_jn_bound}
 I_n + J_n \leq (p-\frac{1}{2})^2 + \frac{2^{k/2+2}(p-\frac{1}{2})}{\sqrt{\pi n \cdot p' }}+\frac{2 }{(n+1) \cdot p'} \left(2^k+ \frac{1}{2}+ \frac{2^{k+2}}{\pi} + \frac{2^{3k/2 + 2}}{\pi \sqrt{\pi}} + 3 \cdot \frac{2^{2k+2}}{\pi^2} \right). 
\end{align*}

Finally, 
\begin{align*}
& \Esp{ (r_{k,1,n}(X) - p)^2 \ind{X \in \mathcal{B}} \ind{X \in \tilde{\mathcal{B}}_k} } + \Esp{ (r_{k,1,n}(X) - p)^2 \ind{X \in \mathcal{B}} \ind{X \in \tilde{\mathcal{W}}_k} }  \\
\leq & \left(p-\frac{1}{2}\right)^2 \Prob{X \in \tilde{\mathcal{B}}_k, X \in \mathcal{B}} +  \Esp{ (I_n+J_n) \Prob{X \in \tilde{\mathcal{W}}_k, X \in \mathcal{B} | \tilde{\mathcal{W}}_k}}
\end{align*}
Since for all $\tilde{\mathcal{B}}_k$, there is exactly the same number of black cells and white cells in $\tilde{\mathcal{B}}_k$, we have 
$$\Prob{X \in \tilde{\mathcal{W}}_k, X \in \mathcal{B} | \tilde{\mathcal{W}}_k} = \frac{\Prob{X \in \tilde{\mathcal{W}}_k | \tilde{\mathcal{W}}_k}}{2},$$
yielding
\begin{align}
& \Esp{ (r_{k,1,n}(X) - p)^2 \ind{X \in \mathcal{B}} \ind{X \in \tilde{\mathcal{B}}_k} } + \Esp{ (r_{k,1,n}(X) - p)^2 \ind{X \in \mathcal{B}} \ind{X \in \tilde{\mathcal{W}}_k} }  \\
\leq & \frac{1}{2} \left( p - \frac{1}{2} \right)^2 + \frac{2^{k/2+1}(p-\frac{1}{2})}{\sqrt{\pi n}} +\frac{1}{(n+1)} \left(2^k+ \frac{1}{2}+ \frac{2^{k+2}}{\pi} + \frac{2^{3k/2 + 2}}{\pi \sqrt{\pi}} + 3 \cdot \frac{2^{2k+2}}{\pi^2} \right) \\
\leq & \frac{1}{2} \left( p - \frac{1}{2} \right)^2 + \frac{2^{k/2+1}(p-\frac{1}{2})}{\sqrt{\pi n}} +  \frac{3 \cdot 2^{2k+2}}{(n+1) \pi^2} (1 + \varepsilon_{1}(k))
\label{eq:final_inequality2}
\end{align}
where $\varepsilon_1(k) = \frac{\pi^2}{3 \cdot 2^{(2k+2)}} \left(2^k+ \frac{1}{2}+ \frac{2^{k+2}}{\pi} + \frac{2^{3k/2 + 2}}{\pi \sqrt{\pi}} \right)$.

The two intermediate terms of \eqref{eq:risk_decompo_bw} can be similarly bounded from above. Indeed, 
\begin{align}
& \Esp{ (r_{k,1,n}(X) - (1-p))^2 \ind{X \in \mathcal{W}} \ind{X \in \tilde{\mathcal{B}}_k}} + \Esp{ (r_{k,1,n}(X) - (1-p))^2 \ind{X \in \mathcal{W}} \ind{X \in \tilde{\mathcal{W}}_k}}\\
& = \Esp{ \Esp{\left(\frac{1}{N_n(\tilde{\mathcal{B}}_k)} \displaystyle \sum_{X_i \in \tilde{\mathcal{B}}_k} Y_i -(1-p)\right)^2 \Big| \tilde{\mathcal{B}}_k }   \Prob{X \in \tilde{\mathcal{B}}_k, X \in \mathcal{W} | \tilde{\mathcal{B}}_k } } \nonumber 
\\
\label{eq:two_inter_bound}
& \qquad + \Esp{ \Esp{ \left(\frac{1}{N_n(\tilde{\mathcal{W}}_k)} \displaystyle \sum_{X_i \in \tilde{\mathcal{W}}_k} Y_i -(1-p)\right)^2 \Bigg| \tilde{\mathcal{W}}_k } \Prob{X \in \tilde{\mathcal{W}}_k, X \in \mathcal{W} | \tilde{\mathcal{W}}_k}},
\end{align}
where, by definition of $\tilde{\mathcal{W}}_k$, 
$$ \Esp{\left(\frac{1}{N_n(\tilde{\mathcal{W}}_k)} \displaystyle \sum_{X_i \in \tilde{\mathcal{W}}_k} Y_i -(1-p)\right)^2 \Big| \tilde{\mathcal{W}}_k} \leq \left(p-\frac{1}{2}\right)^2.$$ 
The first term in \eqref{eq:two_inter_bound} can be treated similarly as above: 
\begin{align}
\Esp{\left( \frac{1}{N_n(\tilde{\mathcal{B}}_k)} \displaystyle \sum_{X_i \in \tilde{\mathcal{B}}_k} Y_i -(1-p)\right)^2 \Big| \tilde{\mathcal{B}}_k} 
    &=   \mathbb{E} \left[ \frac{1}{N_n(\tilde{\mathcal{B}}_k)^2} \mathbb{E}  \left[ \left(\displaystyle \sum_{X_i \in \tilde{\mathcal{B}}_k} Y_i - Z_{\tilde{\mathcal{B}}_k}\right)^2 \mid N_1,...,N_{2^{k}}, \tilde{\mathcal{B}}_k \right]  \Big| \tilde{\mathcal{B}}_k \right] \nonumber \\
    & \quad +  \mathbb{E} \left[ \frac{1}{N_n(\tilde{\mathcal{B}}_k)^2} \mathbb{E}  \left[  \left(  Z_{\tilde{\mathcal{B}}_k} - N_n(\tilde{\mathcal{B}}_k)(1-p)\right)^2
    \mid N_1,...,N_{2^{k}}, \tilde{\mathcal{B}}_k \right]  \Big| \tilde{\mathcal{B}}_k \right] \nonumber \\
    & = I_n' + J_n', \label{eq:cross_term_white}
\end{align}
where 
$$ Z_{\tilde{\mathcal{B}}_k}= \Esp{ \sum_{X_i \in \tilde{\mathcal{B}}_k} Y_i \mid N_1,...,N_{2^{k}}, \tilde{\mathcal{B}}_k},
$$
and the cross-term in \eqref{eq:cross_term_white} is null according to the definition of $Z_{\tilde{\mathcal{B}}_k}$. Regarding $I_n'$, note that 
\begin{align*}
\mathbb{E}  \left[ \left(\displaystyle \sum_{X_i \in \tilde{\mathcal{B}}_k} Y_i - Z_{\tilde{\mathcal{B}}_k}\right)^2 \mid N_1,...,N_{2^{k}}, \tilde{\mathcal{B}}_k \right]
\end{align*}
is the variance of a binomial random variable  $B(N_n(\tilde{\mathcal{B}}_k),\frac{1}{2} )$ conditioned to be strictly larger than  $N_n(\tilde{\mathcal{B}}_k) /2 $. According to Technical Lemma \ref{lem:cond_var}, we have
\begin{align}
    \label{eq:In_prime_bound}
    I_n' \leq \frac{1}{4} \Esp{  \frac{\ind{N_n(\tilde{\mathcal{B}}_k)>0}}{N_n(\tilde{\mathcal{B}}_k) \Prob{B(N_n(\tilde{\mathcal{B}}_k), 1/2) > N_n(\tilde{\mathcal{B}}_k)/2}} \Big| \tilde{\mathcal{B}}_k } \leq \Esp{  \frac{\ind{N_n(\tilde{\mathcal{B}}_k)>0}}{N_n(\tilde{\mathcal{B}}_k) } \Big| \tilde{\mathcal{B}}_k}.
\end{align}
To obtain the last inequality, notice that
\begin{align*}
    \Prob{B(N_n(\tilde{\mathcal{B}}_k), 1/2) > N_n(\tilde{\mathcal{B}}_k)/2} &= \frac{1}{2} - \frac{1}{2} \Prob{B(N_n(\tilde{\mathcal{B}}_k), 1/2) = N_n(\tilde{\mathcal{B}}_k)/2} \\
    &\geq \frac{1}{2} \left(1 - \frac{1}{\sqrt{\pi (n/2 + 1/4)}} \right) \geq \frac{1}{4}
\end{align*}

as soon as $n\geq 4$.

Regarding $J_n'$, we have
\begin{align}
    &\mathbb{E} \left[ \frac{1}{N_n(\tilde{\mathcal{B}}_k)^2} \mathbb{E}  \left[  \left(  Z_{\tilde{\mathcal{B}}_k} - N_n(\tilde{\mathcal{B}}_k)(1-p)\right)^2
    \mid N_1,...,N_{2^{k}}, \tilde{\mathcal{B}}_k \right] \right] \\
    &= \mathbb{E} \left[ \frac{1}{N_n(\tilde{\mathcal{B}}_k)^2} \mathbb{E}  \left[  \left(  \displaystyle \sum_{i=1}^{2^k} \left( \Esp{\sum_{X_i \in L_i} Y_i \mid N_j, \tilde{\mathcal{B}}_k} - N_j(1-p) \right) \ind{L_j \subset \tilde{\mathcal{B}}_k}\right)^2
    \mid N_1,...,N_{2^{k}}, \tilde{\mathcal{B}}_k \right] \right].
\end{align}
For all $j$, such that $L_j \subset \tilde{\mathcal{B}}_k$, $ \Esp{\sum_{X_i \in L_j} Y_i \mid N_j, \tilde{\mathcal{B}}_k}$ is a binomial random variable $\mathfrak{B}(N_j,\frac{1}{2})$ conditioned to be larger than $\lfloor (N_j+1)/2 \rfloor$. Then, according to Technical Lemma (\ref{lem:tech_res_bin_cond_exp_b})
$$ \Esp{\sum_{X_i \in L_j} Y_i \mid N_j, \tilde{\mathcal{B}}_k} \leq \frac{N_j}{2}+1+\frac{1}{\sqrt{\pi(N_j+1)}}.$$
Hence, 
\begin{align}
    \Esp{\sum_{X_i \in L_i} Y_i \mid N_j, \tilde{\mathcal{B}}_k} - N_j(1-p) &\leq N_j(p-\frac{1}{2}) + 1 + \frac{1}{\sqrt{\pi(N_j+1)}} \\
    \label{eq:bound_j_n_b}
    &\leq N_j \left(p-\frac{1}{2}\right) + \sqrt{N_j} + \frac{2}{\pi},
\end{align}
for $N_j \geq 1$. Thus, 
\begin{align}
    &\mathbb{E} \left[ \frac{1}{N_n(\tilde{\mathcal{B}}_k)^2} \mathbb{E}  \left[  \left(  Z_{\tilde{\mathcal{B}}_k} - N_n(\tilde{\mathcal{B}}_k)(1-p)\right)^2
    \mid N_1,...,N_{2^{k}}, \tilde{\mathcal{B}}_k \right] \right] \\
    &\leq \mathbb{E} \left[ \frac{1}{N_n(\tilde{\mathcal{B}}_k)^2} \mathbb{E}  \left[  \left(  \displaystyle \sum_{i=1}^{2^k} \left(N_j \left(p-\frac{1}{2}\right) + \sqrt{N_j} + \frac{2}{\pi} \right) \ind{L_j \subset \tilde{\mathcal{B}}_k} \right)^2
    \mid N_1,...,N_{2^{k}}, \tilde{\mathcal{B}}_k \right] \right]\\
    &\leq \mathbb{E} \left[ \frac{1}{N_n(\tilde{\mathcal{B}}_k)^2} \mathbb{E}  \left[  \left(  N_n(\tilde{\mathcal{B}}_k) \left(p-\frac{1}{2}\right) + 2^{k/2} \sqrt{N_n(\tilde{\mathcal{B}}_k)}   + \frac{2^{k+1}}{\pi}  \right )^2
    \mid N_1,...,N_{2^{k}}, \tilde{\mathcal{B}}_k \right] \right].
\end{align}
All together, we obtain
\begin{align*}
    I_n'+J_n' &\leq \left(p-\frac{1}{2}\right)^2 + \left(2^k+ 1 + \frac{2^{k+2}}{\pi}\left(p-\frac{1}{2}\right)\right) \Esp{\frac{ \ind{N_n(\tilde{\mathcal{B}}_k) >0}}{N_n(\tilde{\mathcal{B}}_k)} \Big| \tilde{\mathcal{B}}_k} +\frac{2^{2k+2}}{\pi^2} \Esp{\frac{ \ind{N_n(\tilde{\mathcal{B}}_k) >0}}{N_n(\tilde{\mathcal{B}}_k)^2}\Big| \tilde{\mathcal{B}}_k}  \\
    & \quad + 2^{k/2+1} \left(p-\frac{1}{2}\right)\Esp{\frac{ \ind{N_n(\tilde{\mathcal{B}}_k) >0}}{N_n(\tilde{\mathcal{B}}_k)^{1/2}}\Big| \tilde{\mathcal{B}}_k} + \frac{2^{\frac{3k}{2}+2}}{\pi} \Esp{\frac{ \ind{N_n(\tilde{\mathcal{B}}_k) >0}}{N_n(\tilde{\mathcal{B}}_k)^{3/2}}\Big| \tilde{\mathcal{B}}_k}
\end{align*}

The computation is similar to \eqref{eq:in_jn_bound}, with $p'' = \mathbb{P}\left( X \in \tilde{\mathcal{B}}_k \mid \tilde{\mathcal{B}}_k \right)$ :
\begin{align*}
    I_n + J_n &\leq \left(p-\frac{1}{2}\right)^2 + \frac{2^{k/2+3} (p-\frac{1}{2})}{\sqrt{\pi n \cdot p''}} + \left(2^k +1 + \frac{2^{k+2}}{\pi} \left(p - \frac{1}{2} \right) + \frac{2^{3k/2+2}}{\pi} + \frac{2^{2k+2}}{\pi^2} \right) \frac{2 }{(n+1) p''} \\
    &\leq  \left(p-\frac{1}{2}\right)^2 + \frac{2^{k/2+3} (p-\frac{1}{2})}{\sqrt{\pi n \cdot p''}} + \frac{2^{2k+3}}{\pi^2 (n+1)p''} ( 1 + \varepsilon_2(k))
\end{align*}
with $\varepsilon_2(k) = \frac{\pi^2}{2^{(2k+3)}} \left(2^k+ 1 + \frac{2^{k+2}}{\pi} \left(p-1/2\right) + \frac{2^{3k/2 + 2}}{\pi} \right)$.
Finally, 
\begin{align*}
& \Esp{ (r_{k,1,n}(X) - (1-p))^2 \ind{X \in \mathcal{W}} \ind{X \in \tilde{\mathcal{B}}_k}} + \Esp{ (r_{k,1,n}(X) - (1-p))^2 \ind{X \in \mathcal{W}} \ind{X \in \tilde{\mathcal{W}}_k}}  \\
\leq &   \Esp{ (I_n'+J_n') \Prob{ X \in \mathcal{W}, X \in \tilde{\mathcal{B}}_k | \tilde{\mathcal{B}}_k}} + \left(p-\frac{1}{2}\right)^2 \Prob{ X \in \mathcal{W}, X \in \tilde{\mathcal{W}}_k}\\
\leq & \mathbb{E} \left[ \left( \left(p-\frac{1}{2}\right)^2 + \frac{2^{k/2+3} (p-\frac{1}{2})}{\sqrt{\pi n \cdot p''}} + \frac{2^{2k+3}}{\pi^2 (n+1)p''} ( 1 + \varepsilon_2(k)) \right)  \Prob{ X \in \mathcal{W}, X \in \tilde{\mathcal{B}}_k | \tilde{\mathcal{B}}_k}\right] \\
&+ \left(p-\frac{1}{2}\right)^2 \Prob{ X \in \mathcal{W}, X \in \tilde{\mathcal{W}}_k}.
\end{align*}
Since for all $\tilde{\mathcal{B}}_k$, there is exactly the same number of black cells and white cells in $\tilde{\mathcal{B}}_k$, we have 
$$\Prob{ X \in \mathcal{W}, X \in \tilde{\mathcal{B}}_k | \tilde{\mathcal{B}}_k} = \frac{p''}{2},$$
yielding 
\begin{align}
\Esp{ (r_{k,1,n}(X) - (1-p))^2 \ind{X \in \mathcal{W}} \ind{X \in \tilde{\mathcal{B}}_k}} &+ \Esp{ (r_{k,1,n}(X) - (1-p))^2 \ind{X \in \mathcal{W}} \ind{X \in \tilde{\mathcal{W}}_k}} \nonumber  \\
&\leq \frac{1}{2} \left( p - \frac{1}{2} \right)^2 + \frac{2^{k/2+2} (p-\frac{1}{2})}{\sqrt{\pi n}} + \frac{2^{2k+3}}{2 \cdot \pi^2 (n+1)} ( 1 + \varepsilon_2(k)). \label{eq:final_inequality3}
\end{align}
Gathering \eqref{eq:empty_cell_risk}, \eqref{eq:final_inequality2} and \eqref{eq:final_inequality3}, we have 
\begin{align*}
\Esp{(r_{k,1,n}(X) - r(X))^2 } \leq \left( p - \frac{1}{2} \right)^2 + \frac{2^{k/2+3} (p-\frac{1}{2})}{\sqrt{\pi n}} + \frac{7 \cdot 2^{2k+2}}{\pi^2 (n+1)} (1 + \varepsilon(k)) + \frac{p^2+(1-p)^2}{2} \left(1-\frac{1}{2^k}\right)^n
\end{align*}
where $\varepsilon(k) = \frac{6 \varepsilon_1(k) + \varepsilon_2(k)}{7} $.

\subsubsection{\textbf{Proof of 2. (lower-bound)}}

We have, according to \eqref{eq:empty_cell_risk},
\begin{align}
    \notag
    \Esp{ (r_{k,1,n}(X)-r(X))^2}     & = \Esp{ (r_{k,1,n}(X)-r(X))^2 \ind{N_n(C_n(X)>0)} } + \Esp{ (r(X))^2 \ind{N_n(C_n(X)=0)} }\\
    & = \Esp{ (r_{k,1,n}(X)-r(X))^2 \ind{N_n(C_n(X)>0)} } + \frac{p^2+(1-p)^2}{2} \Prob{N_n(C_n(X) = 0}.
\end{align}
Letting $ Z_2 = \Esp{ \sum_{X_i \in C_n(X)} Y_i \mid N_1,...,N_{2^{k}}, C_n(X)}$, we have
\begin{align}
    & \Esp{ (r_{k,1,n}(X)-r(X))^2 \ind{N_n(C_n(X)>0)} }\\
    &=\Esp{\left(\frac{\ind{N_n(C_n(X)>0)}}{N_n(C_n(X))} \sum_{X_i \in C_n(X)} Y_i -r(X) \right)^2 \ind{N_n(C_n(X)>0)}} \\
    &= \Esp{\frac{\ind{N_n(C_n(X)>0)}}{N_n(C_n(X))^2} \Esp{ \left( \sum_{X_i \in C_n(X)} Y_i -N_n(C_n(X))r(X) \right)^2 \mid N_1,...,N_{2^k},C_n(X)}}\\
    &= \mathbb{E} \left[ \frac{\ind{N_n(C_n(X)>0)}}{N_n(C_n(X))^2} \mathbb{E}  \left[ \left(\displaystyle \sum_{X_i \in C_n(X)} Y_i - Z_2\right)^2 + \left(  Z_2 - N_n(C_n(X))r(X)\right)^2 \right. \right. \\
    & \qquad \left. \left. + 2 \left (\displaystyle \sum_{X_i \in C_n(X)} Y_i - Z_2\right)  \left(  Z_2 - N_n(C_n(X))r(X) \right) \mid N_1,...,N_{2^{k}}, C_n(X) \right] \right].
\end{align}
The cross-term is null according to the definition of $Z$ and because $(Z_2 -N_n(C_n(X)))$ is $(N_1,...,N_{2^{k}}, C_n(X))$ - measurable. Therefore, 
\begin{align}
    & \Esp{\left( \frac{\ind{N_n(C_n(X)>0)}}{N_n(C_n(X))} \displaystyle \sum_{X_i \in C_n(X)} Y_i -r(X)\right)^2 \ind{N_n(C_n(X)>0)}} \\
    &= \mathbb{E} \left[ \frac{\ind{N_n(C_n(X)>0)}}{N_n(C_n(X))^2} \mathbb{E}  \left[ \left(\displaystyle \sum_{X_i \in C_n(X)} Y_i - Z_2\right)^2 \mid N_1,...,N_{2^{k}}, C_n(X) \right] \right] \\
    \notag
    & \quad +\mathbb{E} \left[ \frac{\ind{N_n(C_n(X)>0)}}{N_n(C_n(X))^2} \mathbb{E}  \left[  \left(  Z_2 - N_n(C_n(X))r(X)\right)^2
    \mid N_1,...,N_{2^{k}}, C_n(X) \right] \right] \\
    &= I_n + J_n,
    \label{eq:sum_in_jn}
\end{align}
where $I_n$ and $J_n$ are respectively a variance and bias term. Now, note that
\begin{align}
\notag
     & \mathbb{E}  \left[  \left(  Z_2 - N_n(C_n(X))r(X)\right)^2
    \mid N_1,...,N_{2^{k}}, C_n(X) \right]\\ 
    &= \mathbb{E}  \left[  \left( Z_2 - N_n(C_n(X))p\right)^2 \ind{X \in \mathcal{B}}  + \left(  Z_2 - N_n(C_n(X))(1-p)\right)^2 \ind{X \in \mathcal{W}}
    \mid N_1,...,N_{2^k}, C_n(X) \right].
\end{align}
Additionally, 
\begin{align*}
    \Prob{ X \in \mathcal{B} \mid N_1,...,N_{2^k}, C_n(X) } = \Prob{ X \in \mathcal{W} \mid N_1,...,N_{2^k}, C_n(X) } = 1/2.
\end{align*}
Consequently, 
\begin{align}
    \notag
    & \mathbb{E}  \left[  \left(  Z_2 - N_n(C_n(X))r(X)\right)^2
    \mid N_1,...,N_{2^{k}}, C_n(X) \right]\\
    &= \frac{1}{2} \mathbb{E}  \left[ \left(  Z_2 - N_n(C_n(X))p\right)^2 + \left(  Z_2 - N_n(C_n(X))(1-p)\right)^2
    \mid N_1,...,N_{2^{k}}, C_n(X) \right].
\end{align}
A small computation shows that for all $x \in \mathbb{R}$, for all $N \in \mathbb{N}$
\begin{align*}
    (x - Np)^2 + (x - N(1-p))^2 \geq 2 N^2 (p-\frac{1}{2})^2,
\end{align*}
which leads to 
\begin{align*}
    J_n \geq \left( p - \frac{1}{2} \right)^2 \Prob{N_n(C_n(X)) >0}.
\end{align*}
All in all, 
\begin{align}
\Esp{ (r_{k,1,n}(X)-r(X))^2} & = I_n +  J_n +  \frac{p^2+(1-p)^2}{2}\Prob{N_n(C_n(X)) =0} \\
& \geq \left( p - \frac{1}{2} \right)^2 \Prob{N_n(C_n(X)) >0} + \frac{p^2+(1-p)^2}{2}\Prob{N_n(C_n(X)) =0}\\
& \geq \left( p - \frac{1}{2} \right)^2.
\end{align}

\clearpage

\section{Proof of Proposition \ref{prop:new_risk_large_k}}
\label{proof:new_risk_large_k}
\subsection{Proof of statement 1.: risk of a single tree}
As in the precedent proof, we distinguish the case where the cell containing $X$ might be empty, in such a case the tree will predict 0:
\begin{align}
    \notag
    \mathbb{E}& \left[ (r_{k,0,n}(X) - r(X))^2 )\right] \\
    & = \Esp{(r_{k,0,n}(X) - r(X))^2 \ind{\cardcellof{X} >0}} + \Esp{(r(X))^2 \ind{\cardcellof{X} =0}}\\
    & = \Esp{(r_{k,0,n}(X) - r(X))^2 \ind{\cardcellof{X} >0}} + (p^2 + (1-p)^2) \frac{(1-2^{-k})^n}{2}.
    \label{eq2_proof2}
\end{align}
We denote by $L_1,...,L_{2^k}$ the leaves of the tree. Let $b \in \{1,\hdots , 2^k\}$ such that $L_b$ belongs to $\mathcal{B}$. We have
\begin{align}
    \notag
    \mathbb{E} &\left[ (r_{k,0,n}(X) - p)^2 ) \mathbb{1}_{X \in \mathcal{B}} \ind{\cardcellof{X} >0} \right]  \\
    &=
    \displaystyle \sum_{L_j \subset \mathcal{B}} \mathbb{E} \left[ \left(\frac{\ind{N_n(L_j) >0}}{N_n(L_j) } \displaystyle \sum_{X_i \in L_j} (Y_i -p) \right)^2 \ind{X \in L_j} \right]  \\
     &=  \frac{2^k}{2} \cdot \mathbb{E} \left[ \left(\frac{\ind{N_n(L_b) >0}}{N_n(L_b) } \displaystyle \sum_{X_i \in L_b} (Y_i -p) \right)^2 \right] \Prob{X \in L_b }  \\
    &=  \frac{1}{2} \mathbb{E} \left[ \left(\frac{\ind{N_n(L_b) >0}}{N_n(L_b)} \displaystyle \sum_{X_i \in L_b} (Y_i -p) \right)^2 \right] \\
    &= \frac{1}{2} \mathbb{E} \left[ \frac{\ind{N_n(L_b) >0}}{N_n(L_b)^ 2} \Esp{\left( \displaystyle \sum_{X_i \in L_b} (Y_i -p) \right)^2  | N_n(L_b)} \right] \\
     &= \frac{1}{2} \mathbb{E} \left[ \frac{\ind{N_n(L_b) >0}}{N_n(L_b)^2}  \Esp{\displaystyle \sum_{X_i \in L_b} (Y_i -p)^2 | N_n(L_b) } \right]   \qquad (\text{by independence of the $Y_i$}) \\
    &= \frac{1}{2} \Esp{ \frac{\ind{N_n(L_b) >0}}{N_n(L_b)} p(1-p)} .
    \label{eq3_proof3}
\end{align}

Remark that the above computation holds when $X \in \mathcal{W}$ after replacing $p$ by $(1-p)$, $\mathcal{B}$ by $\mathcal{W}$ and $L_b$ by $L_w$: indeed when $Y$ is a Bernoulli random variable, $Y$ and $1-Y$ have the same variance. Hence, using Equation~\eqref{eq2_proof2}, the computation in \eqref{eq3_proof3} and its equivalence for $\mathcal{W}$, we obtain
\begin{align*}
    & \mathbb{E} \left[ (r_{k,0,n}(X) - r(X))^2 )\right] \\
    &= \frac{1}{2} \Esp{\frac{\ind{N_n(L_b) >0}}{N_n(L_b)} p(1-p)} + \frac{1}{2} \Esp{ \frac{\ind{N_n(L_w) >0}}{N_n(L_w)} p(1-p)}  + (p^2 + (1-p)^2) \frac{(1-2^{-k})^n}{2} \\
    &=  p(1-p)  \Esp{\frac{\ind{N_n(L_w) >0}}{N_n(L_w)}}  + (p^2 + (1-p)^2) \frac{(1-2^{-k})^n}{2},
\end{align*}
since $N_n(L_b)$ and $N_n(L_w)$ are both binomial random variables $\mathfrak{B}(n, \frac{1}{2^{k}})$. Therefore we can conclude using Lemma \ref{lem:tech_res_binomial} (i):
$$ \mathbb{E} \left[ (r_{k,0,n}(X) - r(X))^2 )\right] \leq \frac{2^{k}p(1-p)}{n+1} + \left(p^2 + (1-p)^2\right) \frac{(1-2^{-k})^n}{2}
$$
and 
$$ \mathbb{E} \left[ (r_{k,0,n}(X) - r(X))^2 )\right] \geq \frac{2^{k-1}p(1-p)}{n+1}  + \Bigg(p^2 +  (1-p)^2  -  \frac{2^{k}p(1-p)}{n+1} \Bigg) \frac{(1-2^{-k})^n}{2}.
$$

\subsection{Proof of statement 2.: risk of a shallow tree network}
\label{proof:snd_statement_balanced_damier}

Let $k \in \mathbb{N}$. Denote by $\mathcal{L}_{k} = \{L_{i}, i  = 1, \hdots, 2^k \} $ the set of all leaves of the encoding tree (of depth $k$). We let  $\mathcal{L}_{\tilde{\mathcal{B}}_k}$ be the set of all cells of the encoding tree containing at least one observation, and such that the empirical probability of $Y$ being equal to one in the cell is larger than $1/2$, i.e. 
\begin{align*}
\tilde{\mathcal{B}}_k = \cup_{L \in \mathcal{L}_{\tilde{\mathcal{B}}_k}} \{x, x \in L\} 
\end{align*}
\begin{align*}
\mathcal{L}_{\tilde{\mathcal{B}}_k} = \{ L \in \mathcal{L}_{k}, N_n(L) >0, \frac{1}{N_n(L)} \displaystyle \sum_{X_i \in L} Y_i \geq \frac{1}{2} \}. 
\end{align*}
Accordingly, we let the part of the input space corresponding to $\mathcal{L}_{\tilde{\mathcal{B}}_k}$ as 
\begin{align*}
\tilde{\mathcal{B}}_k = \cup_{L \in \mathcal{L}_{\tilde{\mathcal{B}}_k}} \{x, x \in L\} 
\end{align*}
Similarly, 
\begin{align*}
\mathcal{L}_{\tilde{\mathcal{W}}_k} = \{ L \in \mathcal{L}_{k}, N_n(L) >0, \frac{1}{N_n(L)} \displaystyle \sum_{X_i \in L} Y_i < \frac{1}{2} \}. 
\end{align*}
and
\begin{align*}
\tilde{\mathcal{W}}_k = \cup_{L \in \mathcal{L}_{\tilde{\mathcal{W}}_k}} \{x, x \in L\} 
\end{align*}

\subsubsection{\textbf{Proof of 2. (upper-bound)}}
Recall that $k\geq k^\star$. In this case, each leaf of the encoding tree is included in a chessboard cell. As usual,
\begin{align}
    \mathbb{E} \left[ (r_{k,1,n}(X) - r(X))^2 )\right] &= \Esp{(r_{k,1,n}(X) - r(X))^2 \ind{\cardcellof{X} > 0}} + \frac{p^2+(1-p)^2}{2} \left(1-\frac{1}{2^{k}}\right)^n.
\end{align}
Note that
\begin{align}
    \notag
    &\Esp{(r_{k,1,n}(X) - r(X))^2 \ind{\cardcellof{X} > 0}} \\
    &= \mathbb{E} \left[  \left(\frac{1}{N_n(\tilde{\mathcal{B}}_k)} \displaystyle \sum_{X_i \in \tilde{\mathcal{B}}_k} Y_i - p\right)^2 \mathbb{1}_{X \in \mathcal{B}} \mathbb{1}_{X \in \tilde{\mathcal{B}}_k} \right] \nonumber + \mathbb{E} \left[  \left(\frac{1}{N_n(\tilde{\mathcal{W}}_k)} \displaystyle \sum_{X_i \in \tilde{\mathcal{W}}_k} Y_i - p\right)^2 \mathbb{1}_{X \in \mathcal{B}} \mathbb{1}_{X \in \tilde{\mathcal{W}}_k} \right] \nonumber \\
    &+ \mathbb{E} \left[  \left(\frac{1}{N_n(\tilde{\mathcal{B}}_k)} \displaystyle \sum_{X_i \in \tilde{\mathcal{B}}_k} Y_i - (1-p)\right)^2 \mathbb{1}_{X \in \mathcal{W}} \mathbb{1}_{X \in \tilde{\mathcal{B}}_k} \right] \nonumber + \mathbb{E} \left[  \left(\frac{1}{N_n(\tilde{\mathcal{W}}_k)} \displaystyle \sum_{X_i \in \tilde{\mathcal{W}}_k} Y_i - (1-p)\right)^2 \mathbb{1}_{X \in \mathcal{W}} \mathbb{1}_{X \in \tilde{\mathcal{W}}_k} \right] \\
    &\leq  \frac{1}{2} \mathbb{E} \left[ \left(\frac{1}{N_n(\tilde{\mathcal{B}}_k)} \displaystyle \sum_{X_i \in \tilde{\mathcal{B}}_k} Y_i - p \right)^2 \ind{N_n(\tilde{\mathcal{B}}_k) >0} \right ] + \frac{1}{2} \mathbb{E} \left[  \left(\frac{1}{N_n(\tilde{\mathcal{W}}_k)} \displaystyle \sum_{X_i \in \tilde{\mathcal{W}}_k} Y_i - (1-p)\right)^2 \ind{N_n(\tilde{\mathcal{W}}_k) >0} \right] \nonumber \\
    &\qquad \qquad \qquad + \mathbb{E} \left[ \mathbb{1}_{X \in \mathcal{B}, X \in \tilde{\mathcal{W}}_k} \right] + \mathbb{E} \left[ \mathbb{1}_{X \in \mathcal{W}, X \in \tilde{\mathcal{B}}_k} \right].
    \label{eq:kopt_bound1}
\end{align}
Let L be a generic cell. The third term in \eqref{eq:kopt_bound1} can be upper-bounded as follows: 
\begin{align}
    \mathbb{E} \left[ \mathbb{1}_{X \in \mathcal{B}, X \in \tilde{\mathcal{W}}_k} \right] &= \sum_{j=1}^{2^k} \mathbb{E} \left[ \mathbb{1}_{X \in L_j } \mathbb{1}_{L_j \subset  \tilde{\mathcal{W}}_k \cap \mathcal{B}} \right] \\
    & = \sum_{j=1}^{2^k}  \Prob{X \in L_j} \Prob{L_j \subset \tilde{\mathcal{W}}_k \cap \mathcal{B}}  \hspace{2cm} \\
    & = \sum_{j=1}^{2^k}  \Prob{X \in L_j}  \Prob{L_j \subset \tilde{\mathcal{W}}_k \mid L_j \subset \mathcal{B}} \Prob{L_j \subset \mathcal{B}}\\
    &= \frac{1}{2} \Prob{L \subset \tilde{\mathcal{W}}_k \mid L \subset \mathcal{B}},
\end{align}
    by symmetry. Now,
\begin{align}
    \Prob{L \subset \tilde{\mathcal{W}}_k \mid L \subset \mathcal{B}} &= \mathbb{P} \left( \frac{1}{N_n(L)} \displaystyle \sum_{X_i \in L} \mathbb{1}_{Y_i = 0}  > \frac{1}{2} \mid L \subset \mathcal{B} \right) \\
    &\leq  \mathbb{E} \left[ \mathbb{P} \left( \frac{1}{N_n(L)} \displaystyle \sum_{X_i \in L, L \subset \mathcal{B}} \mathbb{1}_{Y_i = 0} -(1-p) \geq \frac{1}{2} - (1-p)  | N_n(L), L \subset \mathcal{B} \right) \mid L \subset \mathcal{B} \right] \\
     &\leq \mathbb{E} \left[ e^{-2N_n(L) (p-\frac{1}{2})^2} \right] \\
     \notag
     & \qquad \text{(according to Hoeffding's inequality)}\\
     &= \displaystyle \prod_{i=1}^n \Esp{  e^{-2(p-\frac{1}{2})^2 \ind{X_i \in L }}} \\
     \notag 
    & \qquad \text{ (by independence of $X_i$'s) }\\
    &=  \left( \frac{e^{-2(p-\frac{1}{2})^2}}{2^{k}} +1 - \frac{1}{2^k} \right)^n.
    \label{proof_eq1}
\end{align}
Consequently, 
\begin{align*}
 \mathbb{E} \left[ \mathbb{1}_{X \in \mathcal{B}, X \in \tilde{\mathcal{W}}_k} \right] & \leq \frac{1}{2}  \left( \frac{e^{-2(p-\frac{1}{2})^2}}{2^{k}} +1 - \frac{1}{2^k} \right)^n.
\end{align*}
Similar calculations show that 
\begin{align}
\mathbb{E} \left[ \mathbb{1}_{X \in \mathcal{W}, X \in \tilde{\mathcal{B}}_k} \right] 
& = \frac{1}{2} \Prob{L \subset \tilde{\mathcal{B}}_k \mid L \subset \mathcal{W}} \nonumber \\
& \leq  \frac{1}{2}  \left( \frac{e^{-2(p-\frac{1}{2})^2}}{2^{k}} +1 - \frac{1}{2^k} \right)^n. \label{proof_eq1bis}
\end{align}
Therefore, 
\begin{align}
    & \mathbb{E} \left[ (r_{k,1,n}(X) - r(X))^2 )\right] \nonumber \\
    & \leq \frac{1}{2} \mathbb{E} \left[ \left(\frac{1}{N_n(\tilde{\mathcal{B}}_k)} \displaystyle \sum_{X_i \in \tilde{\mathcal{B}}_k} Y_i - p \right)^2 \ind{N_n(\tilde{\mathcal{B}}_k) >0} \right ] + \frac{1}{2} \mathbb{E} \left[  \left(\frac{1}{N_n(\tilde{\mathcal{W}}_k)} \displaystyle \sum_{X_i \in \tilde{\mathcal{W}}_k} Y_i - (1-p)\right)^2 \ind{N_n(\tilde{\mathcal{W}}_k) >0} \right] \nonumber \\
    & \qquad + \left( \frac{e^{-2(p-\frac{1}{2})^2}}{2^{k}} +1 - \frac{1}{2^k} \right)^n + \frac{p^2+(1-p)^2}{2} \left(1-\frac{1}{2^{k}}\right)^n. \label{eq:bound_temp1}
\end{align}
Now, the first term in \eqref{eq:bound_temp1} can be written as
\begin{align}
    & \mathbb{E} \left[ \left(\frac{1}{N_n(\tilde{\mathcal{B}}_k)} \displaystyle \sum_{X_i \in \tilde{\mathcal{B}}_k} Y_i - p \right)^2 \ind{N_n(\tilde{\mathcal{B}}_k) >0} \right ]\\ &= \mathbb{E} \left[ \left(\frac{1}{N_n(\tilde{\mathcal{B}}_k)} \displaystyle \sum_{X_i \in \tilde{\mathcal{B}}_k} Y_i - p \right)^2 \ind{N_n(\tilde{\mathcal{B}}_k) >0} \mathbb{1}_{\mathcal{B} = \tilde{\mathcal{B}}_k} \right ]  + \mathbb{E} \left[ \left(\frac{1}{N_n(\tilde{\mathcal{B}}_k)} \displaystyle \sum_{X_i \in \tilde{\mathcal{B}}_k} Y_i - p \right)^2 \ind{N_n(\tilde{\mathcal{B}}_k) >0} \mathbb{1}_{\mathcal{B} \neq \tilde{\mathcal{B}}_k} \right ]\\
    \label{eq:kopt_bound2}
    &\leq \mathbb{E} \left[ \left(\frac{1}{N_n(\tilde{\mathcal{B}}_k)} \displaystyle \sum_{X_i \in \tilde{\mathcal{B}}_k} Y_i - p \right)^2 \ind{N_n(\tilde{\mathcal{B}}_k) >0} \mathbb{1}_{\mathcal{B} = \tilde{\mathcal{B}}_k} \right ]
    + \mathbb{P} \left(\mathcal{B} \neq \tilde{\mathcal{B}}_k \right)
\end{align}
Now, using a union bound, we obtain
\begin{align}
    \mathbb{P} \left(\mathcal{B} \neq \tilde{\mathcal{B}}_k \right) &\leq  \sum_{L_j \subset \mathcal{B}} \Prob{L_j \not\subset \tilde{\mathcal{B}}_k} + \sum_{L_j \subset \mathcal{W}}  \Prob{ L_j \subset \tilde{\mathcal{B}}_k } \\
    &\leq \frac{2^ k}{2} \cdot \Prob{L \not\subset \tilde{\mathcal{B}}_k \mid L \subset \mathcal{B}} + \frac{2^ k}{2} \cdot \Prob{L \subset \tilde{\mathcal{B}}_k \mid L \subset \mathcal{W}} \\
    &\leq 2^k \left( \frac{e^{-2(p-\frac{1}{2})^2}}{2^{k}} +1 - \frac{1}{2^k} \right)^n, \label{eq:prop2b_proba_bound}
\end{align}
according to \eqref{proof_eq1} and \eqref{proof_eq1bis}. Additionally, the left term in \eqref{eq:kopt_bound2} satisfies
\begin{align}
      \mathbb{E} \left[ \left(\frac{1}{N_n(\tilde{\mathcal{B}}_k)} \displaystyle \sum_{X_i \in \tilde{\mathcal{B}}_k} Y_i - p \right)^2 \ind{N_n(\tilde{\mathcal{B}}_k) >0} \mathbb{1}_{\mathcal{B} = \tilde{\mathcal{B}}_k} \right ] 
     &\leq \mathbb{E} \left[ \left(\frac{1}{N_n(\mathcal{B})} \displaystyle \sum_{X_i \in \mathcal{B}} Y_i - p \right)^2 \ind{N_n(\mathcal{B}) >0}  \right]  \\
     &\leq \mathbb{E} \left[ \frac{\ind{N_n(\mathcal{B}) >0}}{N_n(\mathcal{B})^2} \left( \displaystyle \sum_{X_i \in \mathcal{B}} Y_i - pN_n(\mathcal{B}) \right)^2   \right ] \label{eq:prop2_bin_var} \\
    &= p(1-p)\Esp{\frac{\ind{N_n(\mathcal{B}) >0}}{N_n(\mathcal{B})}},
\end{align}
noticing that the square term of \eqref{eq:prop2_bin_var} is nothing but the conditional variance of a binomial distribution $ B(N_n(\mathcal{B}),p)$. By Lemma \ref{lem:tech_res_binomial} (i) on $N_n(\mathcal{B})$ which is a binomial random variable $B(n,p)$ with $p=1/2$ (exactly half of the cells are black),
\begin{align*}
\mathbb{E} \left[ \left(\frac{1}{N_n(\tilde{\mathcal{B}}_k)} \displaystyle \sum_{X_i \in \tilde{\mathcal{B}}_k} Y_i - p \right)^2 \ind{N_n(\tilde{\mathcal{B}}_k) >0} \ind{N_n(\tilde{\mathcal{B}}_k) >0} \right ] \leq \frac{2p(1-p)}{n+1}.
\end{align*}
Hence
\begin{align}
\mathbb{E} \left[ \left(\frac{1}{N_n(\tilde{\mathcal{B}}_k)} \displaystyle \sum_{X_i \in \tilde{\mathcal{B}}_k} Y_i - p \right)^2  \mathbb{1}_{\mathcal{B} = \tilde{\mathcal{B}}_k} \right ] \leq \frac{2p(1-p)}{n+1} + 2^k \left( \frac{e^{-2(p-\frac{1}{2})^2}}{2^{k}} +1 - \frac{1}{2^k} \right)^n.
\label{proof_eq3}
\end{align}
Similarly, 
\begin{align}
\mathbb{E} \left[  \left(\frac{1}{N_n(\tilde{\mathcal{W}}_k)} \displaystyle \sum_{X_i \in \tilde{\mathcal{W}}_k} Y_i - (1-p)\right)^2 \ind{N_n(\tilde{\mathcal{W}}_k) >0} \right] \leq \frac{2p(1-p)}{n+1} + 2^k \left( \frac{e^{-2(p-\frac{1}{2})^2}}{2^{k}} +1 - \frac{1}{2^k} \right)^n.
\label{proof_eq3b}
\end{align}
Finally, 

Injecting \eqref{proof_eq3} and  \eqref{proof_eq3b} into \eqref{eq:bound_temp1}, we finally get
\begin{align*}
    \mathbb{E} \left[ (r_{k,1,n}(X) - r(X))^2 )\right] & \leq \frac{p^2+(1-p)^2}{2} \left(1-\frac{1}{2^{k}}\right)^n  + 2^{k} \cdot \left( \frac{e^{-2(p-\frac{1}{2})^2}}{2^ k} +1 - \frac{1}{2^ k} \right)^n \\
    & \quad + \frac{2p(1-p)}{n+1} + \left( \frac{e^{-2(p-\frac{1}{2})^2}}{2^ k} +1 - \frac{1}{2^ k} \right)^n,
\end{align*}
which concludes this part of the proof. 

\subsubsection{\textbf{Proof of 2. (lower bound)}}

We have
\begin{align*}
\mathbb{E} \left[ (r_{k,1,n}(X) - r(X))^2 )\right] = \Esp{(r_{k,1,n}(X) - r(X))^2 \ind{\cardcellof{X} > 0}} + \left(\frac{p^2+(1-p)^2}{2}\right) \left(1-\frac{1}{2^{k}}\right)^n,    
\end{align*}
where
\begin{align}
    \notag 
     & \Esp{(r_{k,1,n}(X) - r(X))^2 \ind{\cardcellof{X} > 0}}\\
     &\geq \mathbb{E} \left[ \left(\frac{1}{N_n(\tilde{\mathcal{B}}_k)} \displaystyle \sum_{X_i \in \tilde{\mathcal{B}}_k} Y_i - p\right)^2 \mathbb{1}_{X \in \mathcal{B}} \mathbb{1}_{X \in \tilde{\mathcal{B}}_k} \ind{N_n(\tilde{\mathcal{B}}_k) >0} \ind{\mathcal{B} = \tilde{\mathcal{B}}_k} \right]   \nonumber \\
     & \quad +\mathbb{E} \left[ \left(\frac{1}{N_n(\tilde{\mathcal{W}}_k)} \displaystyle \sum_{X_i \in \tilde{\mathcal{W}}_k} Y_i - (1-p)\right)^2 \mathbb{1}_{X \in \mathcal{W}} \mathbb{1}_{X \in \tilde{\mathcal{W}}_k} \ind{N_n(\tilde{\mathcal{W}}_k) >0} \ind{\mathcal{W} = \tilde{\mathcal{W}}_k} \right] \nonumber \\
     &\geq  \Prob{X \in  \mathcal{B} } \mathbb{E} \left[ \left(\frac{1}{N_n(\tilde{\mathcal{B}}_k)} \displaystyle \sum_{X_i \in \tilde{\mathcal{B}}_k} Y_i - p\right)^2 \mathbb{1}_{\mathcal{B} = \tilde{\mathcal{B}}_k} \ind{N_n(\tilde{\mathcal{B}}_k) >0} \right]  \nonumber \\
     & \quad + \Prob{X \in  \mathcal{W} } \mathbb{E} \left[ \left(\frac{1}{N_n(\tilde{\mathcal{W}}_k)} \displaystyle \sum_{X_i \in \tilde{\mathcal{W}}_k} Y_i - (1-p)\right)^2 \mathbb{1}_{\mathcal{W} = \tilde{\mathcal{W}}_k} \ind{N_n(\tilde{\mathcal{W}}_k) >0} \right].
     \label{proof_eq5}
\end{align}
The first expectation term line \eqref{proof_eq5} can be written as
\begin{align}
    \mathbb{E} \left[ \left(\frac{1}{N_n(\tilde{\mathcal{B}}_k)} \displaystyle \sum_{X_i \in \tilde{\mathcal{B}}_k} Y_i - p\right)^2 \mathbb{1}_{\mathcal{B} = \tilde{\mathcal{B}}_k} \ind{N_n(\tilde{\mathcal{B}}_k) >0} \right] &= \Prob{\mathcal{B} = \tilde{\mathcal{B}}_k} \mathbb{E} \left[ \left(\frac{1}{N_n(\mathcal{B})} \displaystyle \sum_{X_i \in \mathcal{B}} Y_i - p\right)^2 | \mathcal{B} = \tilde{\mathcal{B}}_k \right]
\end{align}
According to \eqref{eq:prop2b_proba_bound}, 
 \begin{align}
\Prob{ \mathcal{B} = \tilde{\mathcal{B}}_k} & \geq 1-2^k \cdot \left( 1+ \frac{ e^{-2(p-\frac{1}{2})^2} -1 }{2^k} \right)^n. \label{eq:prob_b=_tilde}
\end{align}
Similarly, 
\begin{align*}
\Prob{ \mathcal{W} = \tilde{\mathcal{W}}_k} & \geq 1-2^k \cdot \left( 1+ \frac{ e^{-2(p-\frac{1}{2})^2} -1 }{2^k} \right)^n.
\end{align*}
Furthermore,
\begin{align}
    \Esp{\left(\frac{1}{N_n(\mathcal{B})} \displaystyle \sum_{X_i \in \mathcal{B}} Y_i - p\right)^2 | \mathcal{B} = \tilde{\mathcal{B}}_k } &= \Esp{\frac{1}{N_n(\mathcal{B})^2} \mathbb{E} \left[ \left( \displaystyle \sum_{X_i \in \mathcal{B}} Y_i - N_n(\mathcal{B}) p \right)^2 | N_1,...N_{2^k}, \mathcal{B} = \tilde{\mathcal{B}}_k \right] | \mathcal{B} = \tilde{\mathcal{B}}_k} \label{eq:bb_to_bound}
\end{align}
where we let $ Z = \sum_{X_i \in \mathcal{B}} Y_i$. A typical bias-variance decomposition yields
\begin{align}
    &\mathbb{E} \left[ \left( \displaystyle \sum_{X_i \in \mathcal{B}} Y_i - N_n(\mathcal{B}) p \right)^2 | N_1,...N_{2^k}, \mathcal{B} = \tilde{\mathcal{B}}_k \right]\\
    &= \Esp{ \left(Z - \Esp{Z \mid N_1,...N_{2^k}, \tilde{\mathcal{B}}_k = \mathcal{B}}\right)^2 + \left( \Esp{Z \mid N_1,...N_{2^k}, \tilde{\mathcal{B}}_k = \mathcal{B}} - N_n(\mathcal{B})p \right)^2 \mid N_1,...N_{2^k}, \tilde{\mathcal{B}}_k = \mathcal{B}} \\
    &\geq \Esp{ \left(Z - \Esp{Z \mid N_1,...N_{2^k}, \tilde{\mathcal{B}}_k = \mathcal{B} } \right)^2 \mid N_1,...N_{2^k}, \tilde{\mathcal{B}}_k = \mathcal{B}} \\
    &= \Esp{ \left( \displaystyle \sum_{L_j \subset \mathcal{B}} Z_j - \Esp{Z_j \mid N_j, L_j \subset \tilde{\mathcal{B}}_k }\right)^2 \mid N_1,...N_{2^k}, \tilde{\mathcal{B}}_k = \mathcal{B}} \label{eq:change_conditionning}\\
    &= \displaystyle \sum_{L_j \subset \mathcal{B}}  \Esp{ \left( Z_j - \Esp{Z_j \mid N_j, L_j \subset  \tilde{\mathcal{B}}_k } \right)^2 \mid N_j, L_j \subset \ \tilde{\mathcal{B}}_k } \nonumber\\
    &+ 2 \displaystyle \sum_{L_i, L_j \subset \mathcal{B}, L_i \neq L_j}  \Esp{ \left( Z_i - \Esp{Z_i \mid N_i, L_i \subset  \tilde{\mathcal{B}}_k } \right) \left( Z_j - \Esp{Z_j \mid N_j, L_j \subset  \tilde{\mathcal{B}}_k } \right)  \mid N_1,...N_{2^k}, \tilde{\mathcal{B}}_k = \mathcal{B}} \label{eq:cond_esp_sum_decompo}\\
    &= \displaystyle \sum_{L_j \subset \mathcal{B}}  \Esp{ \left( Z_j - \Esp{Z_j \mid N_j, L_j \subset  \tilde{\mathcal{B}}_k } \right)^2 \mid N_j, L_j \subset \tilde{\mathcal{B}}_k } \label{eq:final_decompo}.
\end{align}
with $Z_j = \sum_{X_i \in L_j} Y_i$, and $L_1, \hdots , L_{2^k}$ the leaves of the first layer tree. 
Note that $Z_j | N_j ,L_j \subset \mathcal{B} $ are i.i.d binomial variable $ \mathfrak{B}(N_j,p)$. In \eqref{eq:change_conditionning} and \eqref{eq:cond_esp_sum_decompo}, we used that that given a single leaf $L_j \subset \mathcal{B}$,  $ \Esp{Z_j \mid N_1,...N_{2^k}, \tilde{\mathcal{B}}_k = \mathcal{B}} = \Esp{Z_j \mid N_j, L_j \subset \tilde{\mathcal{B}}_k}$. To obtain \eqref{eq:final_decompo}, we used that conditional to $N_1,...N_{2^k}, \tilde{\mathcal{B}}_k = \mathcal{B}$, $Z_i$ and $Z_j$ are independent. Therefore the double sum equals $0$.
Let $j$ be an integer in $\{1,...,2^k\}$,
\begin{align}
    &\Esp{ \left( Z_j - \Esp{Z_j \mid N_j, L_j \subset  \tilde{\mathcal{B}}_k } \right)^2 \mid N_j, L_j \subset  \tilde{\mathcal{B}}_k} \\
    &= \Esp{Z_j^2 \mid N_j, L_j \subset  \tilde{\mathcal{B}}_k} - \Esp{Z_j \mid N_j, L_j \subset  \tilde{\mathcal{B}}_k}^2 \\
    &\geq \Esp{Z_j^2 \mid N_j} - \Esp{Z_j \mid N_j, L_j \subset  \tilde{\mathcal{B}}_k}^2 \label{ineq:exp_pos_cond}\\
    &= \label{eq:transformation_var_trunc_bin} N_j p(1-p) + N_j^2 p^2 - \left( N_j p + \frac{N_j}{2}(1-p) \frac{\Prob{Z_j = \frac{N_j}{2} \mid N_j}}{\displaystyle \sum_{i= \frac{N_j}{2}}^{N_j} \Prob{Z_j = i}} \right)^2 \\
    &\geq N_j (1-p) \left(p - N_j (1-p) \Prob{Z_j = \frac{N_j}{2} \mid N_j}^2 - 2 N_j p \cdot  \Prob{Z_j = \frac{N_j}{2} \mid N_j} \right) \label{eq:int}\\
    &\geq \label{eq:lower_bound_var_trunc_bin} N_j (1-p) \left( p - \frac{N_j(1-p)}{\pi \left(\frac{N_j}{2} + \frac{1 }{4} \right)} \left(4p(1-p) \right)^{N_j}  - \frac{2 N_j}{ \sqrt{\pi \left(\frac{N_j}{2} + \frac{1}{4} \right)}} \left(4p(1-p) \right)^{N_j/2}\right) \\
    &\geq N_j p(1-p) -  \left(  \frac{2(1-p)^2}{\pi} + 2 \sqrt{2} (1-p) \right) \cdot N_j^{3/2} \cdot \left(4p(1-p) \right)^{N_j/2}. 
    \label{eq:var_truncated_bin_low_bound_final}
\end{align}
We deduced Line \eqref{ineq:exp_pos_cond} from the fact that $Z_j^2$ is a positive random variable, \eqref{eq:transformation_var_trunc_bin} from Lemma (\ref{lem:tech_res_binomial}) (\ref{lem:ahmed_formula}), Line \eqref{eq:int} from the fact that $p>1/2$ and Line \eqref{eq:lower_bound_var_trunc_bin} from the  inequality  \eqref{ineq:binomial_coeff_pi} on the binomial coefficient.
Injecting \eqref{eq:cond_esp_sum_decompo} and \eqref{eq:var_truncated_bin_low_bound_final} into \eqref{eq:bb_to_bound} yields
\begin{align}
    \notag
    &\Esp{\left(\frac{1}{N_n(\mathcal{B})} \displaystyle \sum_{X_i \in \mathcal{B}} Y_i - p\right)^2 | \mathcal{B} = \tilde{\mathcal{B}}_k } \\
    &\geq \Esp{\frac{1}{N_n(\mathcal{B}_k)^2} \displaystyle \sum_{L_j \subset \mathcal{B}} \left( N_j p(1-p) -  \left(  \frac{2(1-p)^2}{\pi} + 2 \sqrt{2} (1-p) \right) \cdot N_j^{3/2}\cdot  \left(4p(1-p) \right)^{N_j/2} \right) | \mathcal{B} = \tilde{\mathcal{B}}_k } \\
    &\geq \Esp{\frac{p(1-p)}{N_n(\mathcal{B})} \mid \mathcal{B} = \tilde{\mathcal{B}}_k} - \left( \frac{2(1-p)^2}{\pi} + 2 \right) \displaystyle \sum_{L_j \subset \mathcal{B}} \Esp{ \left(4p(1-p) \right)^{N_j/2} \mid \mathcal{B} = \tilde{\mathcal{B}}_k} \\
    &\geq p(1-p) \Esp{\frac{1}{N_n(\mathcal{B})} \mid \mathcal{B} = \tilde{\mathcal{B}}_k} - 3\cdot 2^{k-1} \Esp{ \left( 4p(1-p) \right)^{N_b/2}\mid \mathcal{B} = \tilde{\mathcal{B}}_k}
\end{align}
where the last inequality relies on the fact that the $N_j, L_j \subset \mathcal{B}$ are i.i.d, with $b \in {1,...,2^k}$ be the index of a cell included in $\mathcal{B}$. $N_j$ is a binomial random variable $\mathfrak{B}(n, 2^{-k})$.
\begin{align}
    \Esp{ \left( 4p(1-p) \right)^{N_j/2}\mid \mathcal{B} = \tilde{\mathcal{B}}_k} &\leq \Esp{ \left( 4p(1-p) \right)^{N_j/2}} \frac{1}{ \Prob{\mathcal{B} = \tilde{\mathcal{B}}_k}} \\
    &= \left( \sqrt{4p(1-p)} \cdot 2^{-k} + (1-2^{-k}) \right)^n  \frac{1}{ \Prob{\mathcal{B} = \tilde{\mathcal{B}}_k}}.
\end{align}
From the inequality Line \eqref{eq:prob_b=_tilde}, we deduce that as soon as $n \geq  \frac{(k+1)\log(2)}{ \log(2^k) - \log(e^{-2 (p-1/2)^2} - 1 + 2^k) }$, 
\begin{align}
    \frac{1}{ \Prob{\mathcal{B} = \tilde{\mathcal{B}}_k}} \leq 2.
\end{align}
Therefore,
\begin{align}
    \Esp{ \left( 4p(1-p) \right)^{N_j/2}\mid \mathcal{B} = \tilde{\mathcal{B}}_k} &\leq 2 \left( \sqrt{4p(1-p)} \cdot 2^{-k} + (1-2^{-k}) \right)^n.
\end{align}
Moreover,
\begin{align}
     \Esp{\frac{1}{N_n(\mathcal{B})} | \mathcal{B} = \tilde{\mathcal{B}}_k } &\geq   \frac{1}{\Esp{N_n(\mathcal{B}) | \mathcal{B} = \tilde{\mathcal{B}}_k}} \\
     &\geq \frac{ \Prob{\mathcal{B} = \tilde{\mathcal{B}}_k } }{\Esp{N_n(\mathcal{B})}} \\
     &\geq \frac{2}{n} - \frac{2^{k+1}}{n} \left( 1+ \frac{ e^{-2(p-\frac{1}{2})^2} -1 }{2^k}  \right)^n
\end{align}
where the last inequality comes from the probability bound line \eqref{eq:prob_b=_tilde} and the fact that $N_n(\mathcal{B})$ is a binomial random variable $\mathfrak{B}(n,1/2)$.

Finally,
\begin{align}
     &\Esp{\left(\frac{1}{N_n(\mathcal{B})} \displaystyle \sum_{X_i \in \mathcal{B}} Y_i - p\right)^2 | \mathcal{B} = \tilde{\mathcal{B}}_k } \\
     &\geq \frac{2p(1-p)}{n} - 3 \cdot 2^k \left(1 - 2^{-k} \left(1 - \sqrt{4p(1-p)} \right) \right)^n -  \frac{2^{k+1} p(1-p)}{n} \left( 1+ \frac{ e^{-2(p-\frac{1}{2})^2} -1 }{2^k}  \right)^n.
\end{align}

Similarly, regarding the second term of \eqref{proof_eq5}, note that $\Prob{\tilde{\mathcal{B}}_k = \mathcal{B}} = \Prob{\tilde{\mathcal{W}}_k = \mathcal{W}}$ and $$ \mathbb{E} \left[ \left( \displaystyle \sum_{X_i \in \mathcal{W}} Y_i - N_n(\mathcal{W}) (1-p) \right)^2 | N_n(\mathcal{W}), \mathcal{W} = \tilde{\mathcal{W}}_k \right] = \mathbb{E} \left[ \left( \displaystyle \sum_{X_i \in \mathcal{W}} \ind{Y_i=0} - N_n(\mathcal{W})p \right)^2 | N_n(\mathcal{W}), \mathcal{W} = \tilde{\mathcal{W}}_k \right].$$
Thus we can adapt the above computation to this term :
\begin{align}
    &\mathbb{E} \left[ \left(\frac{1}{N_n(\mathcal{W})} \displaystyle \sum_{X_i \in \mathcal{W}} Y_i - p\right)^2 | \mathcal{W} = \tilde{\mathcal{W}}_k \right]\\
    &\geq \frac{ 2p(1-p)}{n} - 3 \cdot 2^{k} \left(1 - 2^{-k} \left(1 - \sqrt{4p(1-p)} \right) \right)^n -  \frac{2^{k+1} p(1-p)}{n} \left( 1+ \frac{ e^{-2(p-\frac{1}{2})^2} -1 }{2^k}  \right)^n.
\end{align}
Rearranging all terms proves the result :
\begin{align*}
    \mathbb{E} \left[ (r_{k,1,n}(X) - r(X))^2 \right] &\geq \left( \frac{ 2p(1-p) }{n} -  2^{k+2} \cdot \left(1 - 2^{-k} \left(1 - \sqrt{4p(1-p)} \right) \right)^n \right. \\
    &- \left. \frac{2^{k+1} p(1-p)}{n} \cdot \left( 1+ \frac{ e^{-2(p-\frac{1}{2})^2} -1 }{2^k} \right)^n \right) \left(1-2^k \cdot \left( 1+ \frac{ e^{-2(p-\frac{1}{2})^2} -1 }{2^k} \right)^n  \right)  \\
    &+  \frac{p^2+(1-p)^2}{2} \left(1-\frac{1}{2^k}\right)^n \\
    &\geq \frac{ 2p(1-p) }{n} - 2^{k+2} \cdot \left(1 - 2^{-k} \left(1 - \sqrt{4p(1-p)} \right) \right)^n 
    - \frac{2^{k+1}p(1-p)}{n} \cdot \left( 1+ \frac{ e^{-2(p-\frac{1}{2})^2} -1 }{2^k} \right)^n\\
    & - \frac{2^{k+1}p(1-p)}{n} \cdot \left( 1+ \frac{ e^{-2(p-\frac{1}{2})^2} -1 }{2^k} \right)^n + \frac{p^2+(1-p)^2}{2} \left(1-\frac{1}{2^k}\right)^n\\
    &\geq \frac{ 2p(1-p) }{n} - 2^{k+2} \cdot \left(1 - 2^{-k} \left(1 - \sqrt{4p(1-p)} \right) \right)^n 
    - \frac{2^{k+2}p(1-p)}{n} \cdot \left( 1 - \frac{ 1 - e^{-2(p-\frac{1}{2})^2} }{2^k} \right)^n\\
    & +  \frac{p^2+(1-p)^2}{2} \left(1-\frac{1}{2^k}\right)^n\\
    &\geq \frac{ 2p(1-p) }{n} - \frac{2^{k+3} \cdot (1- \rho_{k,p})^n}{n} + \frac{p^2+(1-p)^2}{2} \left(1-\frac{1}{2^k}\right)^n
\end{align*}
where 
\begin{align*}
    \rho_{k,p} &=   2^{-k} \min\left(  1 - \sqrt{4p(1-p)},  1 - e^{-2(p-\frac{1}{2})^2}  \right).
\end{align*}
Note that, since $p>1/2$,  $0 < \rho_{k,p} < 1$.

\vspace{1cm}

\begin{lemma}
\label{lem:cond_var}
Let $S$ be a positive random variable. For any real-valued $\alpha \in [0,1]$, for any $n \in \mathbb{N}$,
\begin{align*}
    \Prob{S \leq \alpha n} \mathbb{V}[S | S \leq \alpha n] \leq \mathbb{V}[S]
\end{align*}

\begin{proof}
We start by noticing that: 
\begin{align*}
    A_n &= \mathbb{P} \left( S > \alpha n\right) \Esp{\left( S - \Esp{S \mid S > \alpha n }\right)^2 \mid S > \alpha n} \\
    &+ \mathbb{P} \left( S \leq \alpha n\right) \Esp{\left( S - \Esp{S \mid S \leq \alpha n }\right)^2 \mid S \leq \alpha n} \\
    &\leq \mathbb{P} \left( S > \alpha n\right) \Esp{\left( S -a\right)^2 \mid S > \alpha n } + \mathbb{P} \left( S \leq \alpha n\right) \Esp{\left( S -b\right)^2 \mid S \leq \alpha n }
\end{align*}
for any $(a,b) \in \mathbb{R}^2$.

Then,
\begin{align*}
    A_n &\leq \mathbb{P} \left( S > \alpha n\right) \Esp{\left( S -a\right)^2 \mid S > \alpha n } + \mathbb{P} \left( S \leq \alpha n\right) \Esp{\left( S -a\right)^2 \mid S \leq \alpha n } \\
    &= \Esp{\left(S-a \right)^2}
\end{align*}
for any $a \in \mathbb{R}$. 

Choosing $a= \Esp{S}$, we obtain
\begin{align*}
    A_n \leq \mathbb{V}[S].
\end{align*}
Therefore,
$$ \mathbb{P} \left( S \leq \alpha n\right) \mathbb{V}[S \mid S \leq \alpha n ] \leq \mathbb{V}[S].$$ 

\end{proof}

\end{lemma}

\section{Extended results for a random chessboard}
\label{sec:extended_random_chessboard}

\begin{prop}[Risk of a single tree and a shallow tree network when $k<k^{\star}$]
\label{prop:new_risk_small_k}
Let $N \in \{1,...,\expkstar{}\}$. We consider the data distribution defined by a random chessboard with i.i.d.\  cells such that for each cell $C_i, i \in \{1,...,\expkstar{}\}$   $$\Prob{C_i \subset \mathcal{B}} = \frac{N}{\expkstar{}} 
$$ 
and $\Prob{C_i \subset \mathcal{W}} = 1- \frac{N}{\expkstar{}} $.
Notice that the (random) numbers $N_\mathcal{W}$ and $N_\mathcal{B}$ of white and black cells satisfy $0 \leq N_\mathcal{W} = 2^{k^{\star}}-N_\mathcal{B}\leq 2^{k^\star}$.
We study the risk of the shallow tree network $\hat{r}_{k,1,n}$.
\begin{enumerate}
    \item Consider a single tree $\hat{r}_{k,0,n}$ of depth $k \in \mathbb{N}^\star$,
    \begin{align*}
     R(\hat{r}_{k,0,n}) \leq  4(p-\frac{1}{2})^2 \frac{N}{\expkstar{}}\left(1 - \frac{N}{\expkstar{}} \right) \left(1+\frac{1}{2^{k^\star-k}} \right) + \frac{2^{k-1}}{n+1} + \left( (1-p)^2 - \frac{N}{\expkstar{}} (1-2p)\right) (1-2^{-k})^n
    \end{align*}
    and
    \begin{align*}
    R(\hat{r}_{k,0,n}) \geq 4(p-\frac{1}{2})^2 \frac{N}{\expkstar{}}\left(1 - \frac{N}{\expkstar{}} \right) + \frac{2^k}{n+1}(1-p)^2 + C_{k^\star,k,N,p} (1-2^{-k})^n
    \end{align*}
    where $C_{k^\star,k,N,p} = (1-p)^2 - \frac{N}{\expkstar{}} (1-2p) - \frac{(1-p)^2 2^k}{n+1} - 4(p-\frac{1}{2})^2 \frac{N}{\expkstar{}}\left(1 - \frac{N}{\expkstar{}} \right) \left(1+\frac{1}{2^{k^*-k}} \right)$.
    \item Consider the shallow tree network $\hat{r}_{k,1,n}$, in the infinite sample regime,
        \begin{align*}
            R(\hat{r}_{k,1,n})\geq \left(p-\frac{1}{2} \right)^2 \min \left(1-\frac{N}{\expkstar{}} , \frac{N}{\expkstar{}} \right)^2
        \end{align*}
        and
        \begin{align*}
            R(\hat{r}_{k,1,n})\leq 4\left(p-\frac{1}{2}\right)^2 \left( 1 - \frac{N}{\expkstar{}}\right) \frac{N}{\expkstar{}} + p^2 \min\left( \frac{N}{\expkstar{}}, 1-\frac{N}{\expkstar{}} \right).
        \end{align*}
\end{enumerate}

\end{prop}

\section{Proof of Proposition \ref{prop:new_risk_small_k}}
\label{proof:prop_small_k}

\subsection{First statement: risk of a single tree}

To see the definitions of $\tilde{\mathcal{B}}$ and $\tilde{\mathcal{W}}$ refer to the notations of the second statement of the proof of Proposition \ref{prop:balanced_chessboard}, in Appendix  \ref{proof:snd_statement_balanced_damier}.

Recall that $k<k^\star$, meaning that a tree leaf may contain black and white cells. If a cell is empty, the tree prediction in this cell is set (arbitrarily) to zero. Thus, 
\begin{align}
    \nonumber
    & \Esp{(\hat{r}_{k,0,n}(X) - r(X))^2} \\
    &= \Esp{(\hat{r}_{k,0,n}(X) - r(X))^2 \ind{\cardcellof{X} >0}} + \Esp{(r(X))^2 \ind{\cardcellof{X} =0}} \\
    &= \Esp{\left( \frac{1}{\cardcellof{X}}  \displaystyle \sum_{X_i \in \cellof{X}} Y_i -r(X) \right)^2 \ind{\cardcellof{X} >0} } +\Esp{(r(X))^2 \ind{\cardcellof{X} =0}},
    \label{eq_proof1_1}
\end{align}
where the expectation is taken over the distribution of the chessboard, $(X,Y)$ and the dataset $(X_i,Y_i)_{1\leq i \leq n}$.  Besides,
\begin{align}
\Esp{(r(X))^2 \ind{\cardcellof{X} =0}} & = \Esp{(r(X))^2 \ind{\cardcellof{X} =0} \ind{X \in \mathcal{B}}} + \Esp{(r(X))^2 \ind{\cardcellof{X} =0} \ind{X \in \mathcal{W}}} \\
& = ((1-p)^2 - {\frac{N_\mathcal{B}}{\expkstar{}} }(1-2p)) (1-2^{-k})^n
\label{eq:probempty}
\end{align}
We now study the first term in \eqref{eq_proof1_1}, by considering that $X$ falls into $\mathcal{B}$ (the same computation holds when $X$ falls into $\mathcal{W}$). We denote $\NBC|$ (resp. $\NWC$) the number of black (resp. white) cells included in the cell containing $X$. 
Letting $(X', Y')$ generic random variables with the same distribution as $(X,Y)$, one has
\begin{align}
    &\Esp{\left( \frac{1}{\cardcellof{X}}  \displaystyle  \sum_{X_i \in \cellof{X}} Y_i -p \right)^2 \ind{\cardcellof{X} >0} \ind{X \in \mathcal{B}} } \\
    &= \Prob{X\in \mathcal{B}} \Esp{\left( \frac{1}{\cardcellof{X}}  \displaystyle \sum_{X_i \in \cellof{X}} \left(Y_i - \Esp{Y' | X' \in \cellof{X}, \NBC} \right) \right)^2  \ind{\cardcellof{X} >0}} \\
    \notag
    &\qquad \qquad + \Prob{X\in \mathcal{B}}  \Esp{\left( \Esp{Y' | X' \in \cellof{X}, \NBC} -p \right)^2 \ind{\cardcellof{X} >0}}\\ 
    \label{eq:risk_singletree_B}
    &= \Prob{X\in \mathcal{B}}  \cdot \nonumber \\
    &\Esp{ \frac{\ind{\cardcellof{X} >0}}{\cardcellof{X}^2} \Esp{\left( \displaystyle \sum_{X_i \in \cellof{X}} \left( Y_i - \Esp{Y' | X' \in \cellof{X}, \NBC} \right) \right)^2 \left\vert\right. \cardcellof{X}, \NBC}} \\
    \label{eq:risk_singletree_B_2}
    & \qquad \qquad + \Prob{X\in \mathcal{B}}  \Prob{\cardcellof{X} >0} \Esp{\left( \Esp{Y' | X' \in \cellof{X}, \NBC} -p \right)^2} \nonumber \\
    &= \Prob{X\in \mathcal{B}} \left(\Esp{ \frac{\ind{\cardcellof{X} >0}}{\cardcellof{X}^2} \cdot V_\mathcal{B}} + \beta _\mathcal{B}  \right)
\end{align}
where
\begin{align}
\beta_\mathcal{B} = \Prob{\cardcellof{X} >0} \Esp{\left( \Esp{Y' | X' \in \cellof{X}, \NBC} -p \right)^2} \label{eq:bias_b}
\end{align}
and
\begin{align}
V_\mathcal{B} = \Esp{\left( \displaystyle \sum_{X_i \in \cellof{X}} \left( Y_i - \Esp{Y' | X' \in \cellof{X}, \NBC} \right) \right)^2 \left\vert\right. \cardcellof{X}, \NBC} \label{eq:var_b}
\end{align} 
Similarly we define $\beta_W$ and $V_\mathcal{W}$ by replacing in the expressions~\eqref{eq:bias_b} and \eqref{eq:var_b} $\mathcal{B}$ by $\mathcal{W}$ so that:
\begin{align}
    & \Esp{(\hat{r}_{k,0,n}(X) - r(X))^2} \nonumber \\
    &= \Prob{X\in \mathcal{B}} \left(\Esp{ \frac{\ind{\cardcellof{X} >0}}{\cardcellof{X}^2} \cdot V_\mathcal{B}} + \beta _\mathcal{B} \right)+ \Prob{X\in \mathcal{W}} \left(\Esp{ \frac{\ind{\cardcellof{X} >0}}{\cardcellof{X}^2} \cdot V_\mathcal{W}} + \beta _\mathcal{W} \right) \nonumber \\
    &+ ((1-p)^2 - {\frac{N_\mathcal{B}}{\expkstar{}} }(1-2p)) (1-2^{-k})^n. \label{eq:b_v_decompo}
\end{align}
This last expression can be read as a bias-variance decomposition. 
We already know the probability to be in a non-empty cell, see \eqref{eq:probempty}, then $$ \Prob{X \in \mathcal{B}} = \Esp{\Prob{X \in \mathcal{B} | N_\mathcal{B}}} = \Esp{\frac{N_\mathcal{B}}{2^{k^\star}}} = \frac{N}{\expkstar{}}. $$  
We now make explicit the terms in Equation \eqref{eq:b_v_decompo} starting with the bias term $\beta _\mathcal{B}$:
\begin{align}
    \Esp{Y' | X' \in \cellof{X}, \NBC} &= \frac{p \cdot \NBC + (1-p) \NWC}{2^{k^\star-k}} \label{eq:p_prime_1}\\
    &= (1-p) + \frac{\NBC}{2^{k^\star-k}} \left(2p-1\right) \\
    &= p + \frac{\NWC}{2^{k^\star-k}}(1-2p) \label{eq:p_prime_2},
\end{align}
where $|\cellof{X} \cap \mathcal{B} |$ stands for the number of black cells in $\cellof{X}$. In the same way, $\NWC$ stands for the number of white cells in $\cellof{X}$.
Hence,
\begin{align} 
    \Esp{\left( \Esp{Y' | X' \in \cellof{X}, \NBC} -p \right)^2} &= \frac{4 (p-\frac{1}{2})^2}{2^{2(k^\star-k)}}\Esp{\NWC^2}
\end{align}
Note that $\NWC | N_\mathcal{B} \sim \mathfrak{B}(2^{k^\star-k}, 1 - N/2^{k^\star})$. Thus, we have
\begin{align}
    \Esp{\NWC^2} = \frac{N}{2^k}\left( 1-\frac{N}{\expkstar{}} \right) + \frac{1}{2^{2k}}(\expkstar{}-N)^2. \label{eq:bin_cond_esp}
\end{align}
Therefore,
\begin{align}
    &\Esp{\left( \Esp{Y' | X' \in \cellof{X}, \NBC} -p \right)^2} = \frac{4 (p-\frac{1}{2})^2}{2^{2(k^\star-k)}} \left( \frac{N}{2^k} \left(1-\frac{N}{\expkstar{}}\right) + \frac{1}{2^{2k}}(\expkstar{}-N)^2 \right). \label{eq:b_bias}
\end{align}
Similar computations show that when $X \in \mathcal{W}$, 
\begin{align}
    &\Esp{\left( \Esp{Y' | X' \in \cellof{X}, \NWC} -(1-p) \right)^2} = \frac{4 (p-\frac{1}{2})^2}{2^{2(k^\star-k)}} \left( \frac{N}{2^k} \left(1-\frac{N}{\expkstar{}}\right) + \frac{N^2}{2^{2k}}\right). 
    \label{eq:w_bias}
\end{align}
We deduce from Equations~\eqref{eq:b_bias} and \eqref{eq:w_bias} that 
\begin{align}
     \nonumber \beta _\mathcal{B} + \beta _\mathcal{W} & = \frac{4 (p-\frac{1}{2})^2}{2^{2(k^\star-k)}} \Prob{N_n(\cellof{X})>0}  \left( \Prob{X \in \mathcal{B}} \left( \frac{N}{2^k} \left(1-\frac{N}{\expkstar{}}\right) + \frac{1}{2^{2k}}(\expkstar{}-N)^2 \right) \right.\\
     & \left. \quad  + ~\Prob{X \in \mathcal{W}} \left( \frac{N}{2^k}\left(1-\frac{N}{\expkstar{}}\right) + \frac{N^2}{2^{2k}} \right)  \right) \nonumber \\
    &= 4(p-\frac{1}{2})^2 \frac{N}{\expkstar{}}\left(1 - \frac{N}{\expkstar{}} \right) \left(1+\frac{1}{2^{k^\star-k}} \right) \left( 1 - (1-2^{-k})^n\right) .
\end{align}
Clearly,
\begin{align}
    \beta _\mathcal{B} + \beta _\mathcal{W} &\geq 4 (p-\frac{1}{2})^2 \frac{N}{\expkstar{}}\left(1 - \frac{N}{\expkstar{}} \right) \left( 1 - (1-2^{-k})^n\right). \label{ineq:bias}
\end{align} 
Now we compute the variance term $V_\mathcal{B}$. 
Letting $Z = \sum_{X_i \in \cellof{X}} Y_i$, 
\begin{align*}
Z | N_n(\cellof{X}), \NBC \sim \mathfrak{B}(N_n(\cellof{X}), p')    
\end{align*}
where $p' = (1-p) + \frac{\NBC}{\expkstar{-k}} (2p-1)$ (see Equations \eqref{eq:p_prime_1}  to \eqref{eq:p_prime_2}). 
Therefore, recall that  $V_\mathcal{B}$ is nothing but the variance of the  binomial random variable $Z$ conditional on $\NBC$ defined in Equation \eqref{eq:var_b}, consequently
\begin{align}
    V_\mathcal{B} = \cardcellof{X} p'(1-p'). 
\end{align}
By independence of $\cardcellof{X}$ and $\NBC$, we can write that
\begin{align}
    \Esp{ \frac{\ind{\cardcellof{X} >0}}{\cardcellof{X}^2}\cdot V_\mathcal{B}} 
    &= \Esp{\frac{\ind{\cardcellof{X} >0}}{\cardcellof{X}}} \Esp{\underbrace{p'(1-p)'}_{(1-p)^2\leq  p'(1-p') \leq 1/4}}.  
\end{align}
From Technical Lemma \ref{lem:tech_res_binomial}, we deduce that
$$ \frac{2^k}{n+1} \left( 1 - (1-2^{-k})^n \right)\leq \Esp{\frac{\ind{\cardcellof{X} >0}}{\cardcellof{X}}} \leq \frac{2^{k+1}}{n+1}.$$
Hence,
\begin{align}
    \frac{2^k}{n+1} \left( 1 - (1-2^{-k})^n \right) (1-p)^2 \leq \Esp{\frac{\ind{\cardcellof{X} >0}}{\cardcellof{X}} V_\mathcal{B}}  \leq \frac{2^{k-1}}{n+1} \label{ineq:variance}.
\end{align}
By symmetry, $V_{\mathcal{W}}$ is also the variance of a binomial random variable with parameters $\cardcellof{X}$, $1-p'$ conditional on $\NWC$. Thus $V_\mathcal{B} = V_\mathcal{W}$.
To conclude, combining Equations \eqref{eq:b_v_decompo}, \eqref{ineq:bias} and \eqref{ineq:variance} leads to 
\begin{align*}
    R(\hat{r}_{k,0,n}(X)) \leq  4(p-\frac{1}{2})^2 \left( 1- \frac{N}{\expkstar{}}  (1 - \frac{N}{\expkstar{}})\right) + \frac{2^{k-1}}{n+1} + ((1-p)^2 - \frac{N}{\expkstar{}} (1-2p)) (1-2^{-k})^n
\end{align*}
and
\begin{align*}
    R(\hat{r}_{k,0,n}(X)) &\geq 4(p-\frac{1}{2})^2 \left( 1 - \frac{N}{\expkstar{}}(2 - \frac{N}{\expkstar{}})\right) \frac{2^k}{n+1} \left( 1 - (1-2^{-k})^n \right) (1-p)^2  \\
    & \quad + ((1-p)^2 - \frac{N}{\expkstar{}} (1-2p)) (1-2^{-k})^n.
\end{align*}

\subsection{Second statement: risk of a shallow tree network}
\label{prop:second_statement_small_k}

Recall that we are in the infinite sample regime and that $k<k^\star$. 

\begin{align}
    \Esp{(\hat{r}_{k,1,n}(X) - r(X)^2} &= \Esp{\left(\hat{r}_{k,1,n}(X)-p\right)^2 \ind{X\in \mathcal{B} \cap \Tilde{\mathcal{B}}}} + \Esp{\left(\hat{r}_{k,1,n}(X)-(1-p)\right)^2 \ind{X\in \mathcal{W} \cap \Tilde{\mathcal{W}}}} \nonumber \\
    & \quad + \Esp{\left(\hat{r}_{k,1,n}(X)-(1-p)\right)^2 \ind{X\in \mathcal{W} \cap \Tilde{\mathcal{B}}}} + \Esp{\left(\hat{r}_{k,1,n}(X)-p\right)^2 \ind{X\in \mathcal{B} \cap \Tilde{\mathcal{W}}}} \label{eq:risk_decompo}.
\end{align}
We begin with the computation of the first term.
\begin{align}
    \Esp{\left(\hat{r}_{k,1,n}(X)-p\right)^2 \ind{X\in \mathcal{B} \cap \Tilde{\mathcal{B}}}} &= \Prob{X\in \mathcal{B} \cap \Tilde{\mathcal{B}}} \Esp{\left(\hat{r}_{k,1,n}(X) -p \right)^2 \mid X\in \mathcal{B} \cap \Tilde{\mathcal{B}}} \\
    &= \Prob{X\in \mathcal{B} \cap \Tilde{\mathcal{B}}} \Esp{\left(\Esp{Y' \mid X' \in \Tilde{\mathcal{B}}} -p \right)^2 \mid X\in \mathcal{B} \cap \Tilde{\mathcal{B}}}. \label{eq:b_b_tilde_bias}
\end{align}
Regarding the probability term,
\begin{align}
    \Prob{X\in \mathcal{B} \cap \Tilde{\mathcal{B}}} &= \Prob{X \in \mathcal{B}}\Prob{X \in \Tilde{\mathcal{B}} \mid X \in \mathcal{B}} \\
    &\leq \frac{N}{\expkstar{}}.
\end{align}
We denote by $B_1,...,B_{2^k}$ the number of black cells in the leaves $L_1,...,L_{2^k}$. Then,
\begin{align*}
    \Esp{Y' \mid X' \in \Tilde{\mathcal{B}}} &= \Esp{(1-p) + (2p-1) \displaystyle \frac{\sum_{i=1}^{2^k} B_i \ind{L_i \subset \Tilde{\mathcal{B}} }}{|\Tilde{\mathcal{B}}|} } \\
    &= (1-p) + (2p-1) \Esp{\displaystyle \sum_{i=1}^{2^k} \frac{\ind{L_i \subset \Tilde{\mathcal{B} }}}{|\Tilde{\mathcal{B}}|} \Esp{B_i \Big| |\Tilde{\mathcal{B}}|}}\\
    &= (1-p) + (2p-1) \Esp{\frac{B_j}{\expkstar{-k}} \mid B_j \geq \frac{|L_j|}{2}}
\end{align*}
where $L_j$ is a leaf included in $ \Tilde{\mathcal{B}}$. 
Moreover,
\begin{align*}
    \Esp{B_j \mid B_j \geq \frac{|L_j|}{2}} &= \frac{N}{2^k} + \left(1-\frac{N}{\expkstar{}} \right)\left(\expkstar{-k-1} -1\right)
    \frac{\Prob{B_j = \expkstar{-k-1} -1}}{\Prob{B_j \geq \expkstar{-k-1} -1}} \\
    &\leq \frac{N}{2^k} + \left(1-\frac{N}{\expkstar{}} \right) \expkstar{-k-1}.
\end{align*}
Therefore,
\begin{align*}
    \Esp{Y' \mid X' \in \Tilde{\mathcal{B}}} &\leq  (1-p) + (2p-1) \frac{1}{2} \left(1 + \frac{N}{\expkstar{}} \right) 
\end{align*}
and
\begin{align}
    \Esp{\left(\Esp{Y' \mid X' \in \Tilde{\mathcal{B}}} -p \right)^2 \mid X\in \mathcal{B} \cap \Tilde{\mathcal{B}}} &\geq (p-\frac{1}{2})^2 \left(1 - \frac{N}{\expkstar{}} \right)^2. \label{eq:bias_lower_bound}
\end{align}
To compute the upper bound, note that
\begin{align*}
    \Esp{B_j \mid B_j \geq \frac{|L_j|}{2}} \geq \Esp{B_j} = \frac{N}{2^k}.
\end{align*}
Thus,
\begin{align}
    \Esp{\left(\Esp{Y' \mid X' \in \Tilde{\mathcal{B}}} -p \right)^2 \mid X\in \mathcal{B} \cap \Tilde{\mathcal{B}}} &\leq 4 \left(p-\frac{1}{2}\right)^2 \left( 1 - \frac{N}{\expkstar{}}\right)^2. \label{eq:bias_upper_bound}
\end{align}
We adapt the previous computations to the term $ \Esp{\left(\hat{r}_{k,1,n}(X)-(1-p)\right)^2 \ind{X\in \mathcal{W} \cap \Tilde{\mathcal{W}}}}$ from Equation \eqref{eq:risk_decompo}. We have
\begin{align}
    \Esp{\left(\hat{r}_{k,1,n}(X)-(1-p)\right)^2 \mid X\in \mathcal{W} \cap \Tilde{\mathcal{W}}} &\geq (p-\frac{1}{2})^2\frac{N^2}{2^{2k^\star}}
\end{align}
and
\begin{align}
    \Esp{\left(\hat{r}_{k,1,n}(X)-(1-p)\right)^2 \mid X\in \mathcal{W} \cap \Tilde{\mathcal{W}}} \leq 4 \left(p - \frac{1}{2} \right)^2 \frac{N^2}{2^{2k^\star}}
\end{align}
Moreover, note that 
\begin{align}
    \Esp{\left(\hat{r}_{k,1,n}(X)-p\right)^2 \ind{X\in \mathcal{B} \cap \Tilde{\mathcal{W}}}} \leq p^2 \Prob{X\in \mathcal{B} \cap \Tilde{\mathcal{W}}}
\end{align}
and
\begin{align}
    \Esp{\left(\hat{r}_{k,1,n}(X)-p\right)^2 \ind{X\in \mathcal{B} \cap \Tilde{\mathcal{W}}}} \geq \left(p-\frac{1}{2} \right)^2 \Prob{X\in \mathcal{B} \cap \Tilde{\mathcal{W}}}.
\end{align}
Similarly,
\begin{align}
    \Esp{\left(\hat{r}_{k,1,n}(X)-p\right)^2 \ind{X\in \mathcal{W} \cap \Tilde{\mathcal{B}}}} \leq p^2 \Prob{X\in \mathcal{W} \cap \Tilde{\mathcal{B}}}
\end{align}
and
\begin{align}
    \Esp{\left(\hat{r}_{k,1,n}(X)-p\right)^2 \ind{X\in \mathcal{W} \cap \Tilde{\mathcal{B}}}} \geq \left(p-\frac{1}{2} \right)^2 \Prob{X\in \mathcal{W} \cap \Tilde{\mathcal{B}}}. \label{eq:bias_last}
\end{align}
Gathering Equation \eqref{eq:risk_decompo} and Equations \eqref{eq:bias_lower_bound} to \eqref{eq:bias_last} yields
\begin{align*}
    \Esp{(\hat{r}_{k,1,n}(X) - r(X)^2} &\geq (p-\frac{1}{2})^2 \left(1 - \frac{N}{\expkstar{}} \right)^2 \Prob{X\in \mathcal{B} \cap \Tilde{\mathcal{B}}} + (p-\frac{1}{2})^2\frac{N^2}{2^{2k^\star}} \Prob{X\in \mathcal{W} \cap \Tilde{\mathcal{W}}} \nonumber \\
    & \quad + \left(p-\frac{1}{2}\right)^2 \Prob{X\in \mathcal{W} \cap \Tilde{\mathcal{B}}} + \left(p-\frac{1}{2}\right)^2 \Prob{X\in \mathcal{B} \cap \Tilde{\mathcal{W}}} \\
    &\geq \left(p-\frac{1}{2} \right)^2 \min \left(1-\frac{N}{\expkstar{}} , \frac{N}{\expkstar{}} \right)^2
\end{align*}
as well as
\begin{align*}
    &\Esp{(\hat{r}_{k,1,n}(X) - r(X)^2} \\
    &\leq 4 \left(p-\frac{1}{2}\right)^2 \left( 1 - \frac{N}{\expkstar{}}\right)^2 \Prob{X\in \mathcal{B} \cap \Tilde{\mathcal{B}}} +  4 \left(p - \frac{1}{2} \right)^2 \frac{N^2}{2^{2k^\star}} \Prob{X\in \mathcal{W} \cap \Tilde{\mathcal{W}}} \\
    & \quad + p^2 \Prob{X\in \mathcal{W} \cap \Tilde{\mathcal{B}}} + p^2 \Prob{X\in \mathcal{B} \cap \Tilde{\mathcal{W}}}\\
    &\leq 4 \left(p-\frac{1}{2}\right)^2 \left( 1 - \frac{N}{\expkstar{}}\right)^2 \frac{N}{\expkstar{}}\Prob{X \in \Tilde{\mathcal{B}} \mid X \in \mathcal{B}} +  4 \left(p - \frac{1}{2} \right)^2 \frac{N^2}{2^{2k^\star}}\left(1 - \frac{N}{\expkstar{}}\right) \Prob{X\in \Tilde{\mathcal{W}} \mid X \in \mathcal{W}} \\
    & \quad + p^2 \left(1 - \frac{N}{\expkstar{}}\right) \Prob{X\in \Tilde{\mathcal{B}} \mid X \in \mathcal{W}} + p^2 \frac{N}{\expkstar{}} \Prob{X\in \Tilde{\mathcal{W}} \mid X \in \mathcal{B}} \\
    &\leq 4 \left(p-\frac{1}{2}\right)^2 \left( 1 - \frac{N}{\expkstar{}}\right)^2 \frac{N}{\expkstar{}} +  4 \left(p - \frac{1}{2} \right)^2 \frac{N^2}{2^{2k^\star}}\left(1 - \frac{N}{\expkstar{}}\right) \\
    & \quad + p^2 \left(1 - \frac{N}{\expkstar{}}\right) \Prob{X\in \Tilde{\mathcal{B}}} + p^2 \frac{N}{\expkstar{}} \Prob{X\in \Tilde{\mathcal{W}}} \\
    &\leq 4\left(p-\frac{1}{2}\right)^2 \left( 1 - \frac{N}{\expkstar{}}\right) \frac{N}{\expkstar{}} + p^2 \max\left( \frac{N}{\expkstar{}} 1-\frac{N}{\expkstar{}} \right).
\end{align*}

\end{document}